\documentclass[twocolumn,twoside]{IEEEtran}
\usepackage{url}
\usepackage{etex}
\usepackage{tabularx}
\usepackage[normalem]{ulem}
\reserveinserts{28}
\usepackage{amsmath,graphicx}
\usepackage{amssymb}
\usepackage{algorithmic}
\usepackage{algorithm}
\usepackage{multirow,multicol}
\usepackage{pdfpages}
\usepackage{epstopdf}
\usepackage{subfigure}
\usepackage{balance}

\usepackage{cite}

\usepackage{amsthm}

\newtheorem{proposition}{Proposition}

\newtheorem{remark}{Remark}

\newtheorem{definition}{Definition}
\newtheorem{example}{Example}

\newenvironment{densitemize}
{\begin{list}               
    {$\bullet$ \hfill}{
        \setlength{\leftmargin}{\parindent}
        \setlength{\parsep}{0.04\baselineskip}
        \setlength{\itemsep}{0.5\parsep}
        \setlength{\labelwidth}{\leftmargin}
        \setlength{\labelsep}{0em}}
    }
{\end{list}}

\providecommand{\eref}[1]{\eqref{#1}}  
\providecommand{\cref}[1]{Chapter~\ref{#1}}

\providecommand{\fref}[1]{Figure~\ref{#1}}

\providecommand{\R}{\ensuremath{\mathbb{R}}}

\providecommand{\E}{\ensuremath{\mathbb{E}}}
\providecommand{\N}{\ensuremath{\mathbb{N}}}

\providecommand{\Pb}{\ensuremath{\mathbb{P}}}

\providecommand{\bydef}{\overset{\text{def}}{=}}

\renewcommand{\vec}[1]{\ensuremath{\boldsymbol{#1}}}
\providecommand{\mat}[1]{\ensuremath{\boldsymbol{#1}}}

\providecommand{\calB}{\mathcal{B}}

\providecommand{\calQ}{\mathcal{Q}}

\providecommand{\mG}{\mat{G}}

\providecommand{\mI}{\mat{I}}

\providecommand{\vc}{\vec{c}}

\providecommand{\vtheta}{\vec{\theta}}

\providecommand{\chat}{\widehat{c}}

\providecommand{\vchat}{\boldsymbol{\widehat{c}}}

\providecommand{\Var}{\mathrm{Var}}

\newcommand{\argmax}[1]{\mathop{\underset{#1}{\mbox{argmax}}}}

\begin{document}

\title{Optimal Threshold Design for Quanta Image Sensor}
\author{Omar A. Elgendy,~\IEEEmembership{Student Member,~IEEE} and Stanley H. Chan,~\IEEEmembership{Member,~IEEE}
\thanks{The authors are with the School of Electrical and Computer Engineering, Purdue University, West Lafayette, IN 47907, USA. Email: \texttt{\{ oelgendy, stanchan\}@purdue.edu}. The work was supported, in part, by the U.S. National Science Foundation
under Grant CCF-1718007. A preliminary version of this paper was presented at the 2016 IEEE International Conference on Image Processing (ICIP), Phoenix, AZ.}
\thanks{This paper follows the concept of reproducible research. All the results and examples presented in the paper are reproducible using the code and images available online at http://engineering.purdue.edu/ChanGroup/.}}

\markboth{IEEE TRANSACTIONS ON COMPUTATIONAL IMAGING,~Vol.~X, No.~X, XXX~20XX}{Elgendy-Chan: Optimal Threshold Design for Quanta Image Sensor}

\graphicspath{{pix/}}

\maketitle

\begin{abstract}
\textcolor{black}{Quanta Image Sensor (QIS) is a binary imaging device envisioned to be the next generation image sensor after CCD and CMOS. Equipped with a massive number of single photon detectors, the sensor has a threshold $q$ above which the number of arriving photons will trigger a binary response ``1'', or ``0'' otherwise. Existing methods in the device literature typically assume that $q = 1$ uniformly. We argue that a spatially varying threshold can significantly improve the signal-to-noise ratio of the reconstructed image. In this paper, we present an optimal threshold design framework. We make two contributions. First, we derive a set of oracle results to theoretically inform the maximally achievable performance. We show that the oracle threshold should match exactly with the underlying pixel intensity. Second, we show that around the oracle threshold there exists a set of thresholds that give asymptotically unbiased reconstructions. The asymptotic unbiasedness has a phase transition behavior which allows us to develop a practical threshold update scheme using a bisection method. Experimentally, the new threshold design method achieves better rate of convergence than existing methods.}
\end{abstract}
\IEEEpeerreviewmaketitle

\section{Introduction}
\label{sec:intro}

\subsection{Threshold Design for Quanta Image Sensor}
Quanta Image Sensor (QIS) is a class of solid-state image sensors envisioned to be the next generation imaging device after CCD and CMOS. Originally proposed by Eric Fossum in 2005 \cite{Fossum_2005}, the sensor has gained significant momentum in the past decade, both in terms of hardware design \cite{Ma_Hondongwa_Fossum_2014,Ma_Fossum_2015_1,Ma_Anzagira_Fossum_2016} and image processing \cite{Yang_Lu_Sbaiz_2010, Yang_Lu_Sbaiz_2012, Chan_Lu_2014, Elgendy_Chan_2016, Chan_Elgendy_Wang_2016}. The advantage of QIS over the mainstream CCD and CMOS is attributed to its high spatial resolution (e.g., $10^9$ pixels per sensor  with $200$nm pitch per pixel \cite{Fossum_Ma_Masoodian_2016}) and high speed (e.g., 100k fps as reported in \cite{Antolovic_Burri_Bruschini_2016}). {\color{black} However, in order to simplify circuit, minimize power and reduce data transfer, QIS is operated in a binary mode: When the number of photons arriving at the sensor exceeds a threshold $q$, the sensor generates a binary bit ``1''. When the number of photons is less than $q$, the sensor generates a ``0''. The goal of this paper is to address the question of how to optimally choose $q$ to maximize the signal-to-noise ratio of the reconstructed image.}

\begin{figure}[t]
\centering
\begin{tabular}{ccc}
    \includegraphics[width=2.8cm]{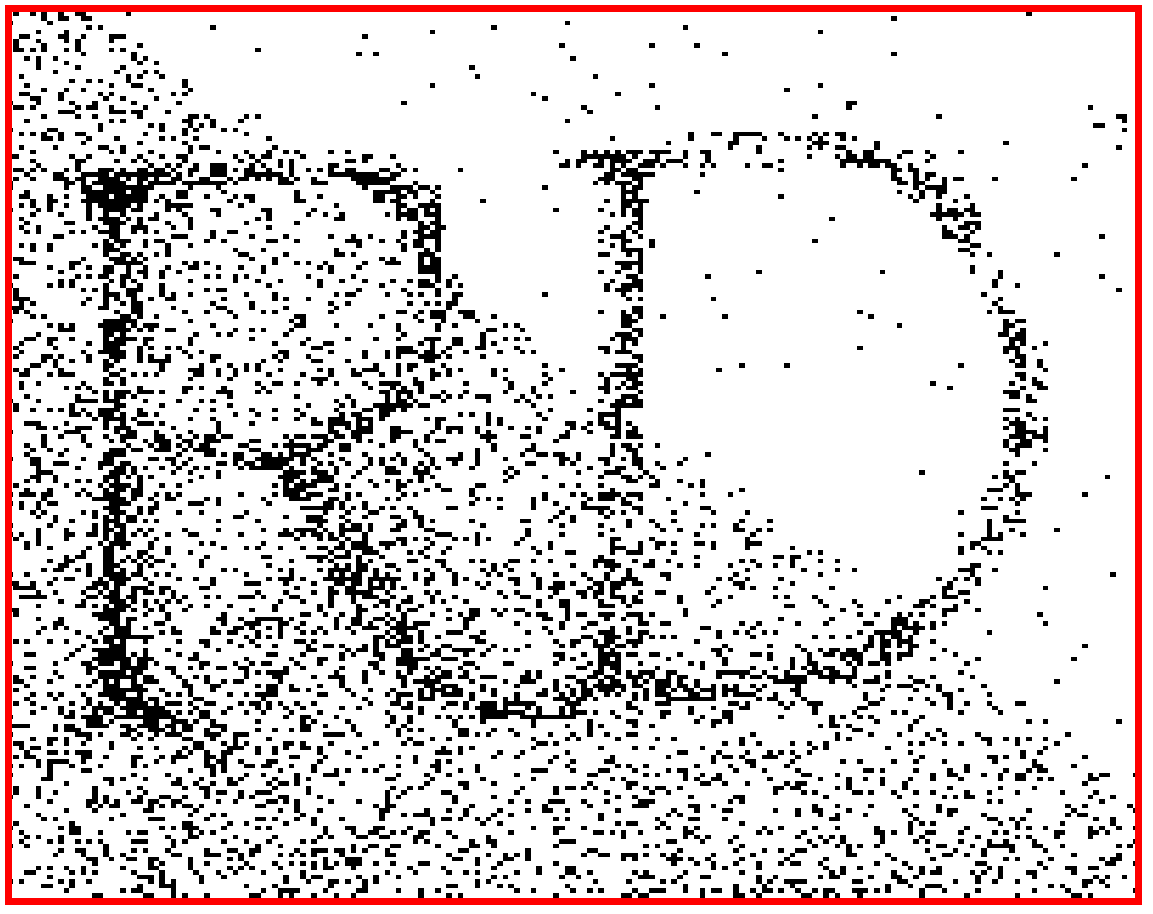} &
    \hspace{-1.5ex}\includegraphics[width=2.8cm]{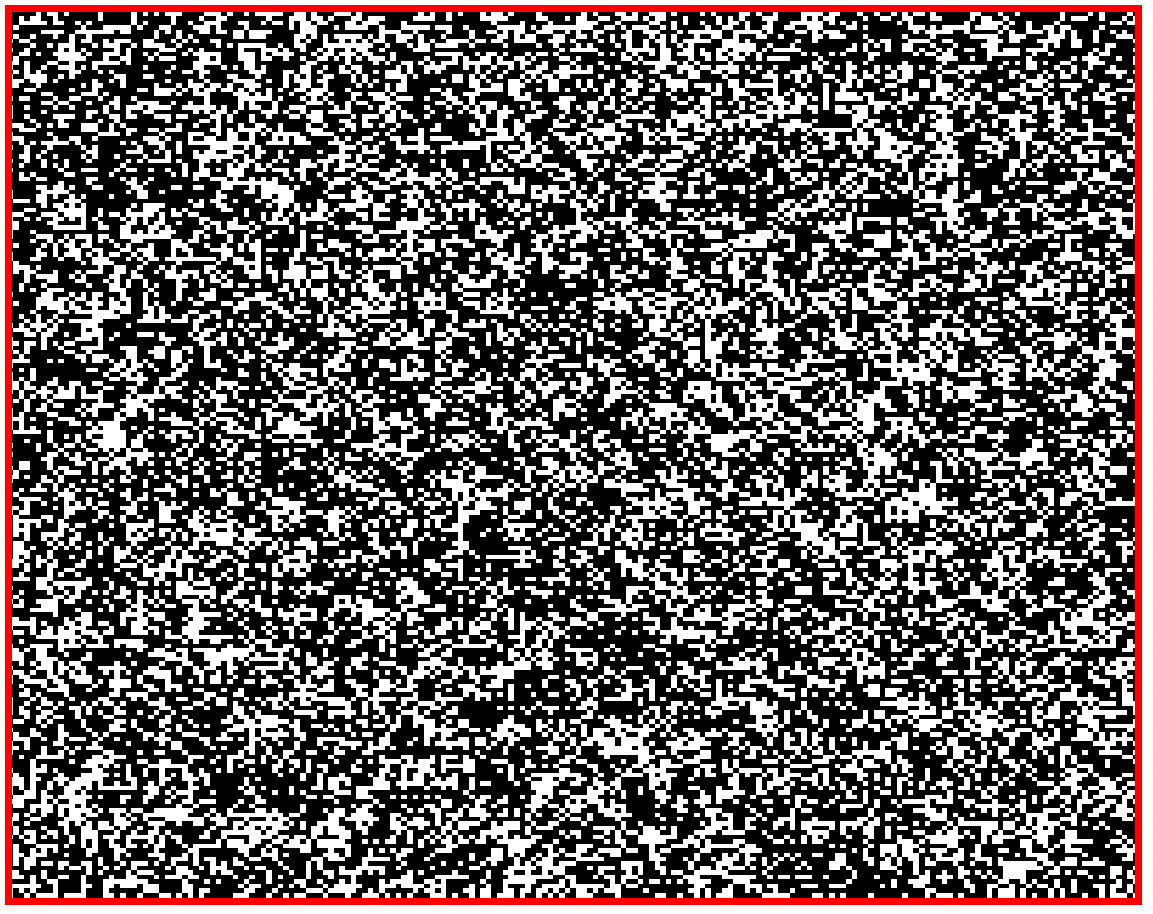} &
    \hspace{-1.5ex}\includegraphics[width=2.8cm]{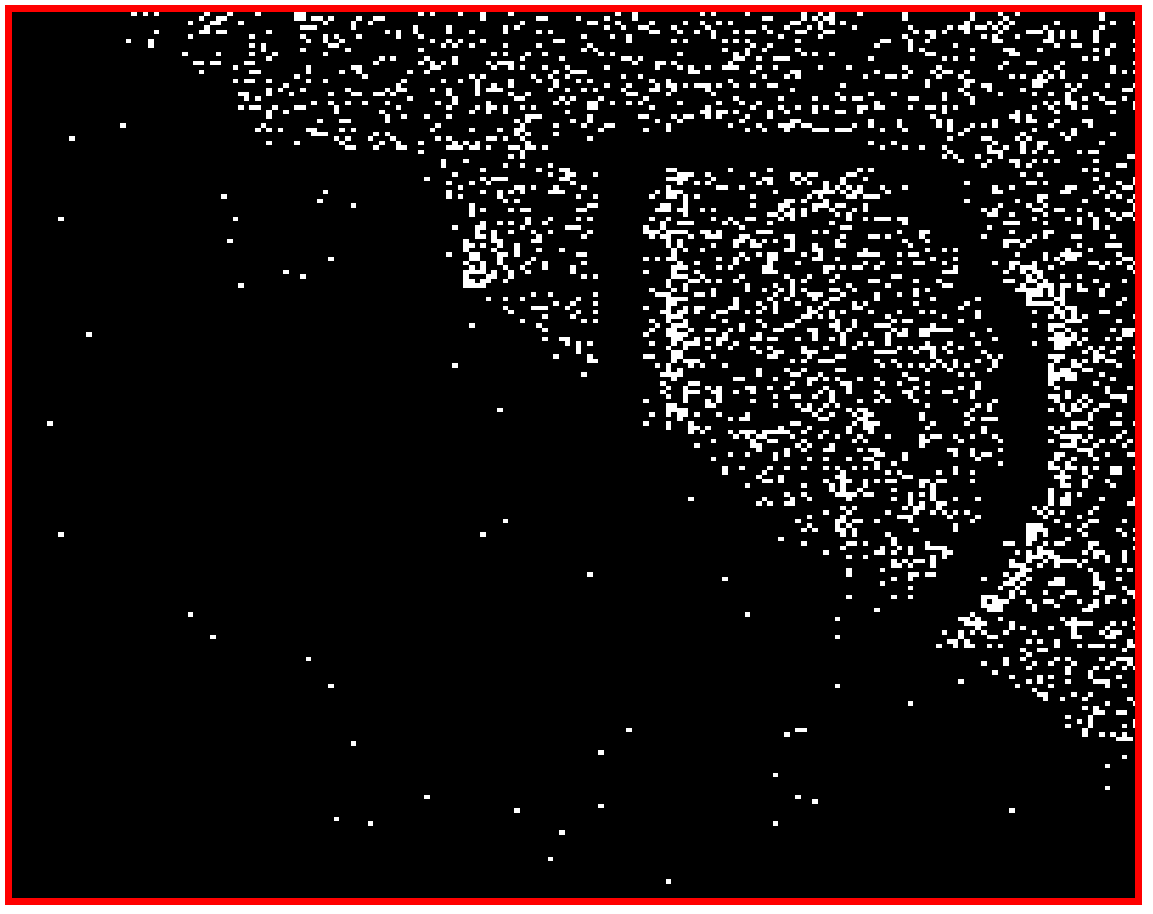} \vspace{-1.0ex}\\
  \footnotesize (a) Observed, $q=3$ &
 \hspace{-1.5ex}\footnotesize (b) Observed, $q=q^*(c)$ &
  \hspace{-1.5ex}\footnotesize (c) Observed, $q=12$ \\
\end{tabular}
\vspace{-1.0ex}
\begin{tabular}{cc}
  \includegraphics[width=4.3cm]{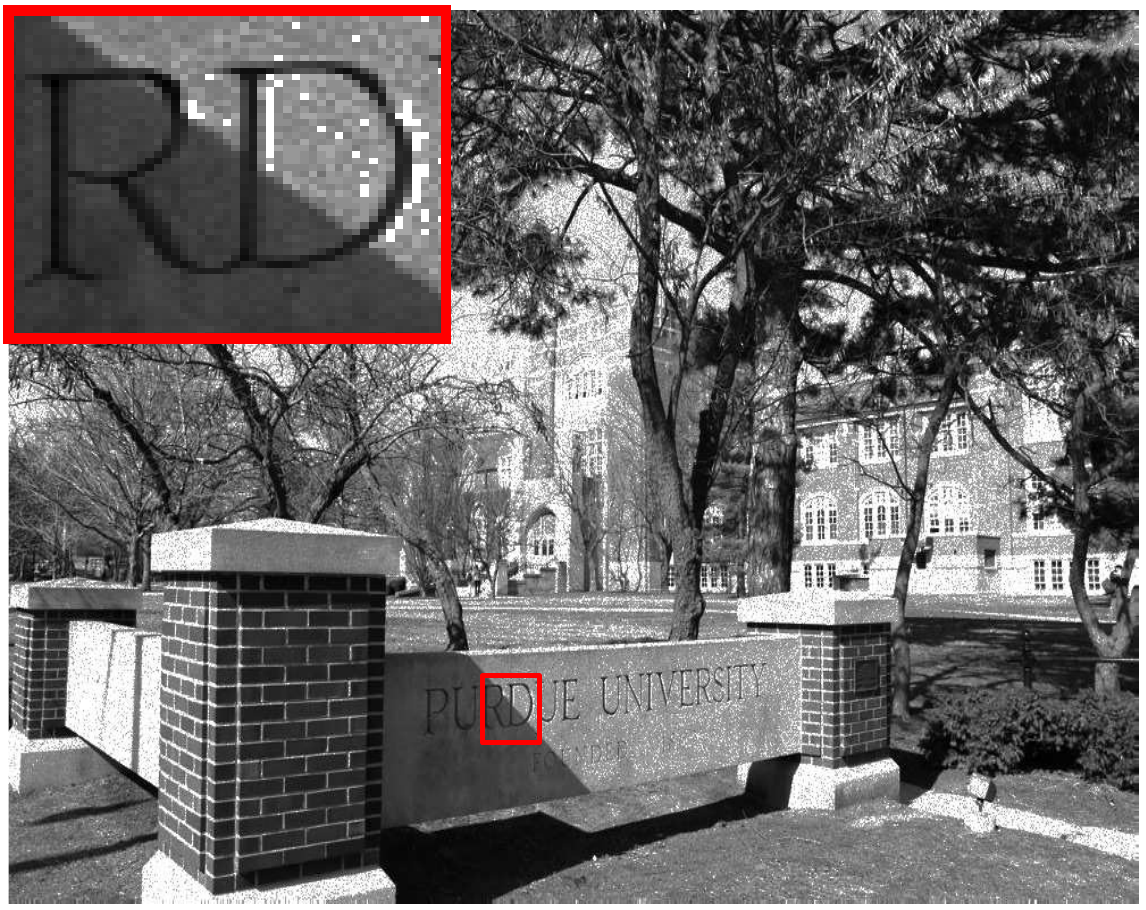} &
  \hspace{-1.5ex}\includegraphics[width=4.3cm]{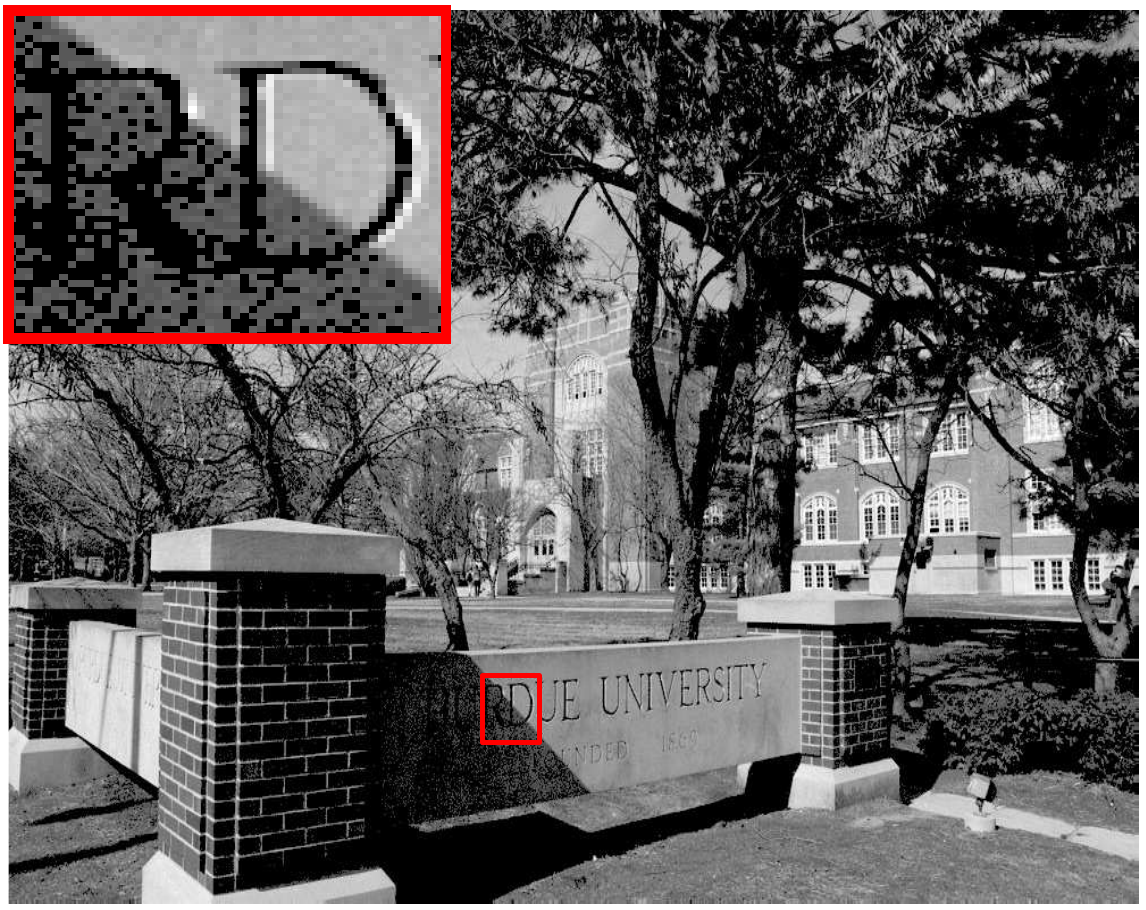}\vspace{-1.0ex} \\
  \footnotesize (d) Reconstruction, $q=3$ &
  \hspace{-1.5ex}\footnotesize (e) Reconstruction, $q=12$ \\
  \includegraphics[width=4.3cm]{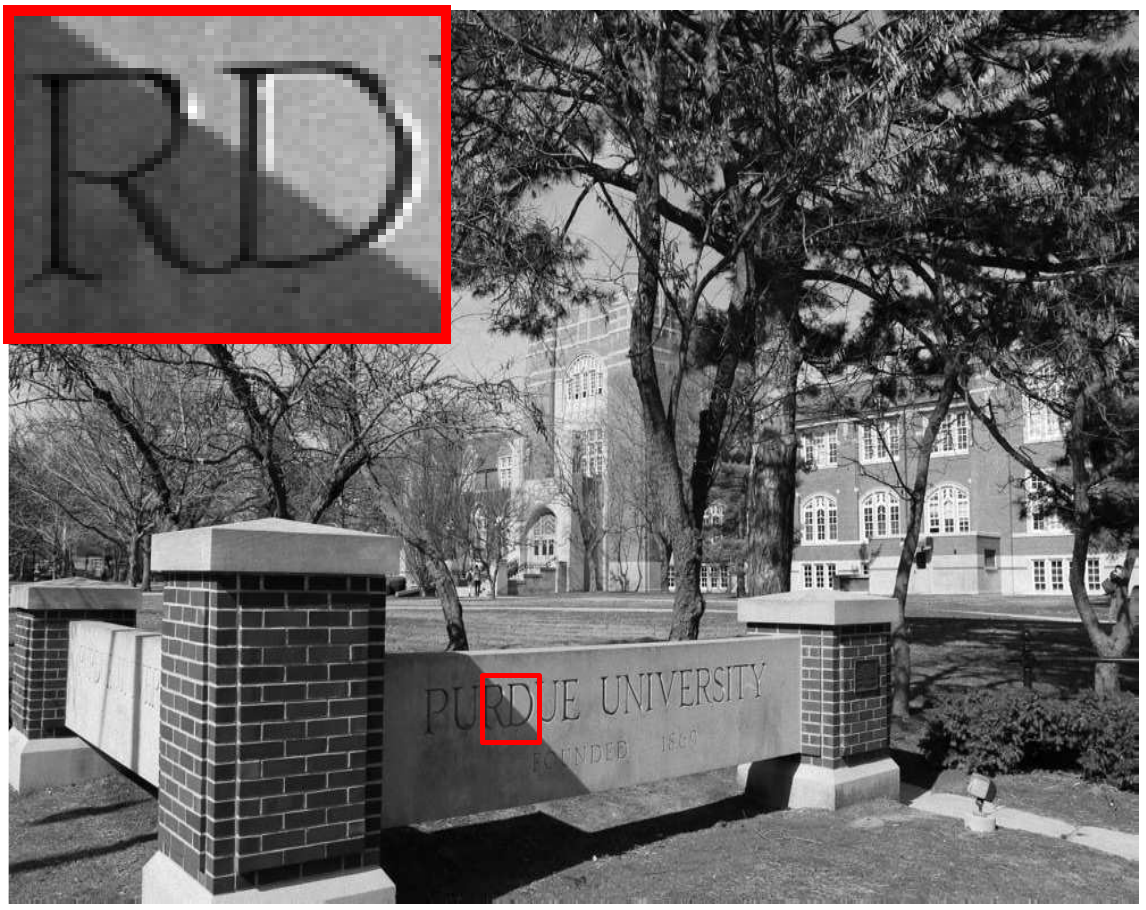} &
  \hspace{-1.5ex}\includegraphics[width=4.3cm]{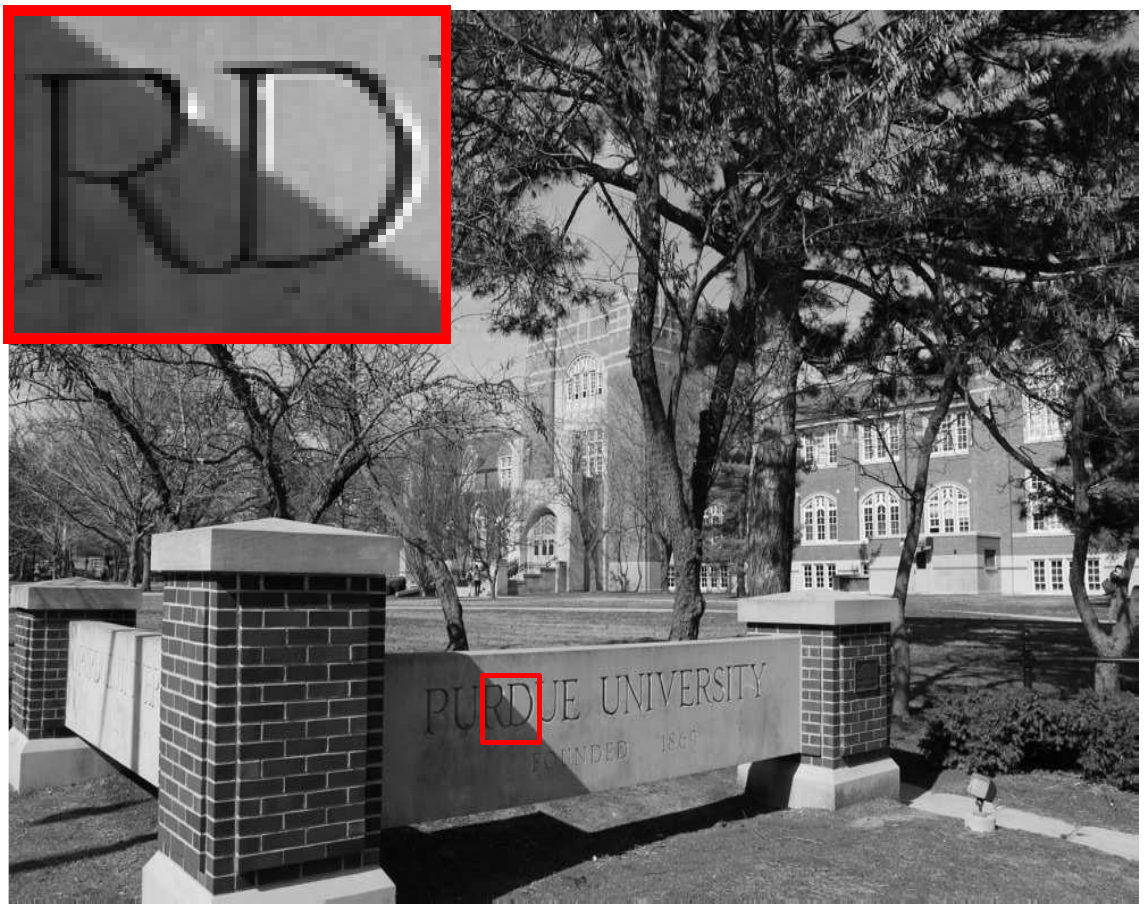} \vspace{-1.0ex}\\
  \footnotesize (d) Reconstruction, $q=q^*(c)$ &
  \hspace{-1.5ex}\footnotesize (e) Ground Truth
\end{tabular}
\vspace{-0.5ex}
\caption{Simulated QIS data and the reconstructed gray-scale images using different thresholds. Top row: The binary measurements obtained using thresholds $q = 3$, $q = q^*(c)$, and $q = 12$. Bottom figures: The maximum likelihood estimates obtained from the binary measurements, with comparison to the ground truth. The results show that our spatially varying threshold $q^*(c)$ offers the best reconstruction. In this experiment, we spatially oversample each pixel by $K = 2
\times 2$ binary bits and we use $T = 25$ independent temporal measurements. }
\label{fig:result1}
\vspace{-2ex}
\end{figure}

\textcolor{black}{Optimal threshold design for QIS is important as it directly affects the dynamic range of an image. \fref{fig:result1} illustrates an example where we simulate the raw binary data acquired by a QIS using a uniform threshold $q$. When $q$ is low, most of the bits in the raw input are ``1''. The reconstructed image is therefore an over-exposed image. On the other hand, when $q$ is high, most of the bits in the raw input are ``0''. The reconstructed image is then under-exposed. In both cases, it is evident from the simulation that a uniform threshold has limited performance. A better way is to allow $q$ to vary spatially so that a pixel (or a group of pixels) has its own threshold value. The result in Figure~\ref{fig:result1}(d) shows the reconstruction result using a spatially varying threshold obtained from our proposed technique, which is clearly better than the uniform thresholds.}

\vspace{-2ex}
\subsection{Scope and Contributions}
The goal of this paper is to present an optimal threshold design methodology and provide theoretical justifications. The two major contributions are summarized as follows.

\begin{table*}
 \centering
 \caption{List of QIS Prototypes and Parameters}
\def\arraystretch{1.3}
\begin{tabular}{c|c|c|c|c|c|c}
  			\hline
  			Camera & Canon 5D CMOS & EMCCD \cite{Andor_2016} & GMAPD \cite{Aull_Schuette_Young_2016} & SPC SPAD \cite{Dutton_Gyongy_Parmesan_2016_1} & SwissSPAD \cite{Antolovic_Burri_Bruschini_2016} &  Fossum QIS \cite{Masoodian_Rao_Ma_2016}   \\
  			\hline
			\hline
			Price & $\$5,000$ & $\$20,000$ & Prototype & Prototype & Prototype  & Prototype   \\
			\hline
  			Resolution & $4096\times2160$ & $1024\times1024$ & $256\times256$ & $320\times240$ & $512\times128$ & $1376\times768$ \\
  			\hline
  			Pixel Pitch ($\mu$m) & {\color{black}$2.3$} & $13$ & 25 &  $8$ & 24 & $3.6$  \\
  			\hline
  			Full-well capacity  & 69 ke- (@ISO100) & {\color{black}180} ke- & - & $56-125$ e- & - & $1-250$ e- \\
  			\hline
  			Frames per second (fps) & 6 & $26-92$ & $8\times 10^3$ &  $2\times 10^4$ & $1.56\times 10^5$ & $1\times 10^3$ \\
  			\hline
  			Sensor data rate & $88.6$ Mbps & 0.48 Gbps & 0.52 Gbps & 1.54 Gbps & 10.24 Gbps & 1 Gbps \\
  			\hline
		\end{tabular}
		 \label{table:QIS}
		 	\vspace{-3.0ex}
\end{table*}

First, we provide a rigorous theoretical analysis of the performance limit of the image reconstruction as a function of the threshold. These results form the basis of our subsequent discussions of the threshold update scheme. Some results are known, e.g., the signal-to-noise ratio is a function of the Fisher Information \cite{Yang_2012, Lu_2013}, but a number of new results are shown. In particular, we show that (i) the maximum likelihood estimate has a closed-form expression in terms of the incomplete Gamma function (Section III.B); (ii) the oracle threshold can be derived in closed-form by maximizing the signal-to-noise ratio (Section III.C); (iii) the image reconstruction has a phase transition behavior (Section IV.A - Section IV.D).

Second, we propose an efficient threshold update scheme based on our theoretical results. The new scheme is a bisection method which iteratively updates the threshold \emph{without} the need of reconstructing the image. By checking whether the proportion of one's and zero's approaches 0.5 in a spatial-temporal block, the threshold is guaranteed to be near optimal. Compared to other existing threshold update schemes such as \cite{Hu_Lu_2012} and \cite{Vogelsang_Stork_2012,Vogelsang_Guidash_Xue_2013,Vogelsang_Stork_Guidash_2014}, the new scheme offers significantly faster rate of convergence (Section IV.E). We also demonstrate how the dynamic range can be extended for high dynamic range (HDR) imaging (Section IV.F).

A preliminary version of this paper was presented in ICIP 2016 \cite{Elgendy_Chan_2016}. This journal version contains significantly more details including complete proofs of major results, more comprehensive comparisons with existing methods, and discussions of HDR imaging.

\section{Background}
\label{sec:background}

\subsection{Current State of QIS}
\label{subsec:currentQIS}
{\color{black}Quanta Image Sensor (QIS) belongs to the family of photon-counting devices. These photon-counting devices have been known for a long time. Some better-known examples are the electron-multiplying charge-coupled device (EMCCD) \cite{Hynecek_2001,Robbins_Hadwen_2003}, single-photon avalanche diode (SPAD) \cite{Dutton_Gyongy_Parmesan_2016, Dutton_Gyongy_Parmesan_2016_1, Antolovic_Burri_Bruschini_2016}, Geiger-mode avalanche photodiode (GMAPD) \cite{Aull_Schuette_Young_2016}, etc. The common feature of these devices is their single photon sensitivity, which makes them useful in medical imaging \cite{Liang_Shen_Camilli_2014,Braga_Gasparini_Grant_2014,Poland_Krstaji_Monypenny_2015}, astronomy \cite{Grubbs_Michell_Samara_2016}, defense \cite{Seitz_Theuwissen_2011}, nuclear engineering \cite{Meng_2006}, depth and reflectivity reconstruction \cite{Shin_Xu_Venkatraman_2016}, ultra-fast low-light tracking \cite{Gyongy_Abbas_Dutton_2017}, and recently in quantum random number generation used in cryptography \cite{Burri_Maruyama_Michalet_2014,Amri_Felk_Stucki_2016}.}

The concept of QIS was first proposed by Fossum in 2005 as a solution for sub-diffraction limit pixels. The sensor was called the digital film sensor, and later the quanta image sensor \cite{Fossum_2011,Ma_Fossum_2015,Masoodian_Rao_Ma_2016}. After the introduction of QIS, researchers in EPFL developed a similar concept called the Gigavision camera \cite{Sbaiz_Yang_Charbon_2009,Yang_Sbaiz_Charbon_2010,Yang_Lu_Sbaiz_2012}. Recently, teams at the University of Edingburgh \cite{Dutton_Parmesan_Holmes_2014,Dutton_Gyongy_Parmesan_2016_1,Dutton_Gyongy_Parmesan_2016} and EPFL \cite{Burri_Maruyama_Michalet_2014,Antolovic_Burri_Hoebe_2016} have made new progresses in QIS using binary single photon detectors. In the industry, Rambus Inc. (Sunnyvale, CA) has developed binary image sensors for high dynamic range imaging \cite{Vogelsang_Stork_2012,Vogelsang_Guidash_Xue_2013,Vogelsang_Stork_Guidash_2014}. Table~\ref{table:QIS} lists several recent QIS prototypes that are available or are currently being developed. As a comparison we also show a Canon 5D Mark III CMOS camera. Among many different features, the most noticeable is the frame rate. For example, SPS SPAD can be operated at 20k fps. SwissSPAD can even achieve 156k fps. Both are significantly faster than a standard CMOS camera.

\subsection{Related Work on Threshold Design}
\label{sub:related}
Existing work on QIS threshold design study can be summarized into three classes of methods.
\begin{itemize}
\item Markov Chain \cite{Hu_Lu_2012}. The Markov Chain method developed by Hu and Lu \cite{Hu_Lu_2012} is a time-sequential update scheme. A Markov Chain probability is used to control how easy the threshold should be increased or decreased. While the method has provable convergence, the threshold of each single photon detector of the QIS has to be updated sequentially in time. In contrast, our proposed method allows a group of single photon detectors to share the same threshold. As a result, our proposed method has significantly faster rate of convergence.
\item Conditional Reset \cite{Vogelsang_Stork_2012,Vogelsang_Guidash_Xue_2013,Vogelsang_Stork_Guidash_2014}. The conditional reset method is a hardware solution proposed by Vogelsang and colleagues. The idea is to take a sequence of images with ascending (or descending) thresholds, and digitally integrate the sequence to form an image. The drawback of the method, besides the additional hardware cost of the per-pixel reset transistors, is the limited quality of the reconstructed image. For the same number of frames, our proposed method produces better images.
{\color{black}
\item Checkerboard Threshold \cite{Yang_2012}. This method constructs a checkerboard of thresholds by alternating two threshold values $q_1$ and $q_2$. The optimality criterion of $q_1$ and $q_2$ is based on minimizing the Cram\'{e}r-Rao lower bound (CRLB) integrated over a range of light intensities, which is essentially an average case result. Our proposed method obtains the optimal threshold for each pixel. This per-pixel optimization has higher reconstruction performance compared to checkerboard threshold.}
\end{itemize}

\subsection{QIS Imaging Model}
\label{subsec:model}
In this subsection we provide an overview of the QIS imaging model. The model has been previously discussed in several papers, e.g., \cite{Yang_Lu_Sbaiz_2012,Chan_Lu_2014,Elgendy_Chan_2016,Chan_Elgendy_Wang_2016}. Readers interested in details can refer to these papers for further explanations.

\subsubsection{Spatial Oversampling} We denote the discrete version of the light intensity as a vector $\vc=[c_0,\ldots,c_{N-1}]^T$, where $n = 0,\ldots,N-1$ specify the spatial coordinates. We assume that $c_n$ is normalized to the range $[0,1]$ for all $n$ so that there is no scaling ambiguity. To model the actual light intensity, we multiply $c_n$ by a constant $\alpha$ to yield $\alpha c_n$, where $\alpha > 0$ is a fixed scalar constant.

Given the $N$-dimensional vector $\vc$, QIS uses $M \gg N$ tiny pixels called \emph{jots} to sample $\vc$. The ratio $K \bydef M/N$ is known as the spatial oversampling factor. The oversampling process is illustrated in \fref{fig:sigprocess}, where it first upsamples the vector $\vc$ by a factor of $K$, and then filters the output by a lowpass filter $\{g_k\}$. Mathematically, the process can be expressed as
\begin{equation}
\vtheta = \alpha \mG\vc,
\end{equation}
where $\vtheta = [\theta_0,\ldots,\theta_{M-1}]^T$ denotes the light intensity sampled at the $M$ jots, and the matrix $\mG$ is defined as
\begin{equation}
\mG=\frac{1}{K} \mI_{N\times N}  \otimes \mathbf{1}_{K \times 1},
\label{eq:G}
\end{equation}
where $\mathbf{1}_{K \times 1}$ is a vector of all ones and $\otimes$ denotes the Kronecker product. Note that the choice of $\mG$ in \eref{eq:G} is the result of simplifying the model by assuming that the lowpass filter is $g_k = 1/K$ for all $k$. \textcolor{black}{This assumption is typically reasonable, because on each QIS jot there is a micro-lens to focus the incident light. Although previous papers, e.g., \cite{Yang_Lu_Sbaiz_2012,Chan_Lu_2014}, do not make such assumption, in this paper we decide to use a simplified $\mG$, for otherwise the theoretical analysis will become very complicated. Nevertheless, in the Supplementary Material we show comparison between a general $\mG$ and the simplified $\mG$. The gap is usually insignificant.}

\begin{figure}[t]
\includegraphics[width=1.0\linewidth]{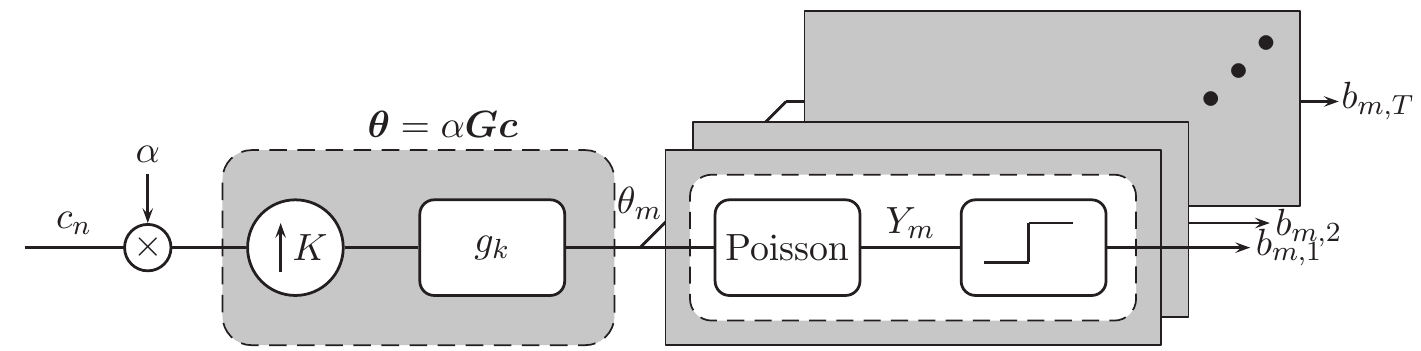}
\caption{Block diagram illustrating the image formation process of QIS.}
\label{fig:sigprocess}
\vspace{-3.0ex}
\end{figure}

\subsubsection{Truncated Poisson Process}
We assume that the operating speed of QIS is significantly faster than the scene motion. Therefore, for a given scene $\vc$ (and also $\vtheta$), we are able to acquire a set of $T$ independent measurements. We illustrate this using the $T$ channels in \fref{fig:sigprocess}.

The oversampled signal $\vtheta$ generates a sequence of Poisson random variables according to the distribution
\begin{equation}\label{eq:prob_Ym}
\Pb(Y_{m,t}=y_{m,t})=\frac{\theta_m^{y_{m,t}} e^{-\theta_m}}{{y_{m,t}}!},
\end{equation}
where $m = 0,1,\ldots,M-1$ denotes the $m$-th jot of the QIS and $t = 0,1,\ldots,T-1$ denotes the $t$-th independent measurement in time. Denoting $q \in \N$ as the quantization threshold, the final observed binary measurement $B_{m,t}$ is a truncation of $Y_{m,t}$:
\begin{equation*}
B_{m,t} =
\begin{cases}
  0, & \mbox{if } Y_{m,t} < q.\\
  1, & \mbox{if } Y_{m,t} \geq q
\end{cases}
\end{equation*}
The probability mass function of $B_{m,t}$ is given by
\begin{equation}
\Pb(B_{m,t}=b_{m,t})
=
\begin{cases}
  \sum\limits_{k=0}^{q-1} \; \frac{\theta_m^k e^{-\theta_m}}{k!},    & \mbox{if } b_{m,t} = 0, \\
  \sum\limits_{k=q}^{\infty} \; \frac{\theta_m^k e^{-\theta_m}}{k!}, & \mbox{if } b_{m,t} = 1.
\end{cases}
\label{eq:prob bm}
\end{equation}

{\color{black} The goal of image reconstruction is to recover the underlying image $\vc$ from the binary measurements $\calB = \{B_{m,t} \;|\; m = 0,\ldots,M-1, \mbox{and}\; t= 0,\ldots,T-1\}$. A pictorial illustration of the reconstruction is shown in Figure~\ref{fig:QIScube}.}

\subsubsection{Properties of Truncated Poisson Processes}
{\color{black} The probability mass function of $B_{m,t}$ in \eref{eq:prob bm} is Bernoulli. However, the right hand side of \eref{eq:prob bm} involves infinite sums which are difficult to interpret. To simplify the equations, we consider the upper incomplete Gamma function ${\Psi_q:\R^{+}\rightarrow[0,1]}$ defined in \cite{Abramowitz_Stegun_1964} as:
\begin{equation*}
\Psi_q(\theta) \bydef \frac{1}{\Gamma{(q)}} \int_{\theta}^{\infty}  t^{q-1} e^{-t} dt, \quad \mathrm{for}\; \theta > 0,\; q \in \N.
\end{equation*}
where $\Gamma(q) = (q-1)!$ is the standard Gamma function. The incomplete Gamma function allows us to rewrite the infinite sums in \eref{eq:prob bm} using the following identity \cite{Abramowitz_Stegun_1964}:}
\begin{equation}
\Psi_q(\theta) = \sum\limits_{k=0}^{q-1} \; \frac{\theta^k}{k!} e^{-\theta}.
\label{eq:Psi Poisson}
\end{equation}
Consequently, the probabilities in \eref{eq:prob bm} become
\begin{align}
\Pb(B_{m,t}=0) &= \Psi_q(\theta_m), \notag\\
\Pb(B_{m,t}=1) &= 1-\Psi_q(\theta_m). \label{eq:prob bm 1}
\end{align}

\begin{figure}[t]
\centering
\includegraphics[width=0.5\textwidth]{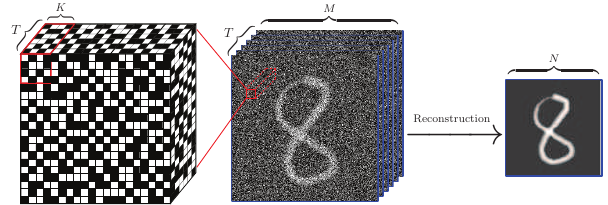}
\caption{{\color{black} Image reconstruction of QIS data. Given the binary bit planes, the reconstruction algorithm estimates the gray-scale image shown on
the right.}}
\label{fig:QIScube}
\vspace{-3.0ex}
\end{figure}

\begin{example}
In the special case of $q=1$, we obtain:
\begin{align*}
\Pb(B_{m,t}=0) &= \frac{1}{\Gamma{(1)}} \int_{\theta_m}^{\infty}  t^{0} e^{-t} dt = e^{-\theta_m},
\end{align*}
which coincides with the results shown in \cite{Yang_Lu_Sbaiz_2012} and \cite{Chan_Lu_2014}.
\end{example}

The incomplete Gamma function $\Psi_q(\theta)$ is a decreasing function of $\theta$ because the first order derivative of $\Psi_q(\theta)$ with respect to $\theta$ is negative:
\begin{equation}\label{eq:diff}
\frac{d}{d\theta}\Psi_q(\theta)=\frac{-\theta^{q-1} e^{-\theta}}{\Gamma(q)} < 0,\quad \forall q\in\N,\textrm{ and } \theta>0.
\end{equation}
The limiting behavior of $\Psi_q(\theta)$ is important. For a fixed $q$, the function $\Psi_q(\theta) \rightarrow 1$ as $\theta \rightarrow 0$ and $\Psi_q(\theta) \rightarrow 0$ as $\theta \rightarrow \infty$. While $\Psi_q^{-1}$ still exists in these situations because $\Psi_q$ is monotonically decreasing, for a given $z$ the value $\Psi_q^{-1}(z)$ could be numerically very difficult to evaluate. To characterize the sets of $\theta$ and $q$ that $\Psi_q$ is (numerically) invertible, we define the \emph{$\theta$-admissible set} and the \emph{$q$-admissible set}.
\begin{definition}
\label{def:admissible set}
The \emph{$\theta$-admissible set} and \emph{$q$-admissible set} of the incomplete Gamma function are
\begin{align}
\Theta_q        &\bydef \{ \theta \;|\; \varepsilon \le \Psi_q(\theta) \le  1-\varepsilon\}, \notag \\
\calQ_\theta   &\bydef \{q \;|\; \varepsilon \le \Psi_q(\theta) \le  1-\varepsilon\},
\end{align}
respectively, where $0 < \varepsilon < \frac{1}{2}$ is a constant.
\end{definition}
More discussions of the incomplete Gamma function can be found in the Supplementary Material.

{\color{black}
\begin{remark}
In this paper, we assume that QIS is noise-free, i.e., the only source of randomness is the truncated Poisson random variable. In real sensors, there will be readout noise, photo-response non-uniformity caused by conversion gain variation, dark count rate (a.k.a. dark current), optical crosstalk and electronic crosstalk. See \cite{Fossum_2016} for details.
\end{remark}}

\section{Optimal Threshold: Theory}

\subsection{Image Reconstruction by MLE}\label{subsec:MLE}
\textcolor{black}{We begin the optimal threshold design by discussing image reconstruction because the optimality of the threshold is measured with respect to the reconstructed image. However, since QIS is a new device, the number of reconstruction methods is limited. A few examples that can be found in the literature are the gradient descent \cite{Yang_Lu_Sbaiz_2012}, dynamic programming \cite{Yang_Sbaiz_Charbon_2009}, ADMM \cite{Chan_Lu_2014}, and Transform-Denoise method \cite{Chan_Elgendy_Wang_2016}, and neural network \cite{Rojas_Luo_Murray_2017}. In this paper, we shall focus on the maximum likelihood estimation (MLE) approach as it provides closed-form expressions.}

Given $\calB$, MLE solves the following optimization problem:
\begin{align}
\vchat
&\overset{(a)}{=} \argmax{\vc} \;\; \prod_{t=0}^{T-1}\prod_{m=0}^{M-1} \Pb[B_{m,t} = 1 \,;\, \theta_{m}]^{b_{m,t}} \notag \\
&\hspace{3.5cm} \times \Pb[B_{m,t} = 0 \,;\, \theta_m]^{1-b_{m,t}} \notag \\
&\overset{(b)}{=} \argmax{\vc} \;\; \sum_{t=0}^{T-1}\sum_{m=0}^{M-1} \Big\{ b_{m,t} \log (1-\Psi_q(\theta_m)) \notag \\
&\hspace{3.5cm} + (1-b_{m,t})\log \Psi_q(\theta_m)  \Big\},
\label{eq:MLE 1}
\end{align}
subject to the constraint that $\vtheta = \alpha\mG\vc$. Here, the right hand side of $(a)$ is the likelihood function of a Bernoulli random variable, and $(b)$ follows from taking the logarithm. With the $\mG$ defined in \eref{eq:G}, we can partition $\calB$ into $N$ blocks $\{\calB_{1},\ldots,\calB_N\}$ where each block is
$$\calB_n \bydef \{B_{Kn+k,t} \;|\; k = 0,\ldots,K-1, t=0,\ldots,T-1\}.$$
Then, the pixel $\chat_n$ can be estimated as follows.
\begin{proposition}[Closed-form ML Estimate]
\label{prop:mle solution}
The solution of the MLE in \eref{eq:MLE 1} is
\begin{equation}
\chat_n = \frac{K}{\alpha}\Psi_q^{-1}\left(1-\frac{S_n}{KT}\right),
\label{eq:mle solution}
\end{equation}
where $S_n \bydef \sum_{t=0}^{T-1}\sum_{k=0}^{K-1} B_{Kn+k,t}$ is the sum of bits in the $n$-th block $\calB_n$.
\end{proposition}
\begin{proof}
See \cite{Chan_Elgendy_Wang_2016}.
\end{proof}

\subsection{Signal-to-Noise Ratio of ML Estimate}
In order to determine the optimal threshold, we need to quantify the performance of the ML estimate. The performance metric we use is the signal-to-noise ratio of the ML estimate at every pixel $\chat_n$. Considering each $\chat_n$ individually is allowed here because they are independently determined according to \eref{eq:mle solution}. For notation simplicity we drop the subscript $n$ in the subsequent discussions.

\begin{definition}
The signal-to-noise ratio (SNR) of the ML estimate $\chat$ is defined as
\begin{equation}\label{eq:snr}
  \mathrm{SNR}_q(c) \bydef 10\log_{10}\frac{c^2}{\E[(\chat - c)^2]},
\end{equation}
where the expectation is taken over the probability mass function of the binary measurements in \eref{eq:prob bm 1}.
\end{definition}

The difficulty of working with $\mathrm{SNR}_q(c)$ is that it does not have a simple closed-form expression. In view of this, Lu \cite{Lu_2013} showed that the SNR is asymptotically linear to the log of the Fisher Information.
\begin{proposition}
\label{prop:snr}
{\color{black} As $KT \rightarrow \infty$,
\begin{equation} \label{eq:SNR}
\mathrm{SNR}_q(c) \approx 10\log_{10}\big(c^2 I_q(c)\big) + 10\log_{10} KT,
\end{equation}
where $I_q(c)\bydef \E_{B}\left[\frac{-\partial^2}{\partial c^2} \mathrm{log} \; \Pb(B=b;\theta)\right]$ is the Fisher Information measuring the amount of information that the random variable $B$ carries about the unknown value $c$.}
\end{proposition}
\begin{proof}
See \cite{Lu_2013}.
\end{proof}

While the asymptotic result shown in Proposition~\ref{prop:snr} has significantly simplified the SNR, we still need to determine the Fisher Information. The following proposition gives a new result of the Fisher Information with arbitrary $q$.
\begin{proposition}\label{prop:Fisher}
The Fisher Information $I_q(c)$ of the probability mass function in \eref{eq:prob bm 1} under a threshold $q$ is:
\begin{equation}
I_q(c)=\left(\frac{\alpha }{K}\right)^2\frac{e^{-2\left(\frac{\alpha c}{K}\right)}\left(\frac{\alpha c}{K}\right)^{2q-2}}{\Gamma^2(q)\Psi_q\left(\frac{\alpha c}{K}\right)\left(1-\Psi_q\left(\frac{\alpha c}{K}\right)\right)}.
\label{eq:fisher}
\end{equation}
\end{proposition}

\begin{proof}
See Appendix~\ref{proof:Fisher}.
\end{proof}

Substituting \eref{eq:fisher} into \eref{eq:SNR}, we observe that the SNR can be approximated as
\begin{align}
\mathrm{SNR}_q(c) &\approx 10\log_{10} \frac{KT e^{-2\left(\frac{\alpha c}{K}\right)}\left(\frac{\alpha c}{K}\right)^{2q}}{\Gamma(q)^2\Psi_q\left(\frac{\alpha c}{K}\right)\left(1-\Psi_q\left(\frac{\alpha c}{K}\right)\right)},
\label{eq:SNR approx}
\end{align}
which is characterized by the unknown pixel value $c$, the threshold $q$, the spatial oversampling ratio $K$ and the number of temporal measurements $T$. To understand the behavior of \eref{eq:SNR approx}, we show in Figure~\ref{fig:Fisher} $\mathrm{SNR}_q(c)$ as a function of $c$ for different thresholds $q\in\{1,\ldots,16\}$. For a fixed $q$, $\mathrm{SNR}_q(c)$ is a convex function with a unique maximum. The goal of optimal threshold design is to determine a $q$ which maximizes $\mathrm{SNR}_q(c)$ for a fixed $c$.

\begin{figure}[t]
\centering
\includegraphics[width=0.40\textwidth]{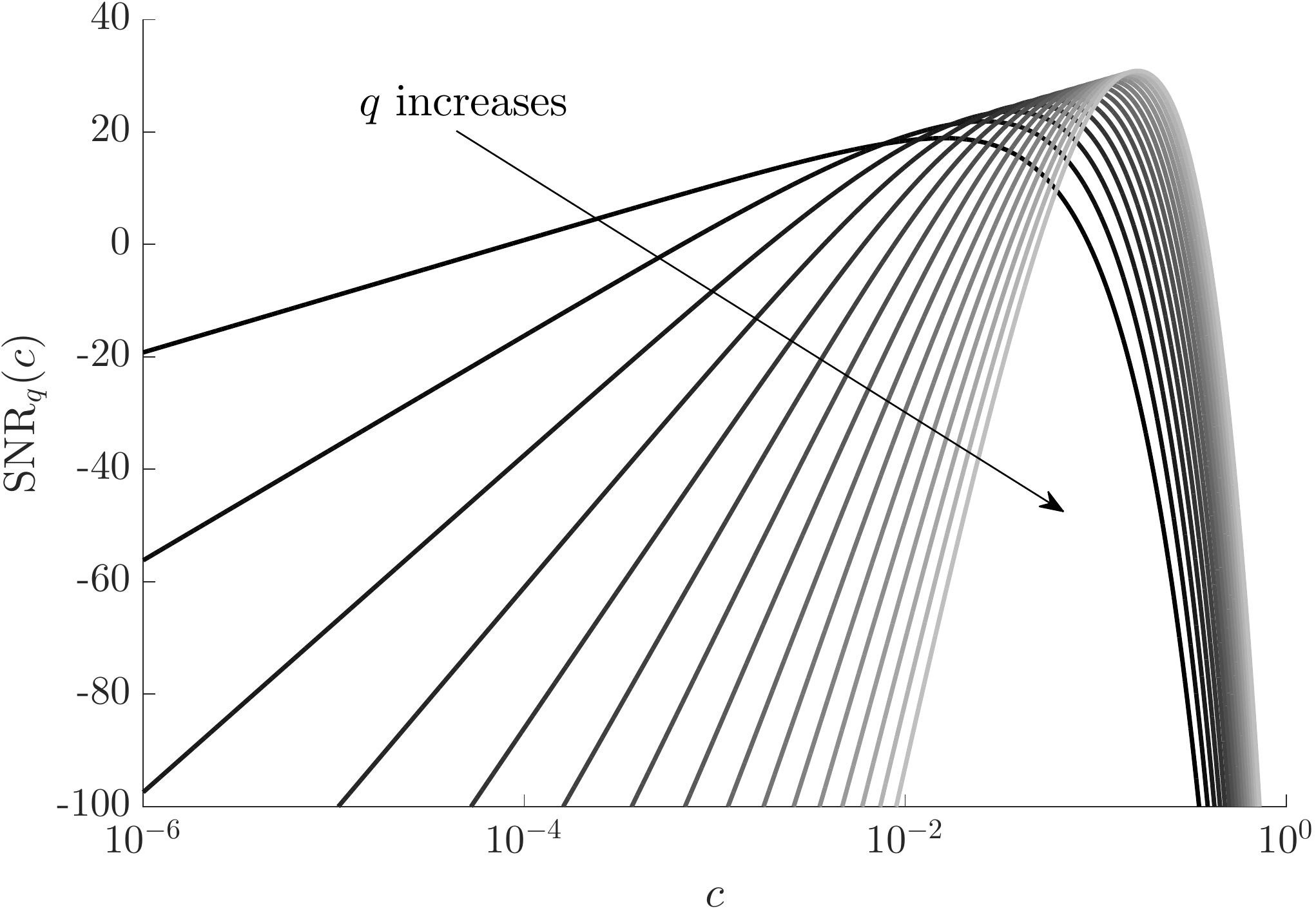}
\vspace{-2ex}
\caption{$\mathrm{SNR}_q(c)$ for different thresholds $q\in\{1,\ldots,16\}$. In this experiment, we set $\alpha=400$, $K=4$, and $T=30$. For fixed $q$, $\mathrm{SNR}_q(c)$ is always a convex function.}
\label{fig:Fisher}
\vspace{-2.0ex}
\end{figure}

\begin{remark}
The $\mathrm{SNR}_q(c)$ in \eref{eq:SNR approx} can also be derived from a concept in the device literature called the \emph{exposure-referred SNR} \cite{Fossum_2013}. See Supplementary Material for discussions.
\end{remark}

\subsection{Oracle Threshold}
We now discuss the optimal threshold design in the oracle setting. We call the result oracle because the optimal threshold depends on the unknown pixel intensity $c$. The practical threshold design scheme will be discussed in Section~\ref{sec:practical threshold}.

Using the definition of the signal-to-noise ratio, the optimal threshold is determined by maximizing $\mathrm{SNR}_q(c)$ with respect to $q$:
\begin{equation}
q^{\ast} = \argmax{q\in\N} \;\; \textrm{SNR}_q(c) = \argmax{q\in \N} \;\; \log(c^2 I_q(c)).
\label{eq:q oracle optimization problem}
\end{equation}
The second equality follows from Proposition~\ref{prop:snr}. Substituting \eref{eq:fisher} yields an expression of the right hand side of \eref{eq:q oracle optimization problem}. To further simplify the expression we derive the following lower bound.

\begin{proposition}\label{prop:lowerbound}
The function $\log(c^2 I_q(c))$ is lower bounded as follows.
\begin{equation*}
\log(c^2 I_q(c)) \ge \underset{ \bydef L_q(c)}{\underbrace{2\left(\log 2-\frac{\alpha c}{K} + q \log \frac{\alpha c}{K} - \log \Gamma(q)\right)}}.
\end{equation*}
\end{proposition}
\begin{proof}
See Appendix~\ref{proof:lowerbound}.
\end{proof}
Using this lower bound, we can derive the optimal threshold $q$ as follows \footnote{\textcolor{black}{Straightly speaking, the result shown in Proposition~\ref{prop:optimal q} is a ``near-optimal'' result because we are minimizing the lower bound. From our experience, the gap between the near-optimality and the exact optimality is typically insignificant.}}.
\begin{proposition}
\label{prop:optimal q}
The optimal threshold $q^*(c)$ is
\begin{equation}
q^*(c) = \argmax{q \in \N} \; L_q(c) = \left\lfloor \frac{\alpha c}{K} \right\rfloor +1,
\end{equation}
where $\lfloor \cdot \rfloor$ denotes the flooring operator that returns the largest integer smaller than or equal to the argument.
\end{proposition}
\begin{proof}
See Appendix~\ref{proof:oracle q}.
\end{proof}
The result of Proposition~\ref{prop:optimal q} is important as it states that the oracle threshold is \emph{exactly the same} as the light intensity $\alpha c / K$. The flooring operation and the addition of a constant 1 are not crucial here because they are only used to ensure that $q$ is an integer. In \cite{Hu_Lu_2012}, a special where $\alpha = 1$ was demonstrated experimentally. Proposition~\ref{prop:optimal q} now provides a theoretical justification.

\section{Optimal Threshold: Practice}
\label{sec:practical threshold}
{\color{black}The oracle threshold derived in the previous section provides a theoretical foundation but is practically infeasible as it requires knowledge of the ground truth $c$. In this section, we present an alternative solution by relaxing the optimality criteria. Our strategy is to consider a set of thresholds which are close to the oracle threshold $q^*(c)$, and show that they are asymptotically unbiased when the number of observed bits approaches infinity (Section IV.A). This result will allow us to characterize the estimate $\widehat{c}$ (Section IV.B). We will then show that there exists a phase transition region where the asymptotic unbiasedness is maintained as $q$ stays within a certain range around $q^*(c)$, and is lost rapidly as $q$ falls outside this range (Section IV.C - IV.D). Based on these observations, we will present a practical threshold update scheme (Section IV.E).}

\subsection{Asymptotic Unbiasedness}
\label{subsec:phaseq}
\textcolor{black}{In order to derive an alternative threshold that does not require the ground truth, we start by reconsidering the ML estimate $\chat$ in Proposition~\ref{prop:mle solution}. For a spatial-temporal block $\calB=\{B_{k,t}\,|\,0\leq k < K-1, 0 \leq t < T-1\}$, the ML estimate $\chat$ satisfies the condition
\begin{equation}
\Psi_q\left(\frac{\alpha \chat}{K}\right) = 1-\frac{S}{KT},
\label{eq:Psi equation}
\end{equation}
where $S = \sum_{k,t} B_{k,t}$ is the sum of bits in $\calB$. The right hand side of this equation is an important quantity. We denote it as
\begin{equation}
\gamma_q(c) \bydef 1-\frac{S}{KT}.
\label{eq:bit density}
\end{equation}
In the device literature (e.g., \cite{Fossum_2013}), the term $1-\gamma_q(c)$ is known as the \emph{bit-density} as it is the proportion of ones in $\calB$. Note that $\gamma_q(c)$ is a random variable because $S$ is the sum of $KT$ i.i.d. random binary bits. Therefore, if we want to understand \eref{eq:Psi equation}, we must first derive the the mean and variance of $\gamma_q(c)$.}
\begin{proposition}
\label{prop:mean and var of gamma}
The mean and variance of $\gamma_q(c)$ are
\begin{align}
\E[\gamma_q(c)] &= \Psi_q\left(\frac{\alpha c}{K}\right), \;\; \mbox{and} \notag \\
\Var[\gamma_q(c)] &= \frac{1}{KT}\Psi_q\left(\frac{\alpha c}{K}\right)\left[1-\Psi_q\left(\frac{\alpha c}{K}\right)\right],
\end{align}
respectively.
\end{proposition}
\begin{proof}
See Appendix~\ref{proof:gamma mean}.
\end{proof}
{\color{black} We can now look at the asymptotic behavior of $\gamma_q(c)$ to see if it offers any insight about the optimal threshold.
Applying the strong law of large number to $S/KT$, we can show that as $KT \rightarrow \infty$,
\begin{equation}
\gamma_q(c)=1-S/KT  \overset{a.s.}{\rightarrow} 1-\E[B_{k,t}] = \Psi_q(\alpha c /K).
\label{eq:gamma q c converge}
\end{equation}
Going back to \eref{eq:Psi equation}-\eref{eq:bit density}, the ML estimate $\widehat{c}$ should have the expectation:
\begin{align}
\E[\,\chat\,]
&\overset{(a)}{=} \frac{K}{\alpha}\E\left[\Psi_q^{-1}\left(\gamma_q(c)\right)\right] \notag \\
&\overset{(b)}{\rightarrow}  \frac{K}{\alpha}\Psi_q^{-1} \Psi_q\left(\frac{\alpha c}{K}\right)
\overset{(c)}{=} c. \label{eq:E chat converge to c}
\end{align}
where (a) follows from the definition of $\chat$, (b) follows from \eref{eq:gamma q c converge}, and (c) holds because $\Psi_q$ and $\Psi_q^{-1}$ cancels each other.}

{\color{black} What is the implication of \eref{eq:E chat converge to c}? It shows that the ML estimate $\chat$ is asymptotically unbiased. That is, as the number of independent measurements grows, the estimate $\chat$ approaches to the ground truth $c$. In other words, as long as $KT$ is large enough, the random variable $\chat$ would be an accurate estimate of the ground truth. How can this be used to determine the threshold $q$? Let us look at $\calQ_\theta$.}

{\color{black}
\subsection{Set of Admissible Thresholds $\calQ_{\theta}$}
The result in \eref{eq:Psi equation}-\eref{eq:E chat converge to c} shows that for a given $S$ (or equivalently $\gamma_q(c)$), the ML estimate can be found by
\begin{equation}
\widehat{c} = \frac{K}{\alpha} \Psi_q^{-1}\left(\gamma_q(c)\right).
\label{eq:chat1}
\end{equation}
When this happens, the $\widehat{c}$ given by \eref{eq:chat1} is asymptotically unbiased. However, the inversion $\Psi_q^{-1}$ is not always allowed. There is a set of $q$'s that can make $\Psi_q$ invertible, which is defined as $\calQ_\theta$ in Definition~\ref{def:admissible set}. The following proposition relates $\calQ_\theta$ to $\gamma_q(c)$.

\begin{proposition}
\label{prop:delta}
Let $0 < \delta < 1$ be a constant. Then, for any
\begin{equation}
q \in \calQ_\theta \bydef \left\{q \;\Big|\; 1-\left(\frac{\delta}{2}\right)^{\frac{1}{KT}} \le \Psi_q(\theta) \le \left(\frac{\delta}{2}\right)^{\frac{1}{KT}}\right\},
\label{eq:delta}
\end{equation}
the random variable $\gamma_q(c)$ will not attain 0 or 1 with probability at least $1-\delta$, i.e.,
\begin{equation*}
\Pb[0 < \gamma_q(c) < 1] > 1-\delta.
\end{equation*}
In this case, the ML estimate $\widehat{c}$ is uniquely defined by \eref{eq:chat1}.
\end{proposition}

\begin{proof}
See Appendix \ref{proof:delta}.
\end{proof}

Before we proceed, let us look at some rough magnitude of the parameters in the following example.
\begin{example}
\label{example:q admissible set}
Let the ground truth pixel value be $c = 0.5$. The sensor parameters are set as $T = 50$, $K = 4$, $\alpha = 300$. For a constant $\delta = 2\times 10^{-4}$, the tolerance level is $\varepsilon = 1-(\delta/2)^{1/KT} = 0.045$. Therefore, as long as $q \in \{q \,|\, 0.045 \le \Psi_q(\theta) \le 1-0.045\}$, which is the set $\{q \;|\; 28 \le q \le 48\}$, the probability that $\gamma_q(c)$ equals to 0 or 1 is upper bounded by $\delta = 2\times 10^{-4}$.
\end{example}
}

\begin{figure}[t]
\centering
\includegraphics[width=0.4\textwidth]{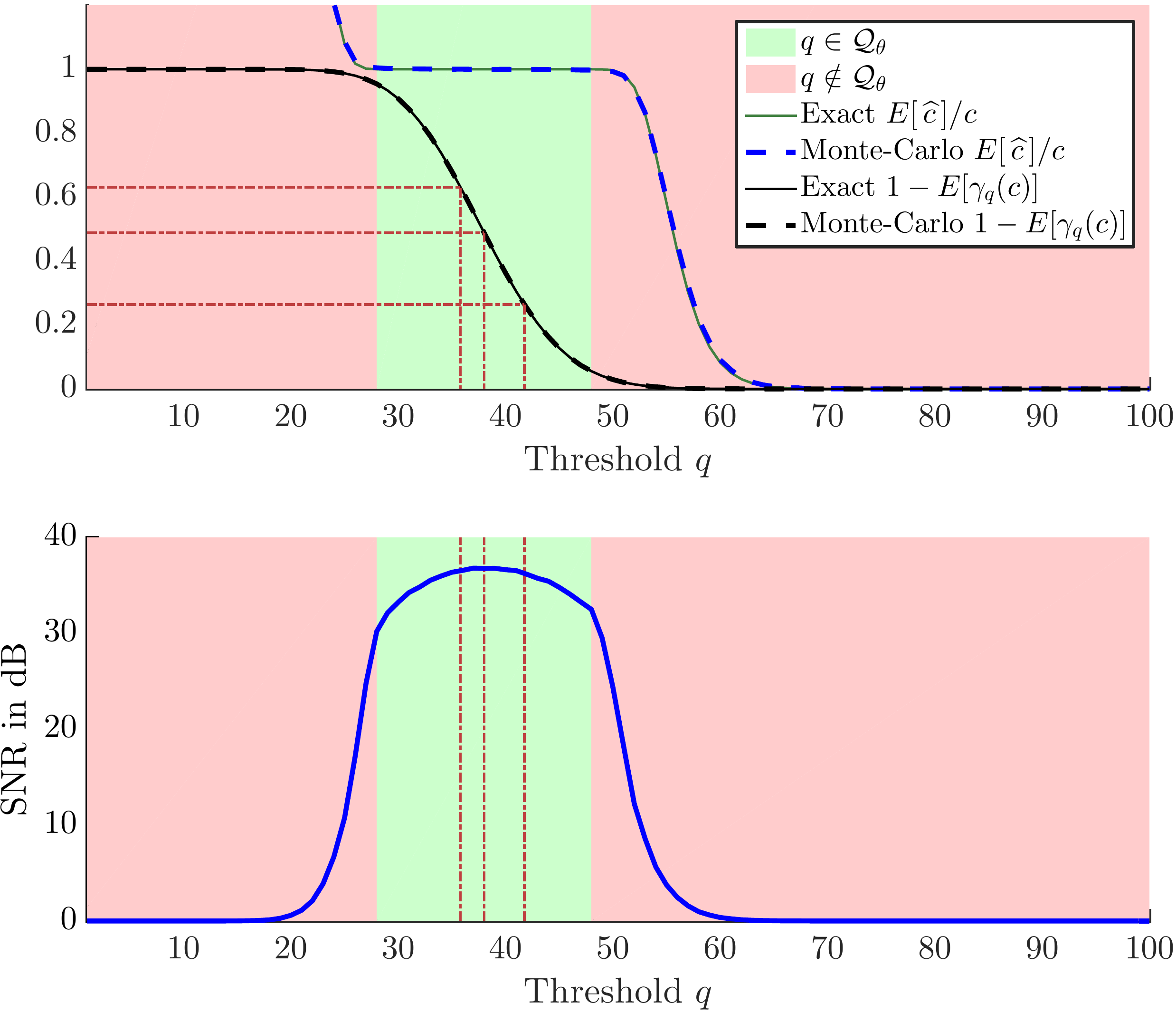}
\caption{Phase transition of the ML estimate and its relationship to the average bit density $1-\E[\gamma_q(c)]$. The red region is where it is impossible to recover $c$, whereas the green region is where we can have perfect recovery. }
\label{fig:c_q}
\vspace{-3.0ex}
\end{figure}

\vspace{-2ex}
\subsection{Gap between $\calQ_\theta$ and $q^*$}
\textcolor{black}{The result in the previous subsection shows that as long as $q \in \calQ_{\theta}$, the ML estimate is asymptotic unbiased. However, how is a $q \in \calQ_{\theta}$ compared to the oracle threshold $q^*$? We answer this question in three parts.}

\textcolor{black}{First, does an asymptotically unbiased estimate maximize the SNR? The answer is no, because Proposition~\ref{prop:optimal q} states that if $q^*$ is the optimal threshold, then $\mathrm{SNR}_{q^*}(c) \ge \mathrm{SNR}_q(c)$ for any $q \not= q^*$. Therefore, moving from the exact optimal $q^*$ to an asymptotically unbiased threshold is a relaxation of the optimality criteria.}

\textcolor{black}{If asymptotic unbiasedness is a relaxed optimality criteria, how much SNR drop will there be if we choose a $q \in \calQ_\theta$ but not necessarily $q = q^*$? We show in \fref{fig:c_q} the plot of a typical experiment with setup discussed in Example~\ref{example:q admissible set}. As shown in the figure, the green zone is the set $\calQ_{\theta} = \{q \;|\; 28 \le q \le 48\}$, or equivalently $\calQ_\theta = \{q \,|\, 0.045 \le \Psi_q(\theta) \le 0.9955\}$. For any $q$ in this $\calQ_{\theta}$, the reconstruction has a SNR at least 30dB. If we further tighten $\calQ_{\theta}$ so that $\calQ_{\theta} = \{q \;|\; 35 \le q \le 42\}$, or equivalently $\calQ_\theta = \{q \,|\, 0.25 \le \Psi_q(\theta) \le 0.6\}$, the SNR stays in the range $36.15\mathrm{dB} \le \mathrm{SNR}_q(c) \le 36.65\mathrm{dB}$, which is reasonably narrow. }

\textcolor{black}{How tight should $\calQ_\theta$ be? Ideally we want $\calQ_\theta$ to be as tight as possible. But knowing the fact that the incomplete Gamma function has a rapid transition (See the black line in \fref{fig:c_q}), $\calQ_\theta$ can be much wider. In fact, we can choose $\calQ_\theta$ such that $1-\gamma_q(c)$ stays close to 0.5, so that we are guaranteed to obtain a near optimal threshold. From an information theoretic point of view, $1-\gamma_q(c) \approx 0.5$ is where the bit density attains the maximum information --- if $q$ is too high then most bits become 0 whereas if $q$ is too low then most bits become 1. It is maximum when $q$ leads to 50\% zeros and 50\% ones. \footnote{\textcolor{black}{The exact optimal value of $1-\gamma_q(c)$ at $q^*$ is slightly lower than 0.5 due to the nonlinearity of the Gamma function. See Supplementary Material for additional discussion.}}
}

\subsection{Phase Transition Phenomenon}
{\color{black} We can now point out a very interesting phenomenon in \fref{fig:c_q}. In the upper plot of \fref{fig:c_q} we show two sets of curves: blue curves (solid and dotted), and black curves (solid and dotted). The black curves represent the ratio $\E[\,\chat\,]/c$, and the black curves represent the average bit density $1-\E[\gamma_q(c)]$. For both sets of curves, we use dotted lines to illustrate the Monte-Carlo simulation using 10,000 random samples, where each sample refers to a spatial-temporal block $\calB_n$ containing $KT = 200$ binary bits. Notice that these dotted lines overlap exactly with their expectations, and hence \eref{eq:Psi equation}-\eref{eq:E chat converge to c} are valid.

Let us take a closer look at the blue curve $\E[\,\chat\,]/c$. Let $\calQ_\theta = \{q \;|\; q_L \le q \le q_H\}$, where $q_L$ and $q_H$ are the smallest and the largest integers in $\calQ_\theta$ respectively. There are three distinct phases:
\begin{densitemize}
\item When $q< q_L$, the threshold is low and so most bits become 1. Therefore, $\gamma_c(q) \rightarrow 0$ and hence $\chat \rightarrow \infty$. Thus, $\E[\,\chat\,]/c \rightarrow \infty$  as $q$ decreases.
\item When $q > q_H$, the threshold high and so most bits become 0. Therefore, $\gamma_c(q) \rightarrow 1$ and hence $\chat \rightarrow 0$. Thus, $\E[\,\chat\,]/c \rightarrow 0$ as $q$ increases.
\item When $q_L \le q \le q_H$, the ML estimate $\chat$ is asymptotically unbiased. Therefore, $\E[\,\chat\,]/c = 1$.
\end{densitemize}
Essentially, \fref{fig:c_q} demonstrates a phase transition behavior of the threshold. Such phase transition exists because $\Psi_q$ is only invertible when $q \in \calQ_\theta$.}

\subsection{Bisection Threshold Update Scheme}\label{subsec:bisection}

{\color{black}Now we present a practical threshold update scheme. As we discussed in Section IV.C, the oracle threshold $q^*$ can be obtained when bit density $\gamma_q(c)$ is close to 0.5. Therefore, a practical procedure to determine $q$ is to sweep through a range of $q$ until the bit density reaches 0.5. To achieve this objective, we propose a bisection method illustrated in \fref{fig:bisection} and Algorithm~\ref{alg:Bisection}. Starting with initial thresholds $q_A$ and $q_B$, we check whether the bit density satisfies $1-\gamma_{q_A} > 0.5$ and $1-\gamma_{q_B} < 0.5$. If this is the case, then we find a mid point $q_M = (q_A+q_B)/2$ and check whether $1-\gamma_{q_M}$ is greater or less than 0.5. If $1-\gamma_{q_M} > 0.5$, we replace $q_A$ by $q_M$, otherwise we replace $q_B$ by $q_M$. The process repeats until $1-\gamma_{q_M}$ is sufficiently close to 0.5.

\begin{figure}[t]
\centering
\begin{tabular}{cccc}
\multicolumn{4}{c}{\includegraphics[width=0.375\textwidth]{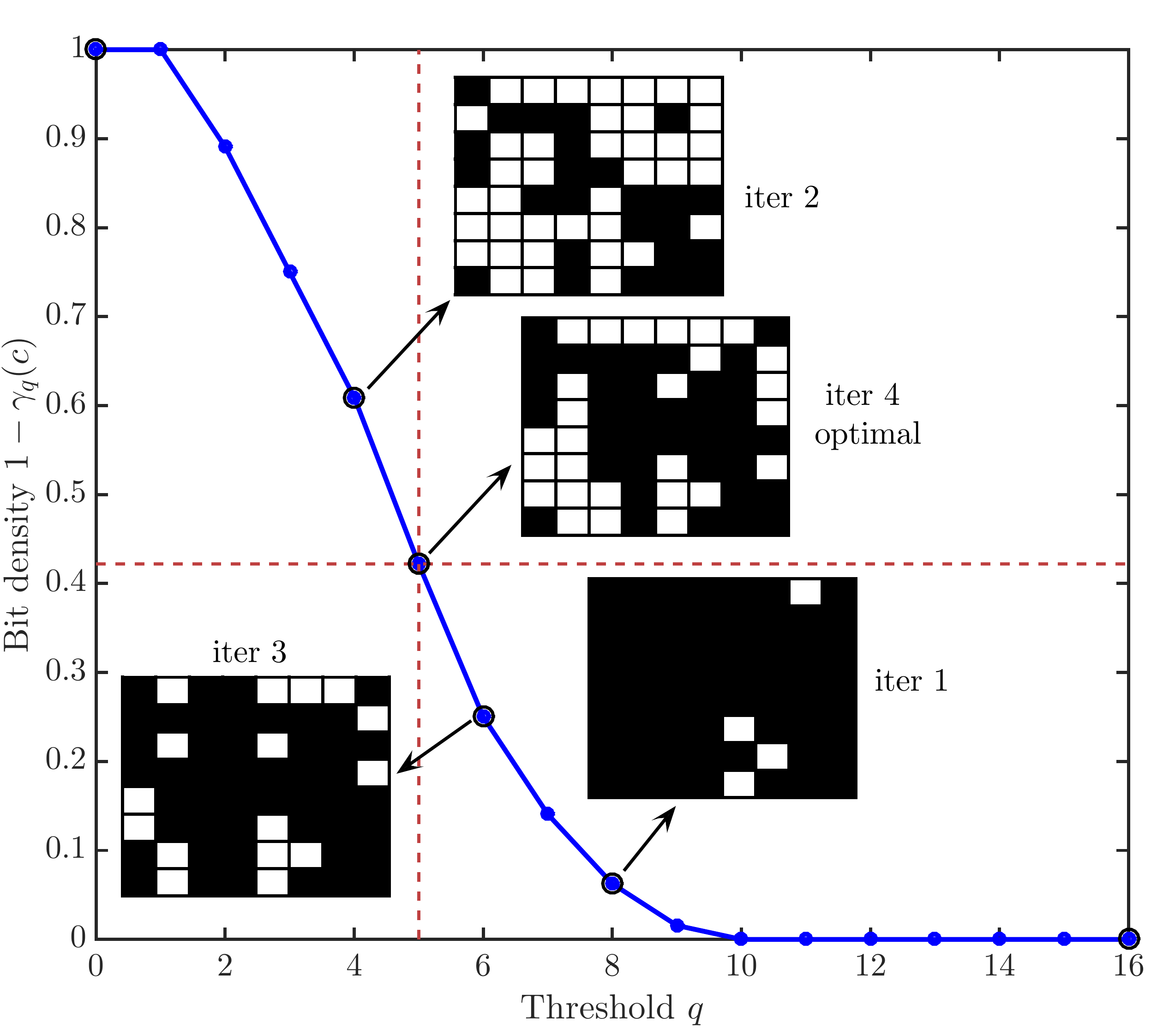}}\\
\hspace{-2ex}\includegraphics[width=0.23\linewidth,clip]{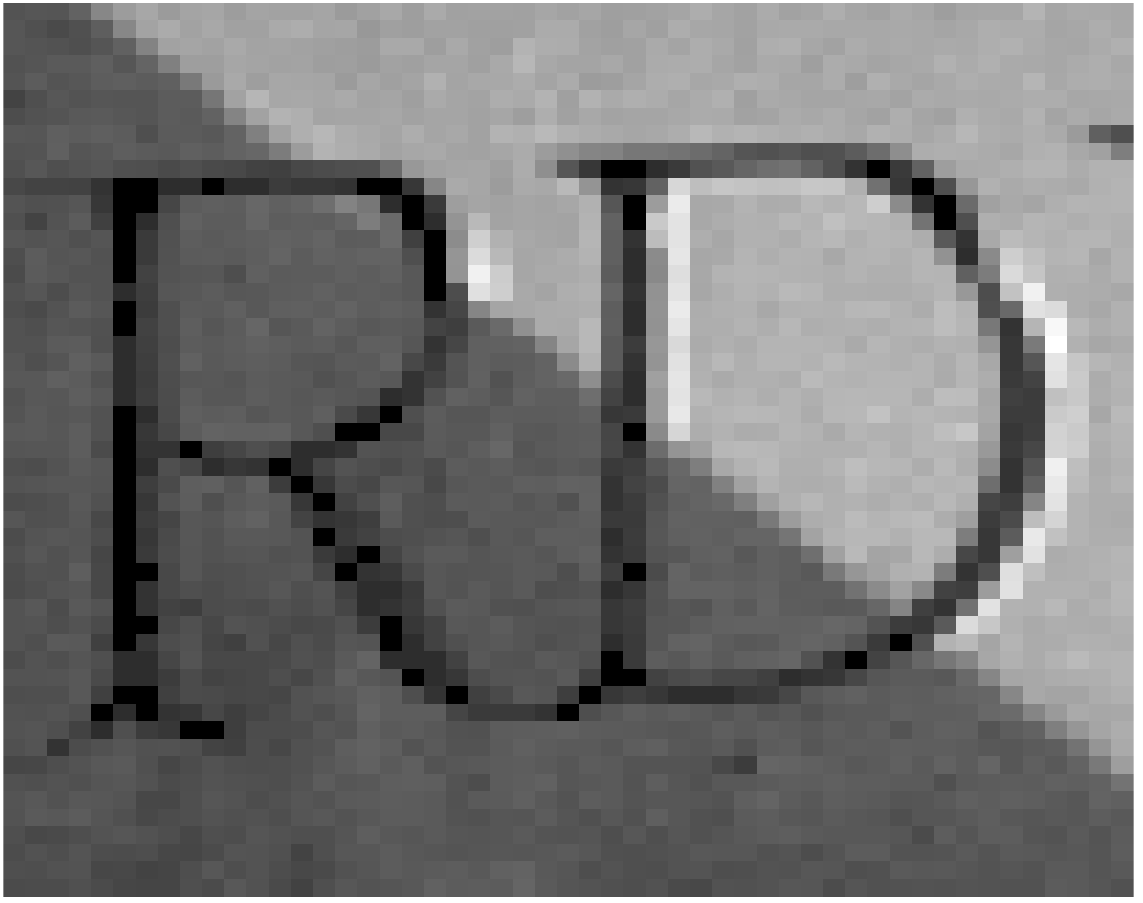}&
\hspace{-2ex}\includegraphics[width=0.23\linewidth,clip]{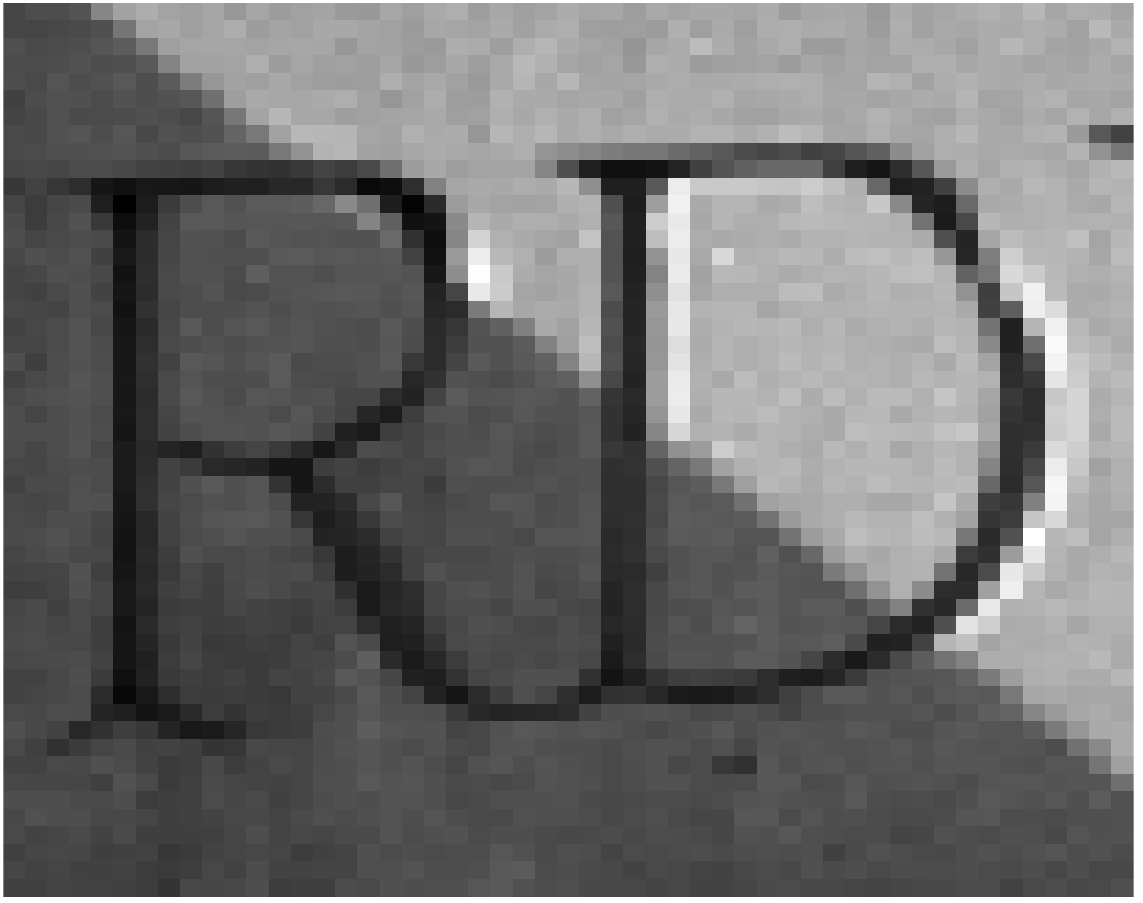}&
\hspace{-2ex}\includegraphics[width=0.23\linewidth,clip]{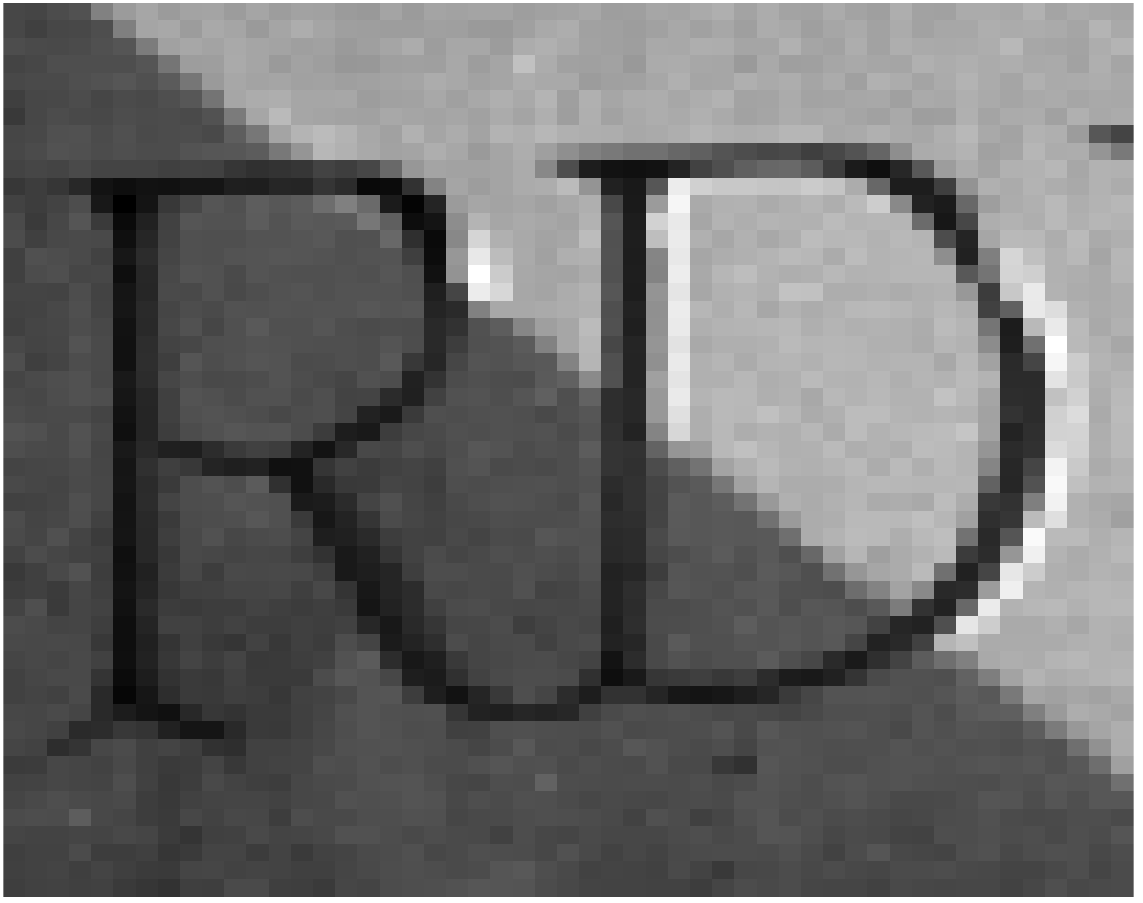}&
\hspace{-2ex}\includegraphics[width=0.23\linewidth,clip]{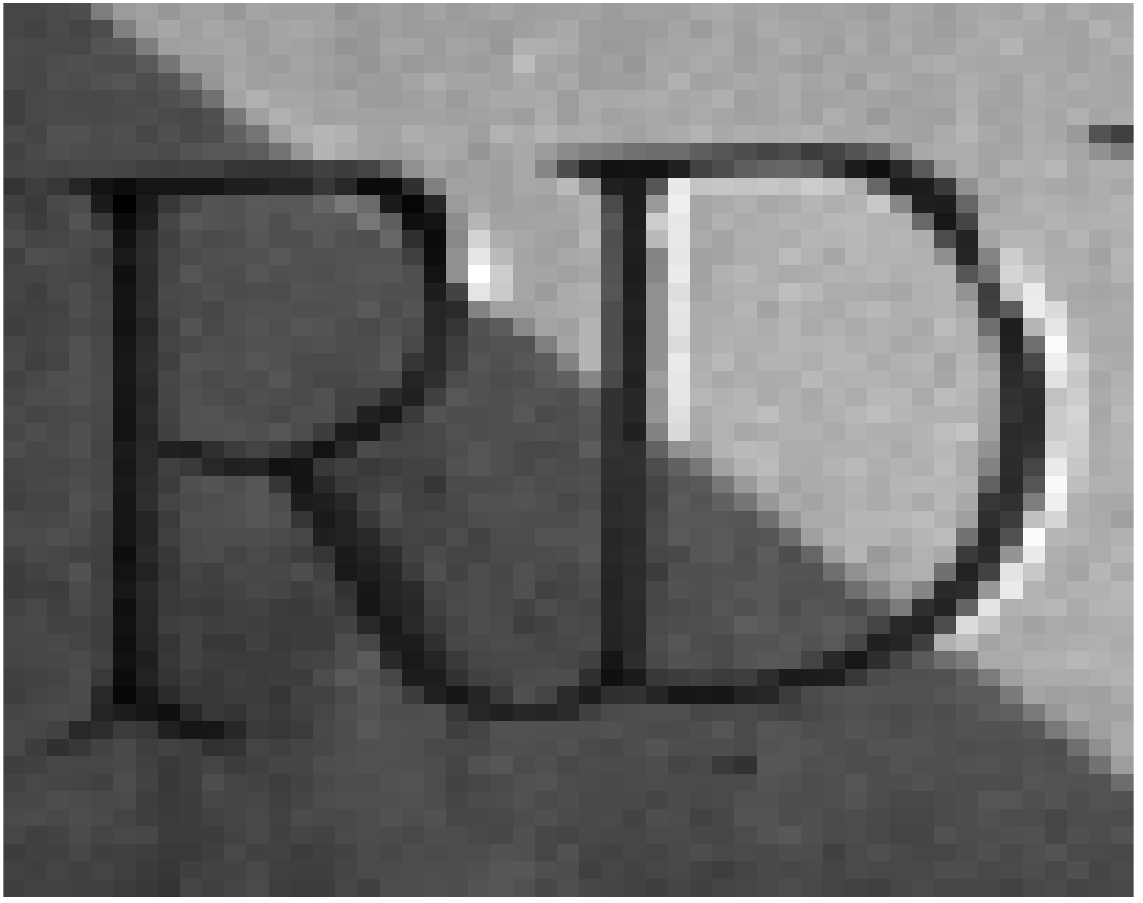}\\
\hspace{-2ex}\footnotesize{iter 1, $27.4$ dB} &
\hspace{-2ex}\footnotesize{iter 2, $37.1$ dB} &
\hspace{-2ex}\footnotesize{iter 3, $38.8$ dB} &
\hspace{-2ex}\footnotesize{iter 4, $39.1$ dB}
\end{tabular}
\caption{The proposed bisection update scheme adjusts the threshold $q$ such that the bit density $1-\gamma_q(c)$ approaches 0.5. The upper graph illustrates the bisection steps. Bottom row shows cropped patches from reconstructed images using threshold maps at different iterations and the PSNRs.}
\label{fig:bisection}
\vspace{-3ex}
\end{figure}

\begin{algorithm}[!t]
	\caption{Bisection Threshold Update Scheme}
	\label{alg:Bisection}
	\begin{algorithmic}
    \STATE Initial thresholds $q_A$ and $q_B$ such that $1-\gamma_{q_A} > 0.5$ and $1-\gamma_{q_B} < 0.5$.
	\STATE Compute $q_M=\lceil (q_A+q_B)/2 \rceil$, where $\lceil \cdot \rceil$ denotes the ceiling operator.
    \WHILE{$|\gamma_{q_M} - 0.5| < \mbox{\texttt{tol}}$}
    \STATE If $\gamma_{q_M} < 0.5$, then set $q_A = q_M$. Else, set $q_B = q_A$.
    \STATE Compute $q_M=\lceil (q_A+q_B)/2 \rceil$.
	\ENDWHILE
    \RETURN $q_M$
	\end{algorithmic}

\end{algorithm}

\begin{figure}[t]
\vspace{-2.0ex}
\centering
\includegraphics[width=0.375\textwidth]{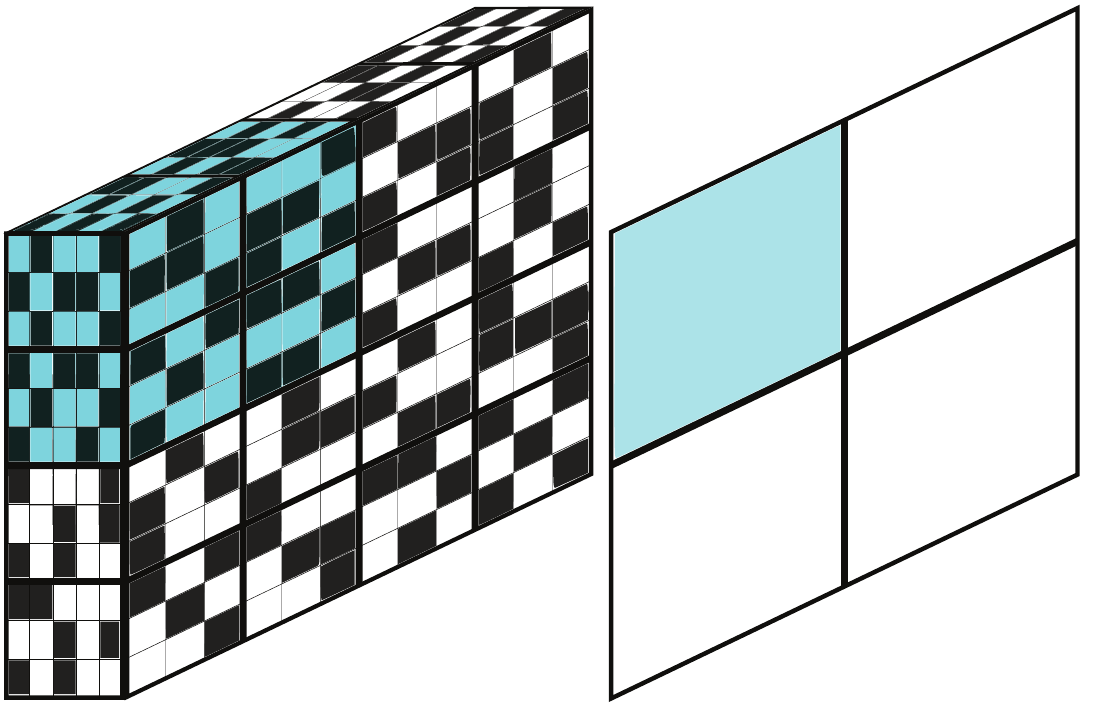}
\caption{Concept of shared thresholds. (Left) binary measurements, spatial oversampling $K=3\times 3$, Temporal oversampling $T=5$ . (Right) Threshold map, one threshold value is shared by $6\times 6$ jots.}
\label{fig:bisection 2}
\vspace{-3.0ex}
\end{figure}

In our proposed threshold update scheme, we assume that the image has been partitioned into $N$ blocks $\{\calB_n \;|\; n = 0,\ldots,N-1\}$. Each $\calB_n$ contains $KT$ binary bits and is used to estimate one pixel value $c_n$. This setting results in $N$ different thresholds, one for every pixel. To generalize the setting, it is also possible to allow multiple pixels to share a common threshold. \fref{fig:bisection 2} shows an example. The advantage of sharing a threshold for multiple pixels is that circuits associated with the sensor can be simplified. In terms of performance, since neighboring pixels are typically correlated, sharing the threshold causes little drop in the resulting SNR.

The price that the proposed bisection algorithm has to pay is the number of frames it requires to determine a good $q$. For every evaluation of $\gamma_{q_M}$, the sensor has to physically acquire one frame and compute the bit density in each of the $N$ blocks. Therefore, the more bisection steps we need, the more frames that the sensor has to physically acquire. The rate of convergence of the proposed method and existing methods will be compared in Section~\ref{sec:exp}.}

\subsection{Extension to High Dynamic Range}\label{subsec:HDR}
{\color{black}While QIS is a photon counting device, it is designed to count a few photons to keep the full-well capacity small, e.g. 20 photoelectrons as reported in \cite{Ma_Starkey_Rao_2015}. Therefore, for practical imaging tasks, we need to extend the dynamic range for QIS.

There are two ways to enable dynamic range extension:
\begin{itemize}
\item Bright Scenes: Reduce Duty Cycle. In the signal processing block diagram shown in \fref{fig:sigprocess}, we can replace the constant $\alpha$ by a fraction as $\alpha \tau$, where $0 \le \tau \le 1$ determines the ratio between the actual integration time and the readout scan time. It can also be referred to the shutter duty cycle because the shutter is opened to collect photons during this proportion of time \cite{Fossum_2015}. For very bright scenes, a low duty cycle will prevent QIS from saturating early.
\item Dark Scenes: Multiple Measurements. For dark scenes, multiple measurements can be taken to ensure enough photons over the measurement period. This, however, is different from conventional HDR imaging. In conventional HDR imaging, the multiple shots are taken at different shutter speeds, e.g., 1/8192, 1/2048, 1/512, 1/128, 1/32, 1/8, 1/2 seconds \cite{Sprow_Kuepper_Baranczuk_2013}, which is redundant. QIS's multiple shot functions more similar to burst photography \cite{Hasinoff_Sharlet_Geiss_2016}. The amount of acquisition time is significantly less than the conventional HDR imaging.
\end{itemize}

These two methods can be used for \emph{any} threshold scheme, including ours and others. The benefit of using our proposed threshold scheme is that it supports a much wider dynamic range extension. In Figure~\ref{fig:DR}, we illustrate the total dynamic range that can be covered using 4 multiple measurements at duty cycles $\tau=1$, $\tau=0.2$, $\tau=0.04$, and $\tau=0.008$. The maximum threshold level is $q_{\max}=25$, and the minimum threshold level is $q_{\min} = 1$. It can be seen from the figure that with the optimal threshold $q^*$, the dynamic range is significantly more than the non-optimal ones. In particular, we observe a 16dB and a 54dB improvement compared to $q_{\min} = 1$ and $q_{\max} = 25$, respectively. Experimental results will be shown in Section V.C.}

\begin{figure}[t]
\centering
\includegraphics[width=0.4\textwidth]{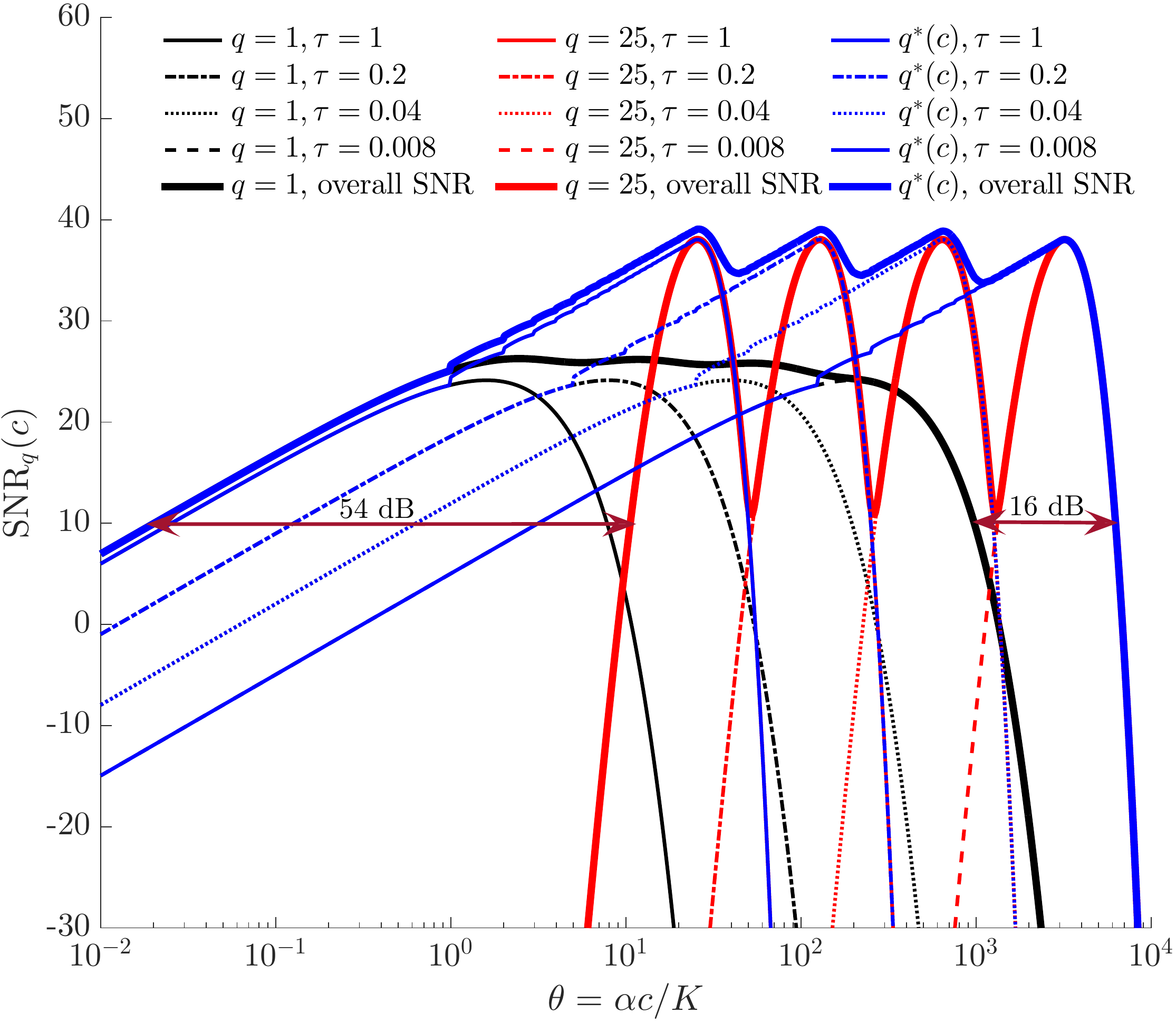}
\vspace{-2ex}
\caption{SNR in dB vs. exposure $\theta$ for HDR imaging mode obtained by fusion of frames with shutter duty cycles $\tau\in\{1,0.2,0.04,0.008\}$. Three scenarios are shown: constant threshold with $q=1$ (black), $q=25$ (red) and an optimal spatially varying threshold (blue).}
\label{fig:DR}
\vspace{-3.0ex}
\end{figure}

\subsection{Hardware Consideration}
{\color{black} Concerning the hardware implementation, we anticipate that future QIS will be equipped with per-pixel FPGAs to perform the proposed threshold update scheme. On-sensor FPGA is an actively developing technology. For example, MIT Lincoln Lab's digital focal plane array can achieve on-sensor image stabilization and edge detection \cite{Schultz_Kelly_Baker_2014} . For QIS threshold update, the complexity is low because we are only counting the number of ones in the bisection. More specifically, in order to perform the bisection, we only need $K$ additions to compute $\sum_{k=0}^{K-1} b_{Kn+k,t}$; one comparison $\sum_{k=0}^{K-1} b_{Kn+k,t} \geq 0.5$; one addition and one multiplication (with a constant 0.5) to update the threshold $q_M = \lceil (q_A+q_B)/2 \rceil$. The dominating factor here is the $K$ additions, which can be implemented efficiently by shifting bits in a buffer.

We should also point out that the proposed bisection method can be flexibly adjusted spatially and temporally for different hardware configurations. For example, we can use a spatial-temporal window $4 \times 4 \times 1$ for low-resolution high-speed imaging, or $1 \times 1 \times 16$ for high-resolution low-speed imaging. This flexibility offers additional advantages of QIS over conventional CCD and CMOS cameras. }

\section{Experimental Results}
\label{sec:exp}
{\color{black} In this section we evaluate the proposed threshold update scheme by comparing it with existing methods. We consider two evaluation metrics: (1) convergence rate of the threshold update methods; (2) quality of the reconstructed images. For reconstruction evaluation, we create our own Purdue dataset comprising 77 images captured by a Canon EOS Rebel T6i camera. For HDR imaging, we use the HDR-Eye dataset by Nemoto et al. \cite{HDREYE,Nemoto_Korshunov_Hanhart_2015}. In all experiments, we fix the spatial over-sampling factor as $K = 4 \times 4 = 16$, and number of temporal frames as $T = 13$. The maximum threshold level is set as $q_{\max} = 16$ to ensure that it is realistic for today's QIS.}

\subsection{Convergence}

We compare the proposed threshold update scheme with the Markov Chain (MC) adaptation proposed by Hu and Lu \cite{Hu_Lu_2012}. The Markov Chain adaptation models the threshold as a variable with $2^L$ states. These $2^L$ states can be regarded as $2^L$ steps before reaching to the next threshold level. The probability of changing from one state to another is controlled by a parameter $1-\beta$ with $0 < \beta < 1$. When a bit arrives, the state will be updated (increased or decreased) or will remain unchanged. Once the state is increased by $2^L$ times, the threshold will be increased by one.

When comparing Markov Chain adaptation with the proposed bisection algorithm, one should be aware of the difference between the two methods. Markov Chain adaptation is a \emph{per-jot} update scheme whereas the proposed bisection algorithm is a \emph{per-pixel} update scheme. For a pixel with $K \times K$ jots, Markov Chain adaptation needs $K^2$ iterations to update the threshold \emph{sequentially}. In contrast, the proposed bisection algorithm updates a common threshold for all $K^2$ jots \emph{simultaneously}. Thus in practice our bisection algorithm is significantly less complex to implement in hardware than the Markov Chain. In order to take the different forms of updates into account, we treat the $K^2$ iterations of Markov Chain adaptation as one ``major iteration'' and compare it with the one bisection step of the proposed algorithm.

The first comparison we make is to check the threshold at different jots. \fref{fig:Jot_Conv} shows the results of three typical runs with underlying optimal thresholds $q^* = 1, 8, 16$. In this experiment, we generate 100 random binary blocks of size $K \times K$ and estimate the threshold at each major iteration. We report the average of these 100 estimates to minimize the randomness of the data. The results show that one iteration of the proposed bisection algorithm works as good as the $K^2$ iterations of the Markov Chain adaptation. In some cases, Markov Chain tends to oscillate whereas the bisection result is stable.

\begin{figure}[t]
\centering
\includegraphics[width=0.45\textwidth]{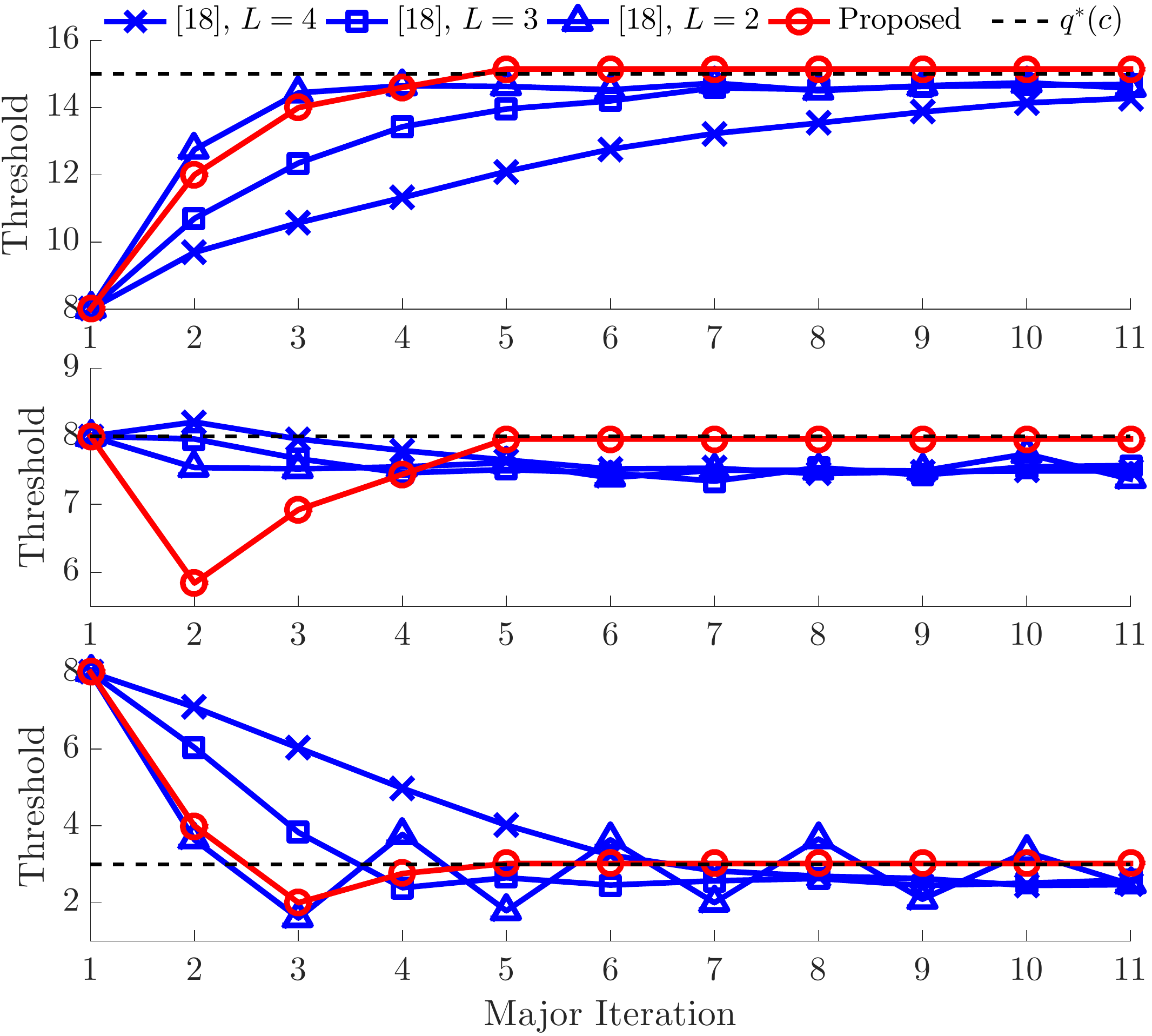}
\vspace{-2ex}
\caption{Convergence of the threshold at 3 jots. Each curve is averaged over 100 random samples. The red curve indicates the proposed bisection method. The blue curves are the Markov chain adaptation \cite{Hu_Lu_2012} with $\beta = 0.25$. Note that one major iteration of Markov Chain adaptation corresponds to $K^2$ sequential updates, and one major iteration of the bisection method corresponds to a single update to $K^2$ jots simultaneously.}
\label{fig:Jot_Conv}
\vspace{-2.0ex}
\end{figure}

\begin{figure}[t]
\centering
\includegraphics[width=0.40\textwidth]{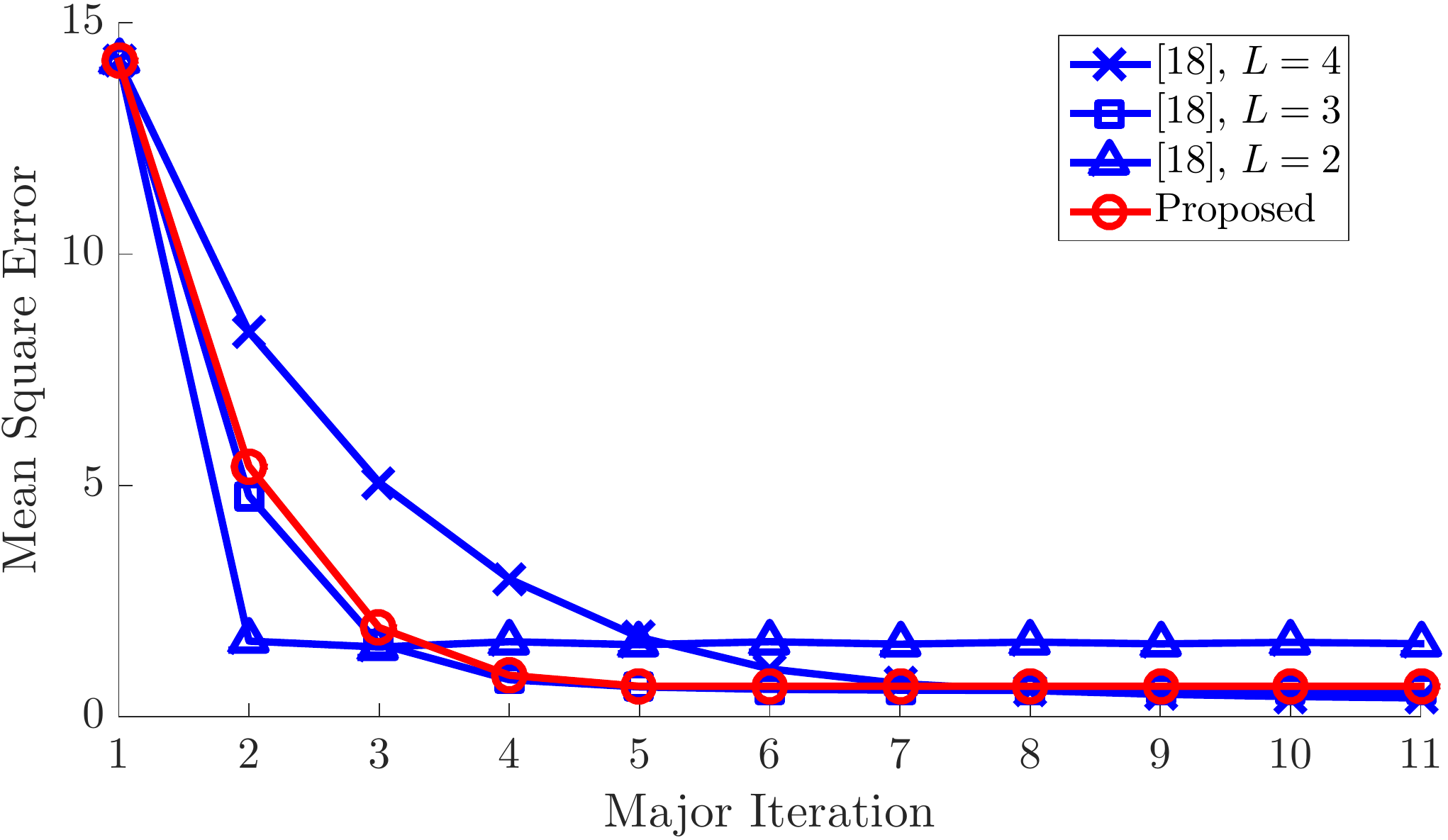}
\vspace{-2ex}
\caption{Mean square error between the estimated threshold and the ideal oracle threshold. Each curve is averaged over 50 random samples and 77 images.  The red curve indicates the proposed bisection method. The blue curves are the Markov chain adaptation \cite{Hu_Lu_2012} with $\beta = 0.25$.}
\label{fig:Conv}
\end{figure}

The second comparison we make is to check how close the estimated threshold is compared to the optimal threshold. The optimal threshold $q^*$ is obtained using the oracle scheme. In \fref{fig:Conv}, we plot the mean squared error between the estimated threshold and the oracle threshold. \textcolor{black}{For fairness, we show the results of the MSE averaged over the 77 images of our dataset, and 50 random samples per image.} One threshold is shared by $K \times K$ jots, and each $K \times K$ jots correspond to one pixel. The result is consistent with the ones shown in \fref{fig:Jot_Conv}.

\subsection{Image Reconstruction Quality}
The convergence comparison in the previous subsection is only useful to compare threshold update methods that actually return a threshold. In the QIS literature, there are methods that implicitly update the threshold, e.g., the conditional reset method \cite{Vogelsang_Stork_Guidash_2014}. For comparison with these methods, we have to compare the quality of the image reconstructed from the binary raw data. The image reconstruction is done using the closed-form ML estimate in Section~\ref{subsec:MLE}.

We consider three classes of methods:
\begin{itemize}
  \item Uniform Threshold. Uniform threshold is commonly used in the device literature \cite{Yang_Lu_Sbaiz_2010, Yang_Lu_Sbaiz_2012, Chan_Lu_2014}. A uniform threshold is a single threshold applied to all pixels in the image. In this experiment, we consider the following choices of uniform thresholds: $q = 1$, $q = 5$, $q = 10$ and $q = 16$.
  \item Conditional Reset \cite{Vogelsang_Stork_Guidash_2014}. Conditional reset counts the number of photons and is reset when it is above the threshold. The threshold in conditional reset is sequentially increasing or decreasing. The reconstructed image is obtained by digitally integrating the raw binary frames.
  \item Proposed Method. As we discussed in Section~\ref{subsec:bisection}, the proposed method can be implemented to let multiple pixels share a common threshold. Thus, in this experiment we consider three sharing strategies: (1) Share a threshold between a neighborhood of $K \times K$ jots (i.e., one threshold for one pixel); (2) Share a threshold between a neighborhood of $K^2 \times K^2$ jots (i.e., one threshold for $K \times K$ pixel); (3) Share a threshold between a neighborhood of $2K^2 \times 2K^2$ jots (i.e., one threshold for $2K \times 2K$ pixels).
\end{itemize}

The result of the experiment is shown in Table~\ref{table:result_2}. The PSNR values reported are averaged over 77 images in our dataset. Each image generates 50 random realizations, and the PSNR of an image is averaged over these 50 random realizations to minimize the randomness. As shown in the table, while conditional reset generally performs better than a uniform threshold, it performs significantly worse than the proposed threshold update scheme.

\begin{table}[t]
\footnotesize
\centering
\renewcommand{\arraystretch}{1.3}
\caption{Average PSNR and Standard deviation of 77 recovered images using different Q-maps and 50 random samples.}
\label{table:result_2}
\begin{tabular}{c|c|c|c}
  \hline
    & Configuration &   \begin{tabular}{@{}c@{}}Average \\ PSNR\end{tabular} &  Std \\
  \hline
  \hline
  \multirow{4}{*}{Uniform Threshold}
  &$q=1$&10.30 &0.01 \\
  \cline{2-4}
  &$q=5$&28.80 &0.04 \\
  \cline{2-4}
  &$q=10$&23.22 &0.02 \\
  \cline{2-4}
  &$q=16$&12.95 &0.01 \\
  \hline
  \multirow{2}{*}{Conditional Reset \cite{Vogelsang_Stork_Guidash_2014}}
 &Ascending $q$ sequence&23.77 &0.52 \\
 \cline{2-4}
 &Descending $q$ sequence&24.95 &0.53\\
 \hline
  \multirow{3}{*}{Proposed Method}
  &$2K^2\times 2K^2$&30.14 &0.06 \\
  \cline{2-4}
  &$K^2\times 	K^2$&31.18 &0.06 \\
  \cline{2-4}
  &$K\times K$&32.78 &0.02 \\
  \hline
\end{tabular}
\end{table}

\subsection{Influence of QIS Threshold on HDR Imaging}

\begin{figure*}[t]
\centering
\begin{tabular}{ccccccccc}
\includegraphics[width=0.10\linewidth]{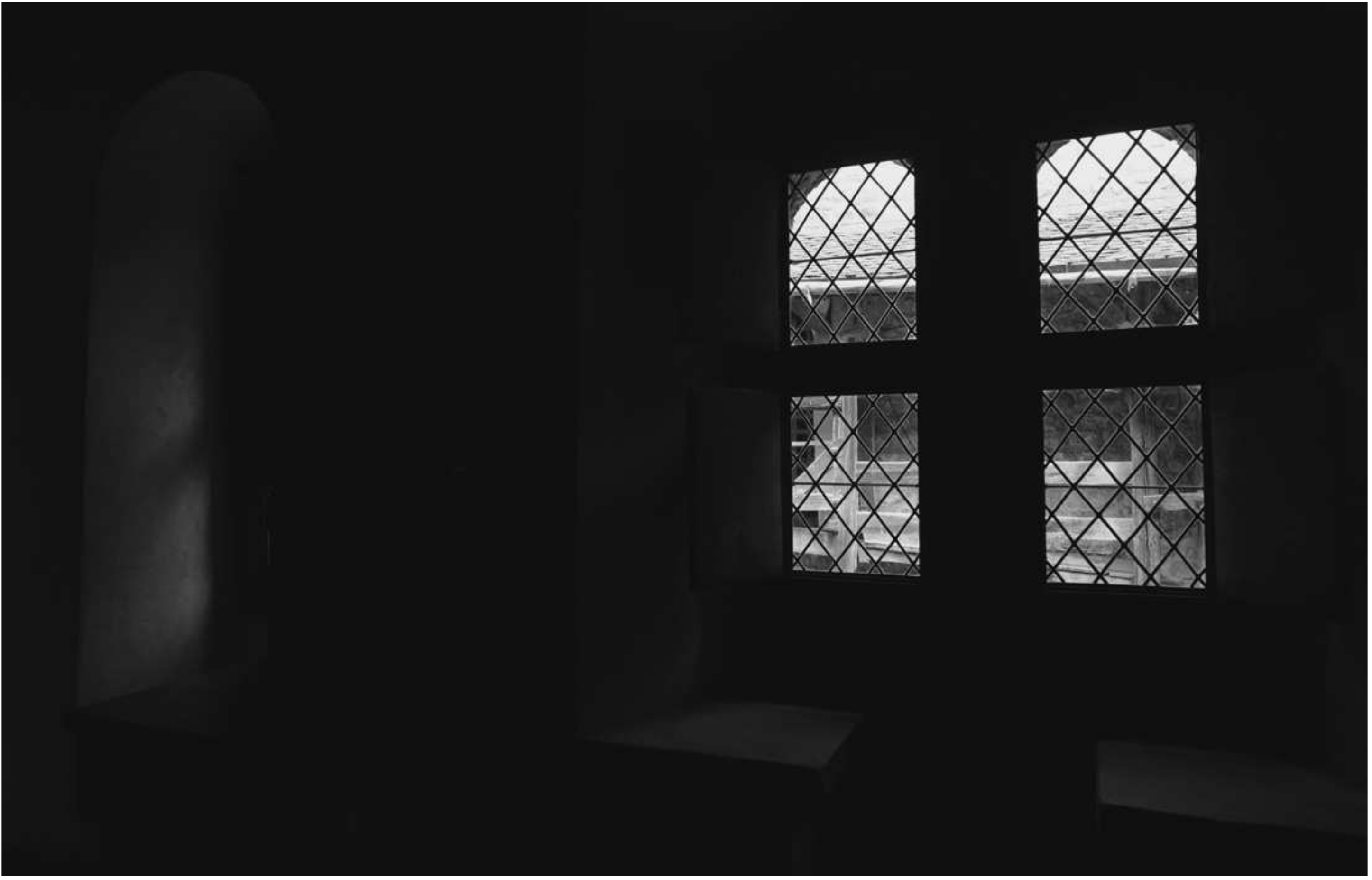}&
\hspace{-2ex}\includegraphics[width=0.10\linewidth]{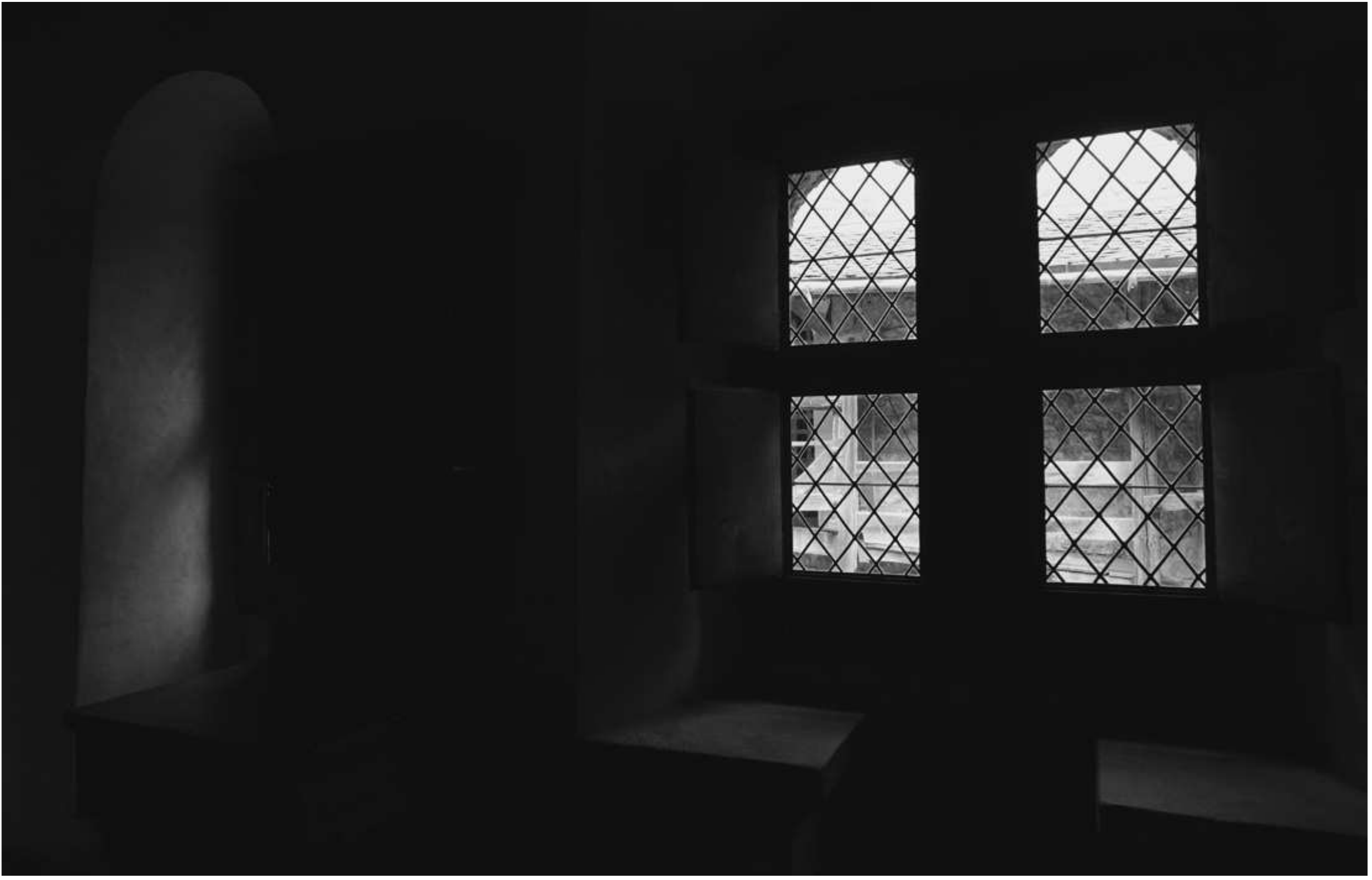}&
\hspace{-2ex}\includegraphics[width=0.10\linewidth]{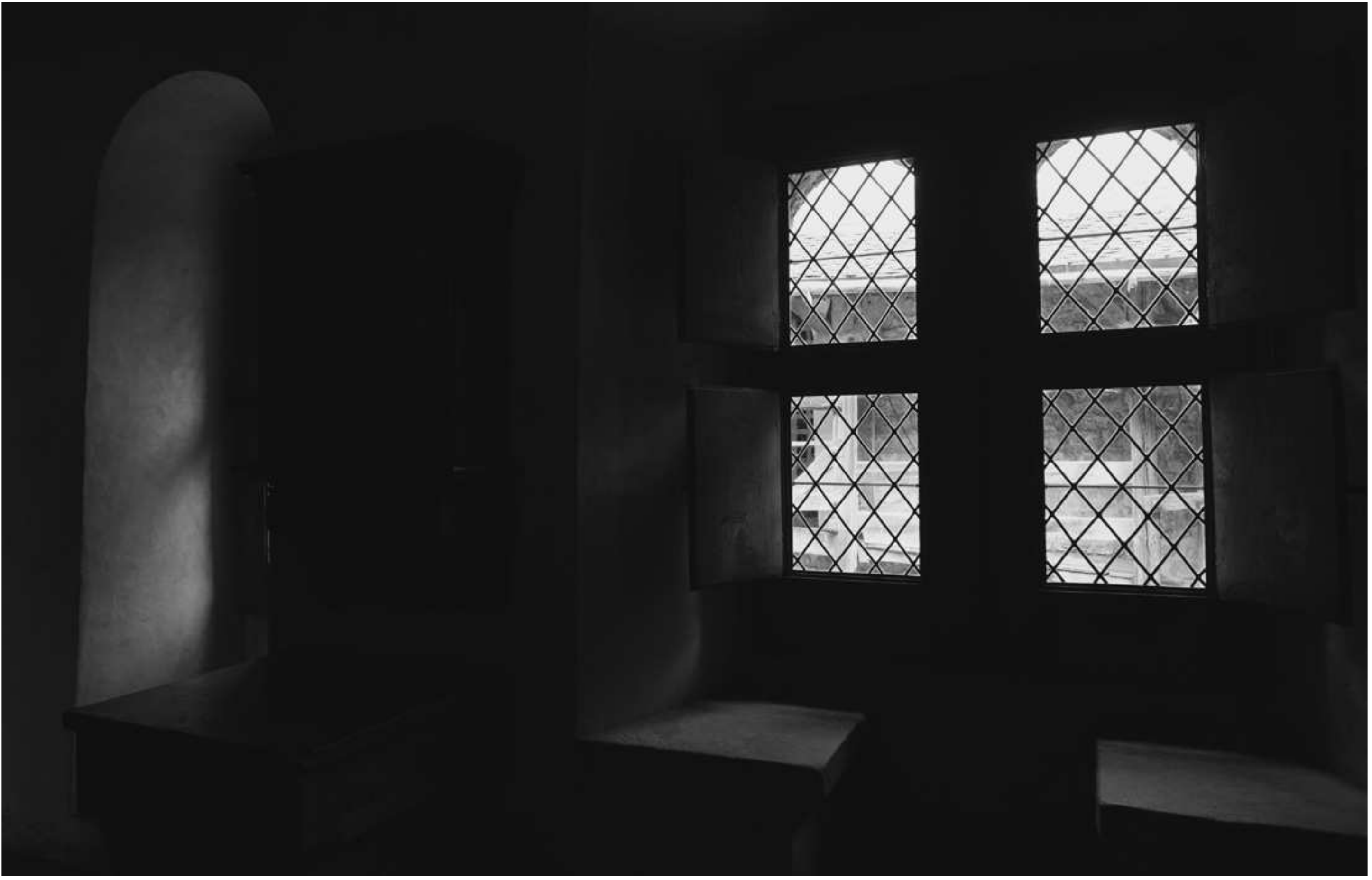}&
\hspace{-2ex}\includegraphics[width=0.10\linewidth]{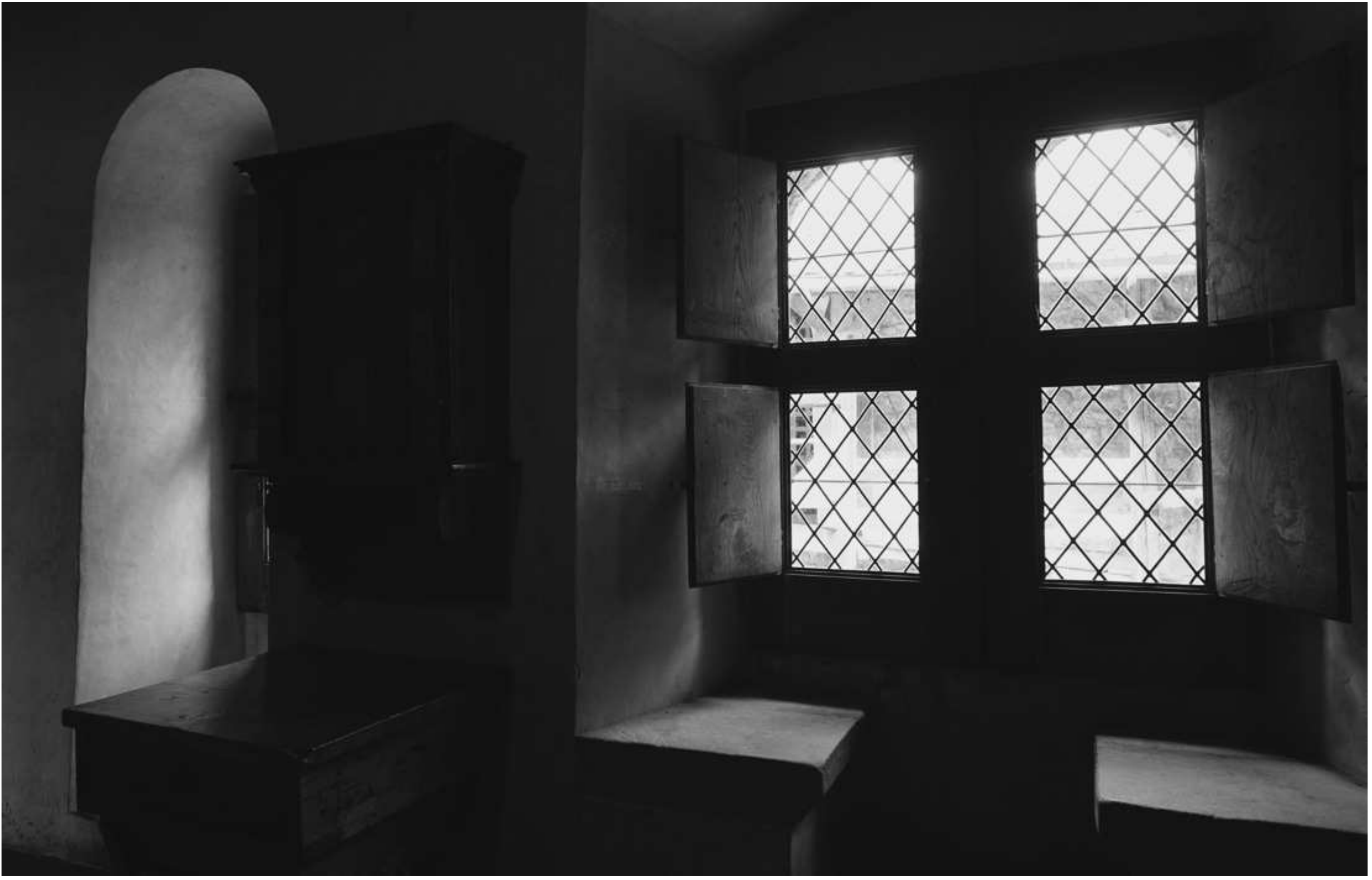}&
\hspace{-2ex}\includegraphics[width=0.10\linewidth]{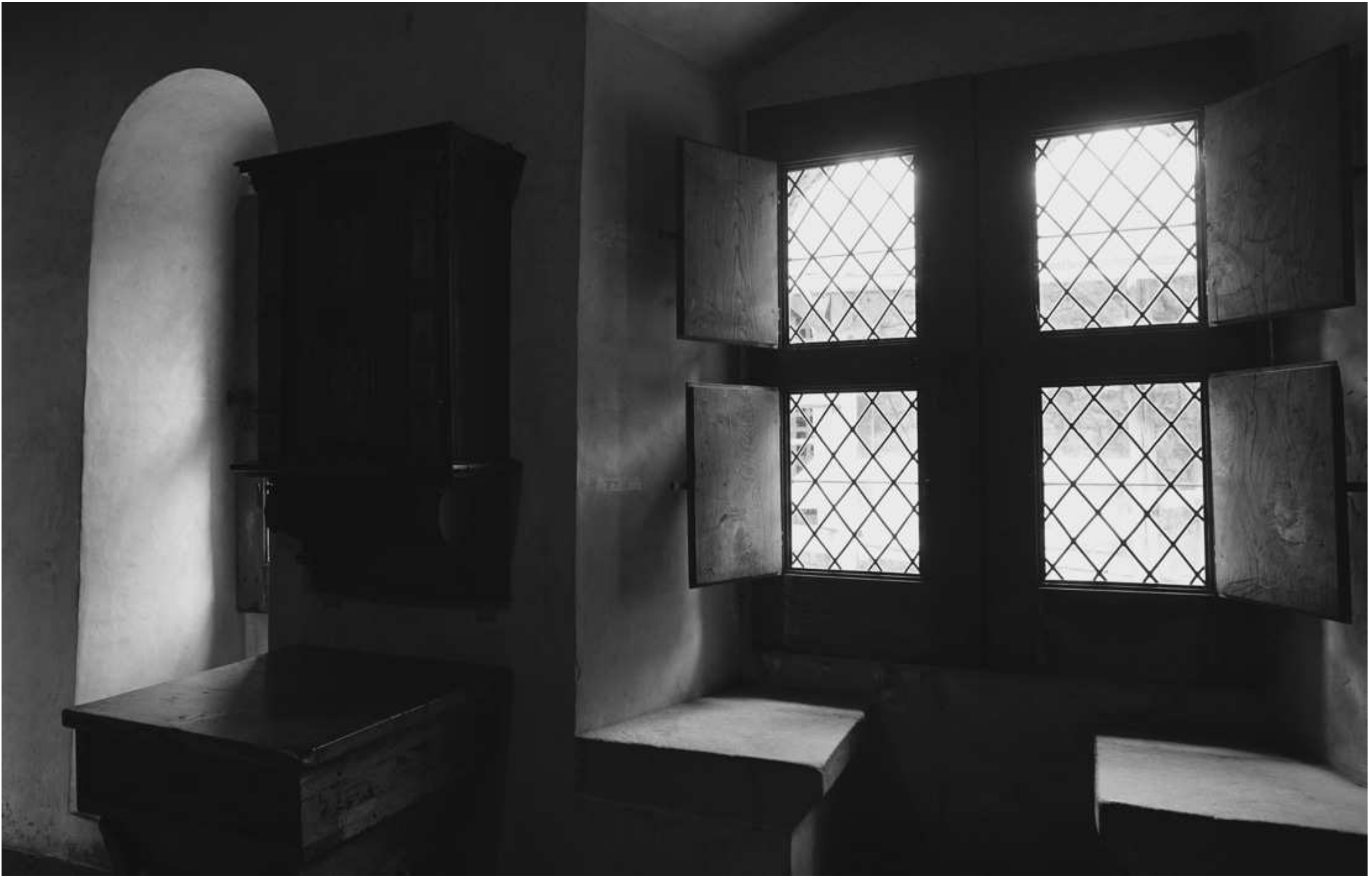}&
\hspace{-2ex}\includegraphics[width=0.10\linewidth]{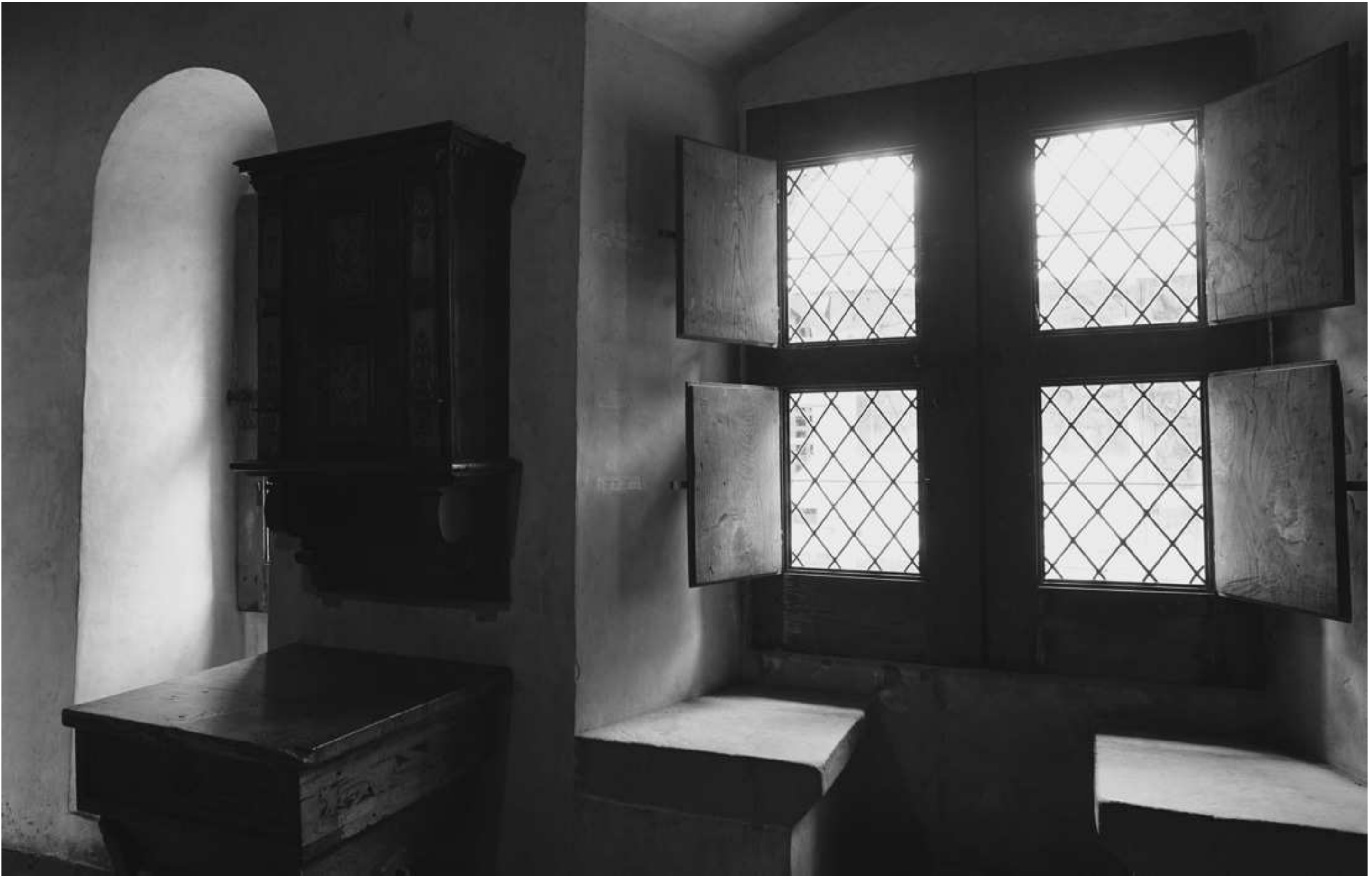}&
\hspace{-2ex}\includegraphics[width=0.10\linewidth]{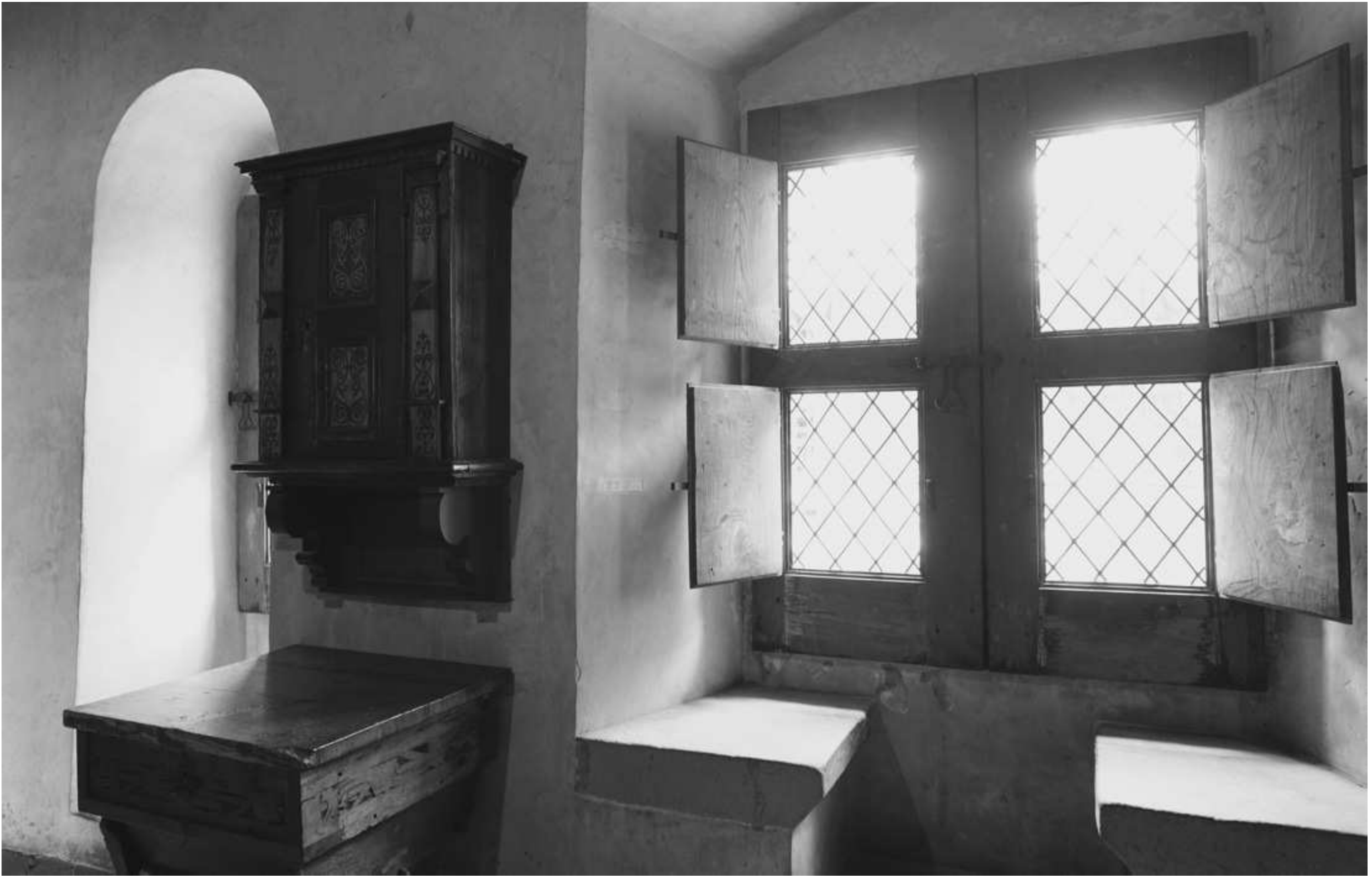}&
\hspace{-2ex}\includegraphics[width=0.10\linewidth]{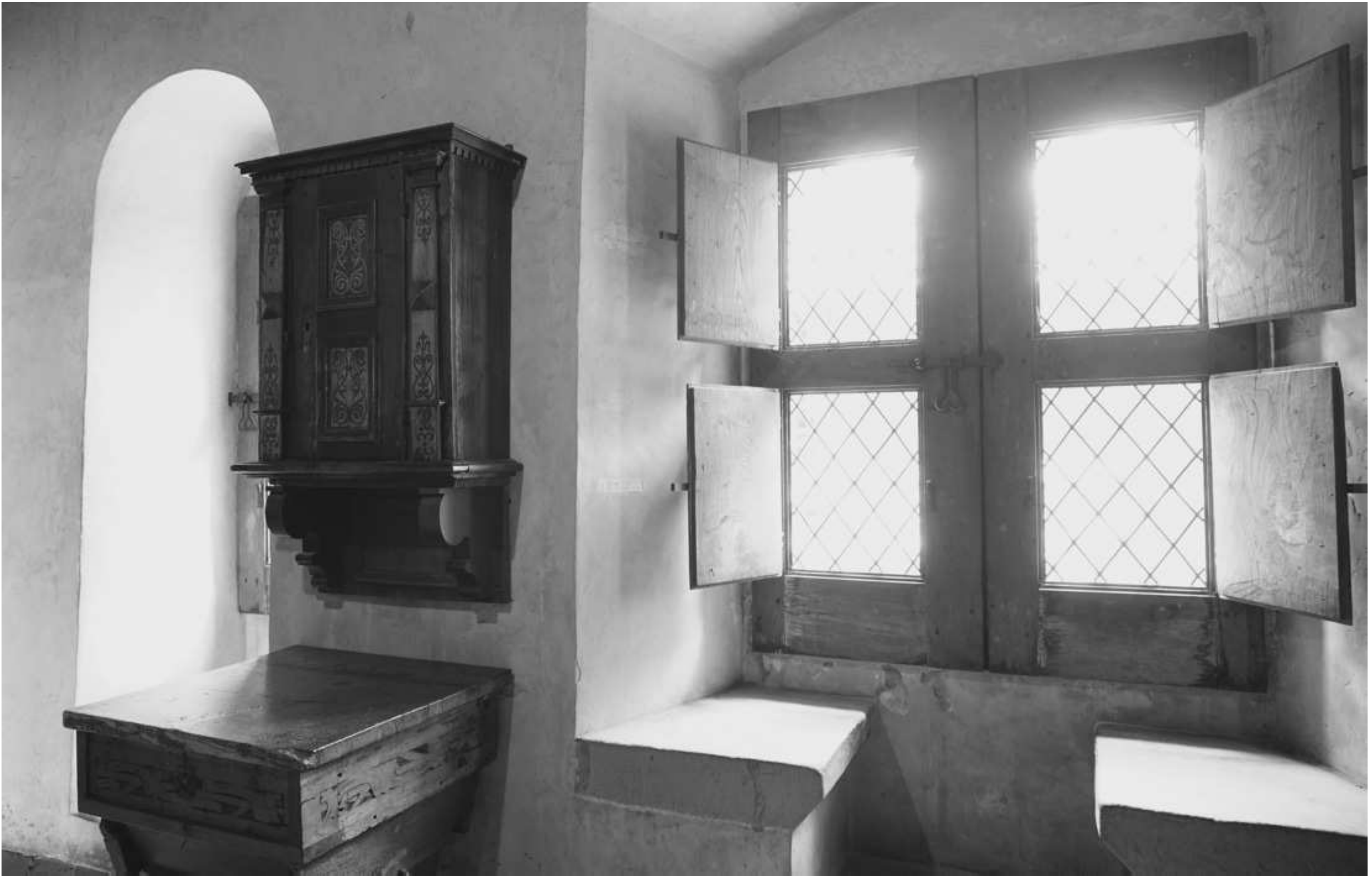}&
\hspace{-2ex}\includegraphics[width=0.10\linewidth]{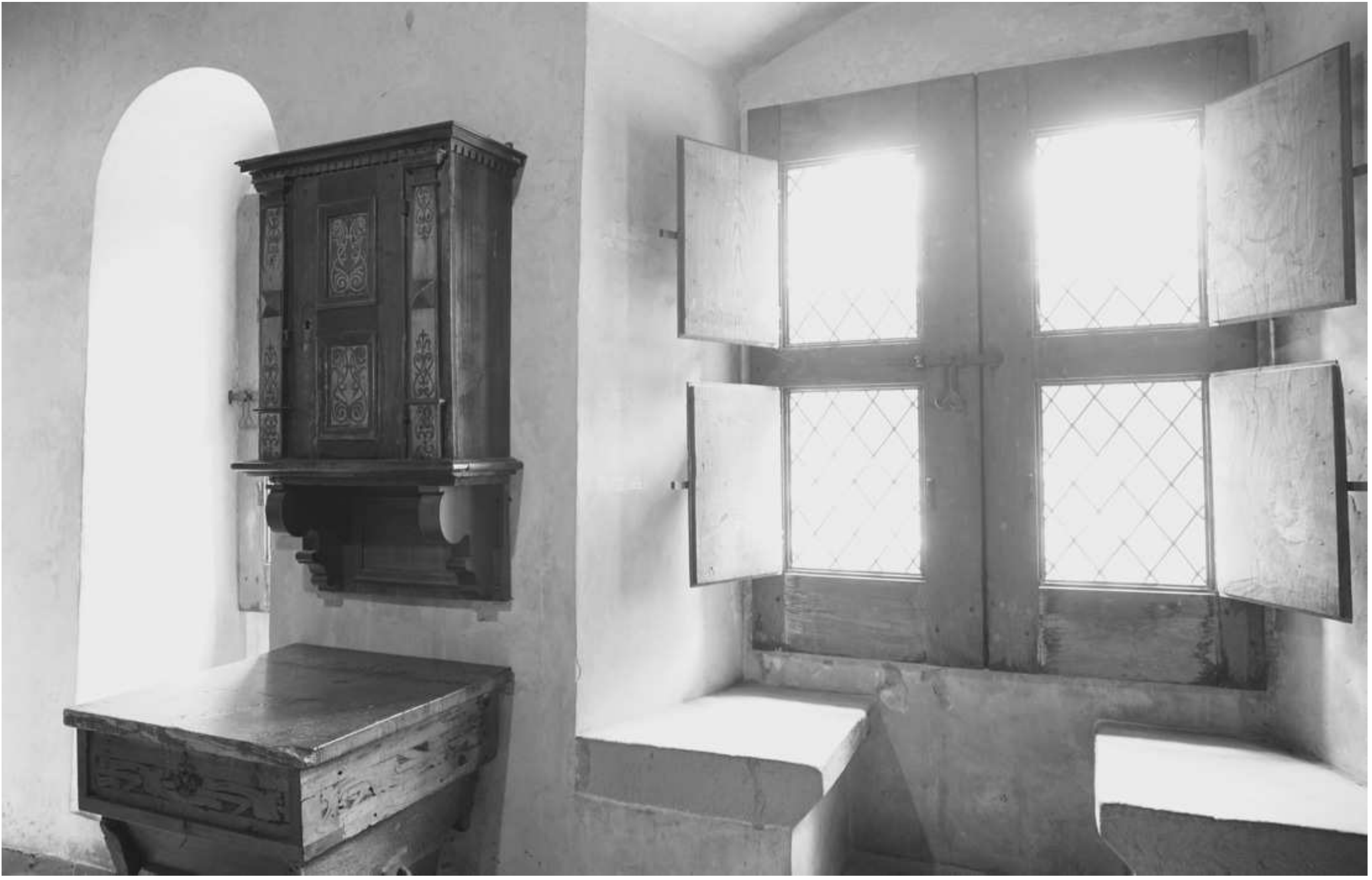}
\end{tabular}
\vspace{-1ex}
\caption{Bracketed images with different exposure settings. From Left to Right: $-2.7$,$-2$, $-1.3$, $-0.7$, $0$, $0.7$, $1.3$, $2$, and $2.7$ EV.}
\label{fig:shots}
\end{figure*}

\begin{figure*}[t]
\centering
\begin{tabular}{ccc}
\includegraphics[width=0.3\linewidth]{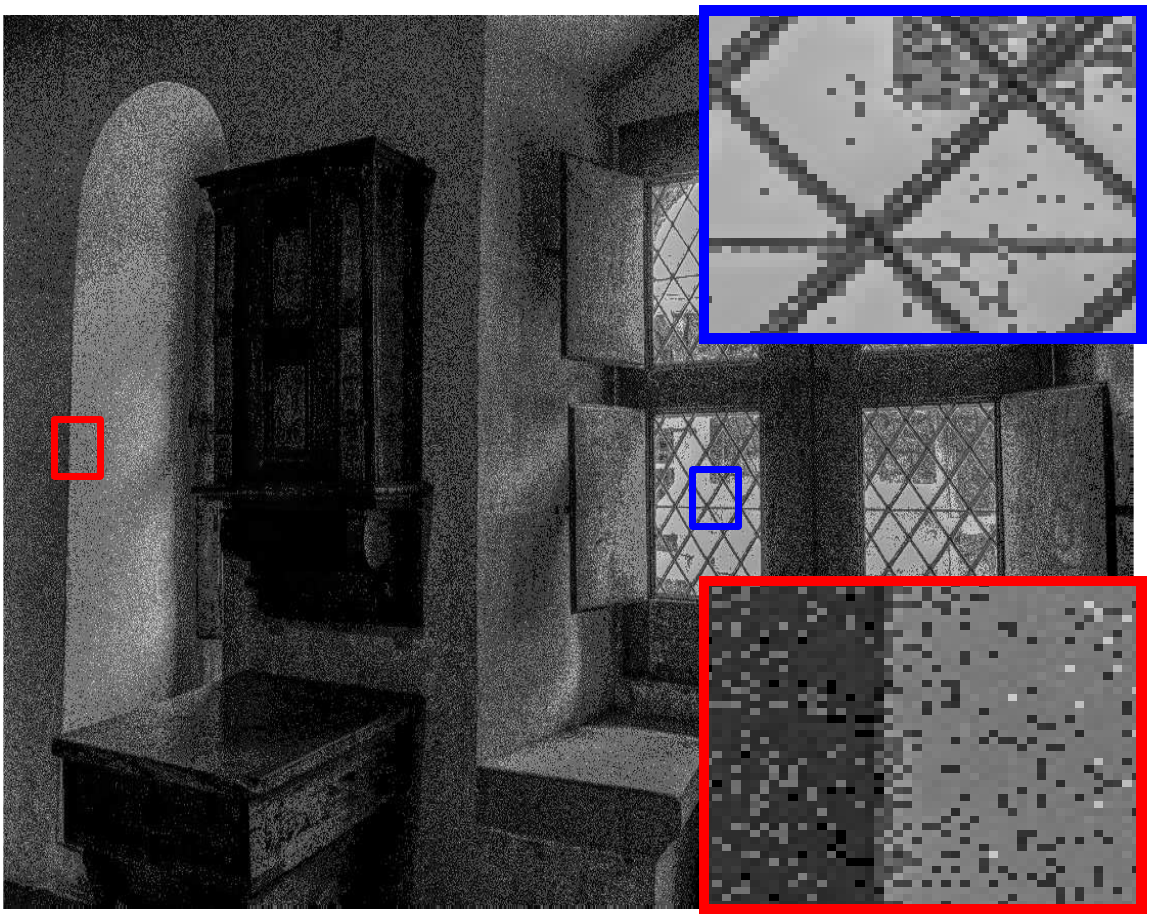}&
\includegraphics[width=0.3\linewidth]{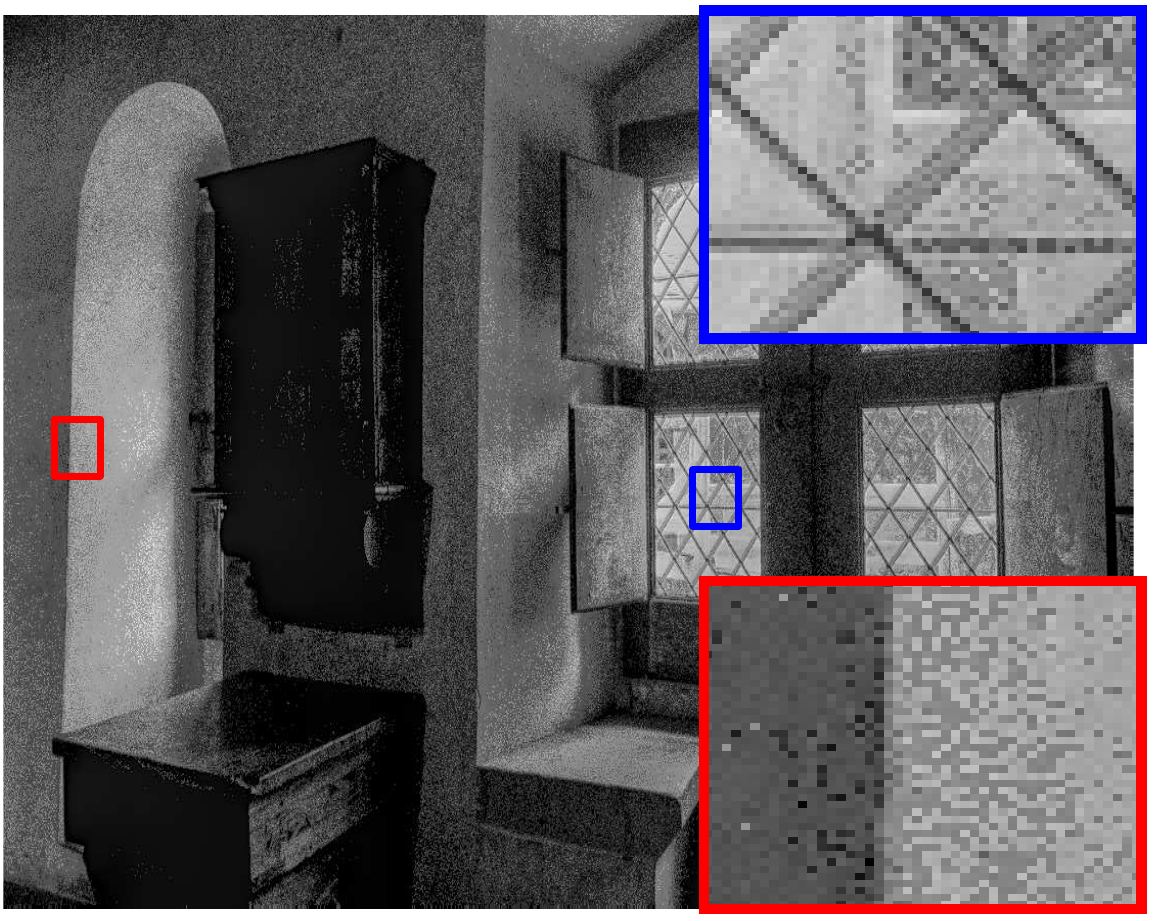}&
\includegraphics[width=0.3\linewidth]{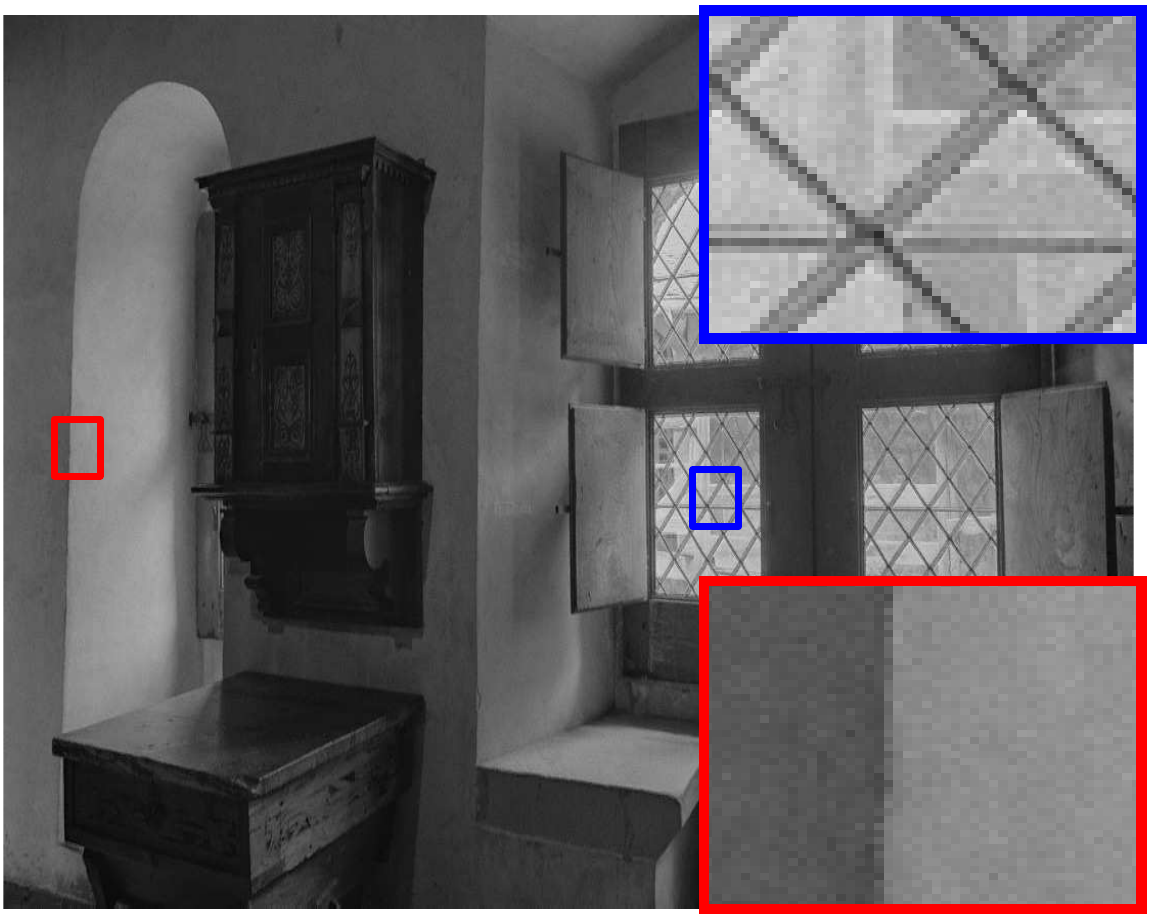}\\
\footnotesize $q = 1$, $\textrm{PSNR}=17.94$ dB  &
\footnotesize $q = 16$, $\textrm{PSNR}=20.77$ dB &
\footnotesize Proposed, $\textrm{PSNR}=31.46$ dB
\end{tabular}
\caption{The reconstructed HDR images using different thresholds. See supplementary material for additional results.}
\label{fig:HDR}
\vspace{-2.0ex}
\end{figure*}

Since QIS does not have sufficient full well capacity to accumulate photons for HDR imaging, we apply the dynamic range extension method discussed in Section~\ref{subsec:HDR}. When different threshold schemes are used, the reconstructed HDR images will be affected. The objective of this experiment is to evaluate the influence of the threshold in HDR imaging.

In this experiment, we consider the HDR-Eye image dataset  \cite{HDREYE,Nemoto_Korshunov_Hanhart_2015}. Each HDR image in this dataset contains 9 images acquired at different exposure settings ($-2.7$,$-2$, $-1.3$, $-0.7$, $0$, $0.7$, $1.3$, $2$, and $2.7$ EV). A snapshot of these images are shown in Figure~\ref{fig:shots}. From each exposure, we simulate the photon counts resulting from the luminance channel. The sensor gain is set as $\alpha = K^2(q_{\max}-1)$ to ensure proper number of photons, where $K = 4 \times 4 = 16$ and $q_{\max} = 16$. On the reconstruction side, we reconstruct the 9 images using the MLE discussed in Section~\ref{subsec:MLE}. \textcolor{black}{Tone mapping and exposure fusion \cite{Mertens_Kautz_Reeth_2007} are applied to the 9 imags to generate an HDR image.} As a reference, we apply the same tone mapping and fusion algorithm to the 9 ground truth images. PSNR between the reference and the estimated is then recorded.

The result of this experiment is shown in \fref{fig:HDR}. With the proposed threshold update scheme, the reconstructed images achieve the highest PSNR value and visual quality. When $q=1$, which is too low, the image appears under-exposed. When $q=16$, which is too high, the image appears over-exposed. The spatially varying property of the proposed method mitigates the issue by allowing multiple thresholds.

\textcolor{black}{In practice, one would typically add image denoisers to handle the randomness in the ML estimate and potentially other types of noise. This can be done using methods such as \cite{Chan_Elgendy_Wang_2016}. In HDR literature, there are also optical approaches that reduce the number of exposures, e.g., \cite{Nayar_Mitsunaga_2000, Aguerrebere_Almansa_Delon_2017}. These techniques are complementary to QIS, because QIS is a sensor of similar functionality of a CMOS. Thus optical techniques can always be added.}

\section{Conclusion}
\label{sec:conc}
{\color{black}Quanta Image Sensor is a new image sensor for high speed, high resolution and high dynamic range imaging. The sensor has a threshold which needs to be carefully adjusted so that the dynamic range can be maximized. We studied the threshold design problem by establishing several theoretical results. First, we showed that an oracle threshold can be obtained assuming that we know the underlying pixel value. Our result showed that the oracle threshold must match with the pixel value in order to maximize the signal-to-noise ratio. Second, we showed that around the oracle threshold, there exists a set of thresholds that can produce asymptotically unbiased estimates of the pixel value. Within this set of threshold, the signal-to-noise ratio stays very close to the oracle case. Third, we developed a bisection method to update the threshold scheme. We also discussed how QIS can be used in HDR imaging, and its advantages compared to standard sensors. Experimental results showed the effectiveness of our proposed approach compared to the standard approach that uses uniform threshold for all pixels.}

\section*{Acknowledgment}
The authors thank Professor Eric Fossum, Jiaju Ma and Saleh Masoodian at Dartmouth College for many insightful discussions about the physics and circuits of QIS.

%
%

\appendices
\section{} \label{apx:proofs}
\subsection{Proof of Proposition \ref{prop:Fisher}}
\label{proof:Fisher}
The Fisher Information metric is defined as:
\begin{align}
I_q(c) \bydef \E_{B}\left[\frac{-\partial^2}{\partial c^2} \mbox{log} \; \Pb(B=b;\theta,q)\right],
\end{align}
where $\theta=\alpha c / K$. Using the chain rule, we can derive the Fisher Information as follows
\begin{align}
I_q(c) = \left(\frac{\alpha}{K}\right)^2 \E_{B}\left[\frac{-\partial^2}{\partial \theta^2} \mbox{log} \; \Pb(B=b;\theta,q)\right].
\end{align}

\noindent The expectation can be calculated as follows
\begin{align} \label{eq:Fisher}
I_q(c)&=\left(\frac{\alpha }{K}\right)^2\left[\frac{-\partial^2}{\partial \theta^2} \mbox{log} \; \Pb(B=1;\theta,q)\right]\Pb(B=1;\theta,q)\nonumber\\
&+\left(\frac{\alpha }{K}\right)^2\left[\frac{-\partial^2}{\partial \theta^2} \mbox{log} \; \Pb(B=0;\theta,q)\right]\Pb(B=0;\theta,q)
\end{align}
Using (\ref{eq:diff}) to differentiate the 1st term, we get:
\begin{align} \label{eq:partial1}
&\frac{\partial^2}{\partial \theta^2} \mbox{log} \; \Pb(B=1;\theta,q)=\frac{\partial^2}{\partial \theta^2} \mbox{log}\left(1-\Psi_q(\theta)\right) \nonumber \\
&=\frac{R'(1-\Psi_q(\theta))-R^2/\Gamma(q)}{\Gamma(q)\left(1-\Psi_q(\theta)\right)^2},
\end{align}
where $R=e^{-\theta}\theta^{q-1}$ and $R'=\partial R/\partial \theta$. Similarly, the second term is
\begin{equation}
\begin{split} \label{eq:partial0}
&\frac{\partial^2}{\partial \theta^2} \mbox{log} \; \Pb(B=0;\theta,q)=\frac{\partial^2}{\partial \theta^2} \mbox{log} \; \Psi_q(\theta)\\
&=-\frac{R'\Psi_q(\theta)+R^2/\Gamma(q)}{\Gamma(q)\left(\Psi_q(\theta)\right)^2}.
\end{split}
\end{equation}
Substitute \eqref{eq:partial1} and \eqref{eq:partial0} in \eqref{eq:Fisher} yields
\begin{equation}
\begin{split}
I_q(\theta)
&=\left(\frac{\alpha}{K}\right)^2\Big[-\frac{R'\Gamma(q)(1-\Psi_q(\theta))-R^2}{\Gamma^2(q)\left(1-\Psi_q(\theta)\right)}\nonumber\\
&+\frac{R'\Gamma(q)\Psi_q(\theta)+R^2}{\Gamma^2(q)\Psi_q(\theta)}\Big]\nonumber\\
&=\left(\frac{\alpha}{K}\right)^2\frac{e^{-2\theta}\theta^{2q-2}}{\Gamma^2(q)\Psi_q(\theta)\left(1-\Psi_q(\theta)\right)}.
\end{split}
\end{equation}

\subsection{Proof of Proposition~\ref{prop:lowerbound}}
\label{proof:lowerbound}
The lower bound is obtained by observing that the product $\Psi_q(\theta)\left(1-\Psi_q(\theta)\right)$ attains its maximum value when ${\Psi_q(\theta)=1/2}$. Substituting with the upper bound ${\Psi_q(\theta)\left(1-\Psi_q(\theta)\right)\leq1/4}$, we get:
\begin{align*}
\log(c^2 I_q(c)) &= \log\left\{\left(\frac{\alpha c}{K}\right)^2\frac{e^{-2\theta}\theta^{2q-2}}{\Gamma^2(q)\Psi_q(\theta)\left(1-\Psi_q(\theta)\right)}\right\}\\
&= \log\frac{e^{-2\theta}\theta^{2q}}{\Gamma^2(q)\Psi_q(\theta)\left(1-\Psi_q(\theta)\right)}  \\
& \ge  \log\frac{4e^{-2\theta}\theta^{2q}}{\Gamma^2(q)} \\
&=2\log 2-2\theta + 2q \log \theta - 2\log \Gamma(q) \\
&=2\left(\log 2-\frac{\alpha c}{K} + q \log \frac{\alpha c}{K} - \log \Gamma(q)\right).
\end{align*}

\subsection{Proof of Proposition~\ref{prop:optimal q}}
\label{proof:oracle q}
Using the definition of Gamma function $\Gamma(q)=(q-1)!$ and $\theta=\frac{\alpha c}{K}$, we can rewrite the lower bound in Proposition~\ref{prop:lowerbound} as follows.
\begin{align*}
L_q(c) &= 2\left(\log 2-\theta + q \log \theta - \log (q-1)!\right)\\
&= 2\left(\log 2-\theta + (q-1) \log \theta +  \log \theta - \log \prod_{k=1}^{q-1} k\right) \\
&= 2\left(\log 2-\theta + \sum_{k=1}^{q-1} \log (\theta/ k) + \log \theta\right)
\end{align*}
The only dependence on $q$ is in the second term, so we take a closer look at it. When $q-1<\lfloor\theta\rfloor$, all summands ${\log(\theta/k)}$ are positive because $k<\lfloor\theta\rfloor$. Hence, the total sum increases by increasing $q$. On the other hand, when ${q-1>\lfloor\theta\rfloor}$, we start to add negative summands ${\log(\theta/k)}$ because $k>\theta$. Therefore, the total sum decreases on increasing $q-1$ over ${\lfloor\theta\rfloor}$. Thus, maximum is obtained at $q = \lfloor \theta \rfloor+1=\lfloor \frac{\alpha c}{K} \rfloor +1$.

\subsection{Proof of Proposition \ref{prop:mean and var of gamma}}
\label{proof:gamma mean}
By definition, $S \bydef \sum_{t=0}^{T-1}\sum_{k=0}^{K-1} B_{k,t}$ is the summation of $KT$ independent i.i.d. Bernoulli random variables. Therefore, $S$ is a binomial random variable with parameters $n \bydef KT$ and $p \bydef 1-\Psi(\alpha c/K)$. The mean and variance of a binomial random variable is $\E[S]=np$, and $\Var[S]=np(1-p)$. Therefore, we have
\begin{align*}
&\E\left[\gamma_q(c)\right]=1-\frac{\E[S]}{KT}=\Psi_q\left(\frac{\alpha c}{K}\right), \quad \textrm{and}\\
&\Var\left[\gamma_q(c)\right]=\frac{\Var\left[S\right]}{K^2T^2}=\frac{1}{KT} \Psi_q\left(\frac{\alpha c}{K}\right)\left(1-\Psi_q\left(\frac{\alpha c}{K}\right)\right).
\end{align*}

\subsection{Proof of Proposition \ref{prop:delta}}
\label{proof:delta}
The probability $\Pb[0 < \gamma_q(c) < 1]$ can be evaluated by checking the complement when $\gamma_q(c) = 0$ or $\gamma_q(c) = 1$:
\begin{align*}
\Pb[0 < \gamma_q(c) < 1]
&= 1-\Pb[\gamma_q(c) = 0]-\Pb[\gamma_q(c) = 1]\\
&= 1-\Pb[S = 0]-\Pb[S = KT]\\
&\overset{(a)}{=} 1-\Psi_q(\theta)^{KT}-[1-\Psi_q(\theta)]^{KT},
\end{align*}
where (a) follows from the fact that $S$, which is a sum of i.i.d. Bernoulli random variables, is a binomial random variable.

Let $0 < \delta < 1$. If
$$
1-\left(\frac{\delta}{2}\right)^{\frac{1}{KT}} \le \Psi_q(\theta) \le \left(\frac{\delta}{2}\right)^{\frac{1}{KT}},
$$
then we have
\begin{align*}
\Psi_q(\theta)^{KT} < \frac{\delta}{2} \;\; \mbox{and} \;\; [1-\Psi_q(\theta)]^{KT} < \frac{\delta}{2}.
\end{align*}
Thus, it holds that
\begin{align*}
1-\Psi_q(\theta)^{KT}-[1-\Psi_q(\theta)]^{KT}>1-\delta.
\end{align*}


\bibliographystyle{IEEEbib}
\bibliography{refs}

\end{document}


\maketitle

\begin{abstract}
This supplementary report provides the following additional information of the main article.
\begin{itemize}
\item Derivation of $\mathrm{SNR}_q(c)$ from exposure-referred SNR,
\item Properties of the incomplete Gamma function,
\item Comparison with the threshold design scheme by Yang \cite{Yang_2012},
\item Phase transition under different configurations,
\item Influence of Non-Boxcar Kernel $\mathbf{G}$, and
\item Additional results for HDR image reconstruction.
\end{itemize}
\end{abstract}

\section{Derivation of $\mathrm{SNR}_q(c)$ from exposure-referred SNR}

In the literature of QIS devices, one metric to quantify the image quality is the \emph{exposure}-referred signal-to-noise \cite{Fossum_2013}. In image processing, however, exposure-referred SNR is not commonly used. The goal of this section is to show that the SNR we showed in the main article is equivalent to the exposure-referred SNR.

\begin{figure}[h]
\centering
\includegraphics[width=0.45\textwidth]{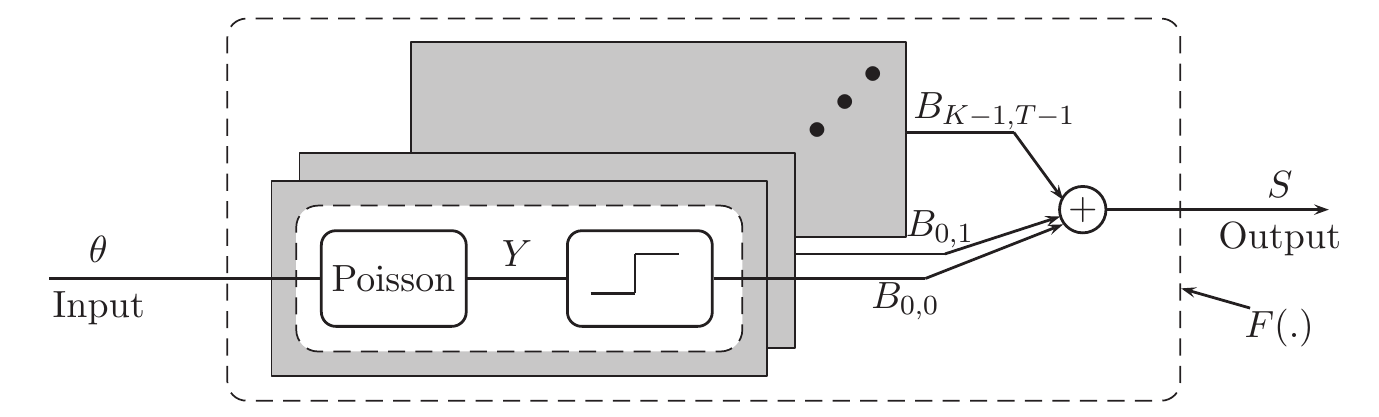}
\caption{Block diagram illustrating a QIS with input-output relation $\textrm{output}=F(\textrm{input})$}
\label{fig:QISsystem}
\end{figure}

To understand the exposure-referred SNR, we have to first understand two common ways of defining a signal to noise ratio. Consider the truncated Poisson part of the QIS model shown in \fref{fig:QISsystem}. The input to this model is the over-sampled measurement $\theta$. The truncated Poisson process can be considered as a black box function $F$ which takes an input $\theta$ and generates an output $S$, defined as
\begin{equation}
S=\sum_{t=0}^{T-1} \sum_{k=0}^{K-1} B_{k,t},
\end{equation}
where $\calB_n = \{B_{k,t} \;|\; k = 0,1, \ldots,K-1,\; t= 0,1,\ldots,T-1\}$ is the spatial-temporal block containing all binary bits corresponding to $\theta$. As shown in the main article, the mean and variance of $S$ are
\begin{equation}\label{eq:statS}
\E[S] = KT(1-\Psi_q(\theta)), \quad \Var[S] = KT \Psi_q(\theta)(1-\Psi_q(\theta)),
\end{equation}
respectively.

The first notion of signal-to-noise, which is the one used in CCD and CMOS, is called the output-referred SNR. $\textrm{SNR}_{\textrm{OR}}$ is defined as the ratio between the output signal and the photon shot noise. Referring to \fref{fig:QISsystem}, this is
\begin{equation}
\textrm{SNR}_{\textrm{OR}} = \frac{\mbox{output signal}}{\mbox{noise}} = \frac{\E[S]}{\sqrt{\Var[S]}} = \sqrt{KT\frac{1-\Psi_q(\theta)}{\Psi_q(\theta)}}.
\end{equation}
However, $\textrm{SNR}_{\textrm{OR}}$ fails to work for QIS because the shot noise is arbitrarily small if all bits are 1 or 0. In \cite{Fossum_2013}, Fossum called it squeezing of the noise. If we plot $\textrm{SNR}_{\textrm{OR}}$ as a function of $\theta$, then we observe that $\textrm{SNR}_{\textrm{OR}}$ approaches to infinity as $\theta$ grows.

The second notion of signal-to-noise, which is a modification of $\textrm{SNR}_{\textrm{OR}}$, is the exposure-referred SNR. $\textrm{SNR}_{\textrm{ER}}$ is the ratio between the exposure signal $\theta$ and the exposure-referred noise. This noise is defined as \cite{Fossum_2013}:
\begin{equation*}
  \textrm{Exposure-referred noise} = \frac{d\theta}{d\E[S]}\sqrt{\Var[S]}
\end{equation*}
The factor $\frac{d\theta}{d\E[S]}$ can be considered as the ``inverse" transfer function from the output to the input. $\frac{d\theta}{d\E[S]}$ can be determined by taking derivative of the expectation in (\ref{eq:statS}) with respect to $\E[S]$
\begin{equation*}
  \frac{d\E[S]}{d\E[S]}=\frac{d KT\left(1-\Psi_q(\theta)\right)}{d\E[S]}
\end{equation*}
Using chain rule, we observe that
\begin{equation*}
  1=-KT\frac{d}{d\theta}\Psi_q(\theta) \frac{d\theta}{d\E[S]}
\end{equation*}
Since $\frac{d}{d\theta}\Psi_q(\theta)=\frac{-e^{-\theta}\theta^{q-1}}{\Gamma(q)}$, it holds that
\begin{equation*}
  1=-KT\left(\frac{-e^{-\theta}\theta^{q-1}}{\Gamma(q)}\right) \frac{d\theta}{d\E[S]}
\end{equation*}
Hence,
\begin{equation*}
  \frac{d\theta}{d\E[S]} = \frac{\Gamma(q)}{KTe^{-\theta}\theta^{q-1}}
\end{equation*}
The exposure-referred SNR is defined as
\begin{align*}
  \textrm{SNR}_{\textrm{ER}} &=  \frac{\mbox{exposure signal}}{\mbox{exposure-referred noise}} \\
  &= \frac{\theta}{\sqrt{\Var[S]}\frac{d\theta}{d\E[S]}}\\
  &= \frac{e^{-\theta} \theta^q}{\Gamma(q)}\sqrt{\frac{KT}{\Psi_q(\theta)\big(1-\Psi_q(\theta)\big)}}.
\end{align*}
Taking logarithm shows that $\mathrm{SNR}_{\mathrm{ER}}$ is identical to the SNR derived from the Fisher Information shown in the main article.

\begin{figure}[t]
\centering
\begin{tabular}{cc}
\includegraphics[width=0.40\textwidth]{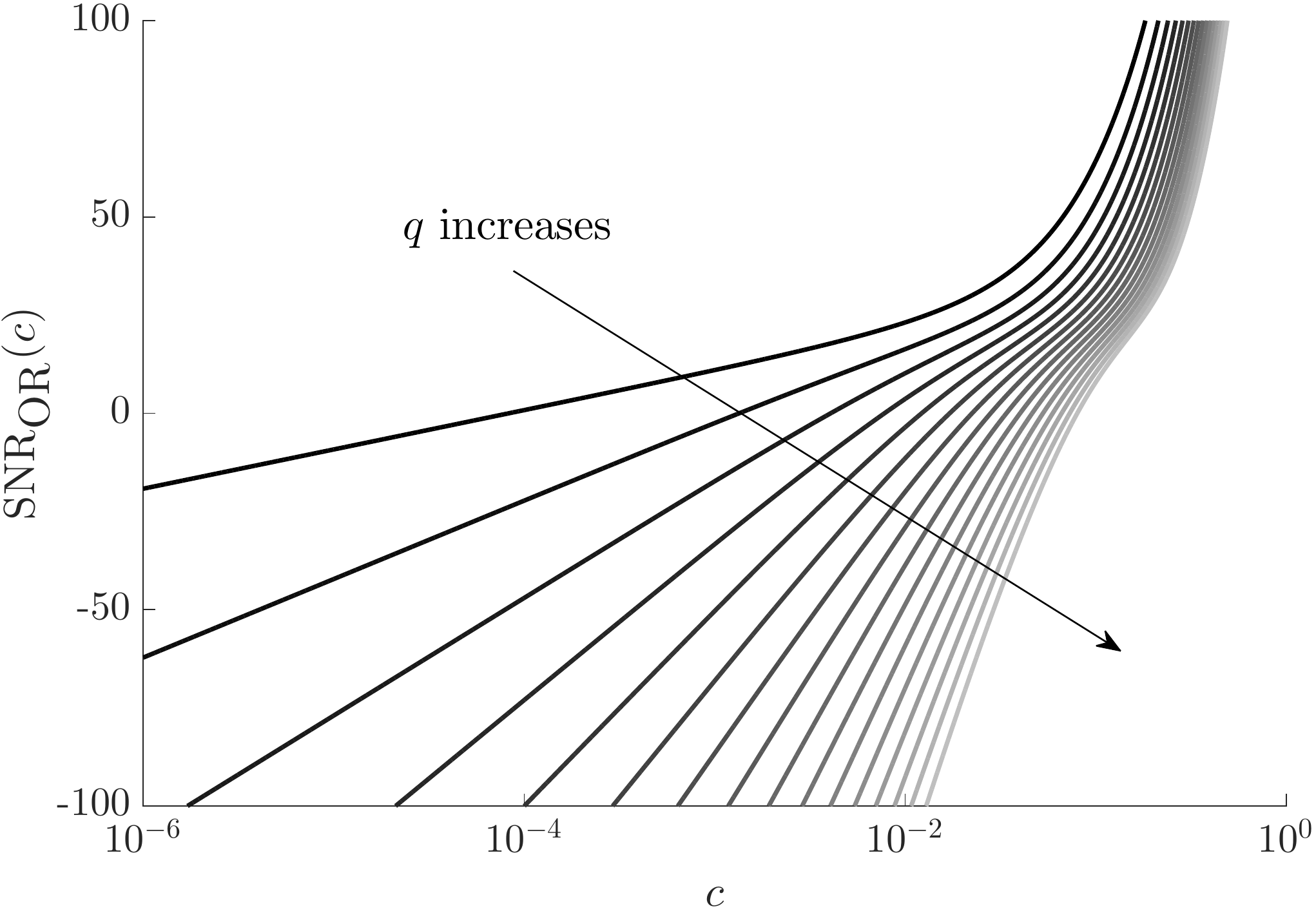}&
\includegraphics[width=0.40\textwidth]{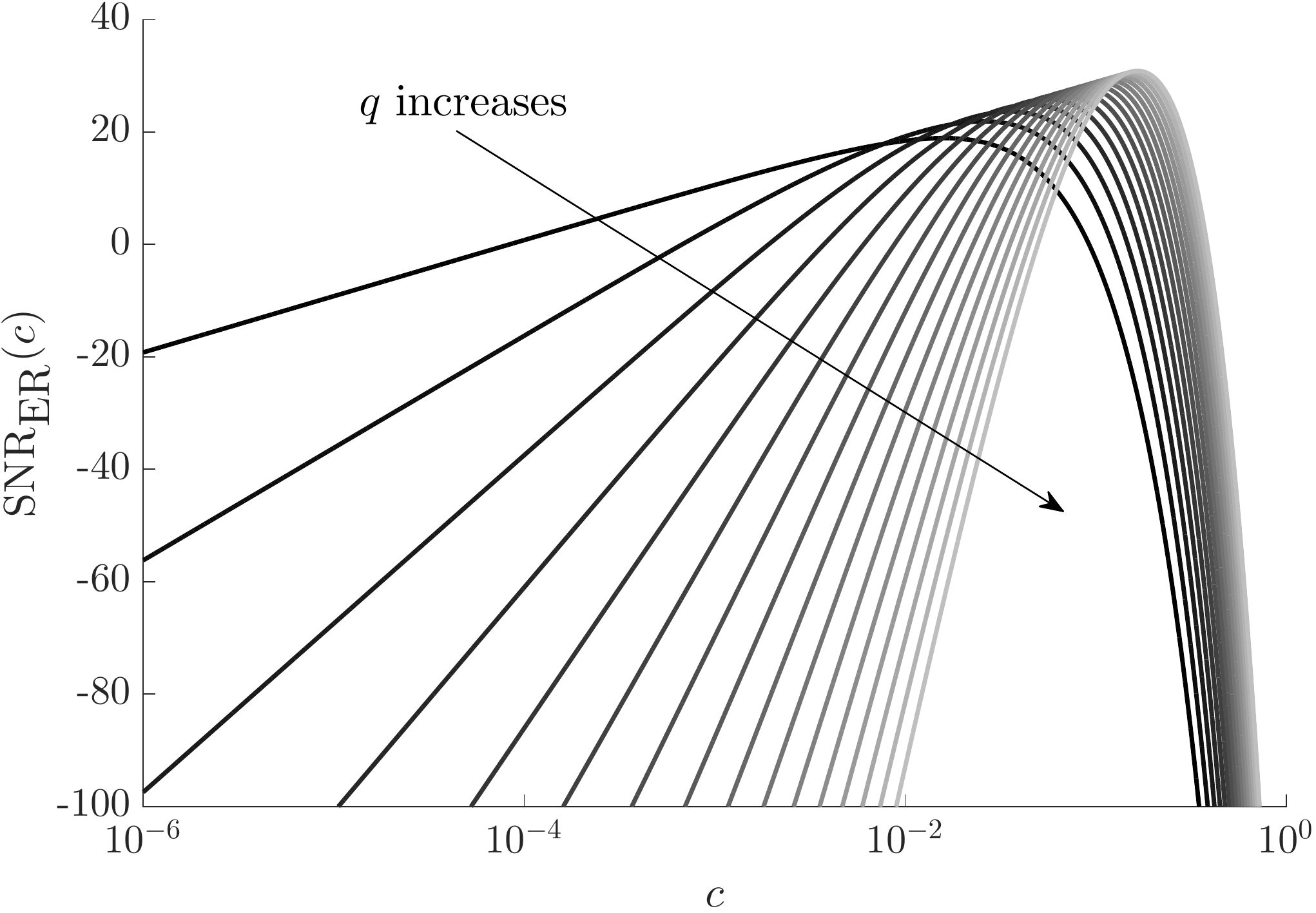}\\
\footnotesize (a) $\mathrm{SNR}_{\mathrm{OR}}$&
\footnotesize (b) $\mathrm{SNR}_{\mathrm{ER}}$
\end{tabular}
\caption{Comparison of the SNRs for $q\in\{1,\ldots,16\}$. In this experiment, we fix $\alpha=400$, $K=4$, and $T=30$.}
\label{fig:SNR}
\end{figure}

\section{Properties of the incomplete Gamma function}
In the main article, we used the incomplete Gamma function for QIS analysis. In this section, we provide more details about the properties of the incomplete Gamma function.

First, we recall that the normalized upper incomplete Gamma function ${\Psi_q:\R^{+}\rightarrow[0,1]}$ is defined as
\begin{equation}
\Psi_q(\theta) \bydef \frac{1}{\Gamma{(q)}} \int_{\theta}^{\infty}  t^{q-1} e^{-t} dt, \quad \mathrm{for}\; \theta > 0,\; q \in \N.
\end{equation}
where $\Gamma(q) = (q-1)!$ is the standard Gamma function.

In this equation, we note that $\Psi_q(\theta)$ depends on two variables: $q$ and $\theta$.
\begin{itemize}
\item As a function of $\theta$. As we showed in the main article, $\Psi_q(\theta)$ is a monotonically decreasing function of $\theta$ because the derivative is negative:
\begin{equation*}
\frac{d}{d\theta}\Psi_q(\theta) = \frac{-\theta^{q-1}e^{-\theta}}{\Gamma(q)} < 0.
\end{equation*}
However, $\Psi_q(\theta)$ is very close to 1 when $\theta$ is small, and is very close to 0 when $\theta$ is large. Therefore, there exists a range of $\theta$ in which $\Psi_q(\theta)$ can attain a reasonably good inverse. We define this set as the $\theta$-admissible set
\begin{equation}
\Theta_q \bydef \{ \theta \;|\; \varepsilon \le \Psi_q(\theta) \le  1-\varepsilon\},
\end{equation}
for any fixed $q$ and a tolerance level $\varepsilon$. An illustration of $\Theta_q$ is shown in \fref{fig:Psi}.
\item As a function of $q$. The incomplete Gamma function $\Psi_q(\theta)$ can also be considered as a function of $q$. In this case, $\Psi_q(\theta)$ is only defined for integer values of $q$. We illustrate the behavior of $\Psi_q(\theta)$ as a function of $q$ in \fref{fig:Psi}. The set of $q$ in which $\Psi_q(\theta)$ is sufficiently away from 0 and 1 is defined as the $q$-admissible set.
\begin{equation}
\calQ_\theta \bydef \{q \;|\; \varepsilon \le \Psi_q(\theta) \le  1-\varepsilon\}.
\end{equation}
\end{itemize}

\begin{figure}[t]
\centering
\begin{tabular}{cc}
\includegraphics[width=0.40\textwidth]{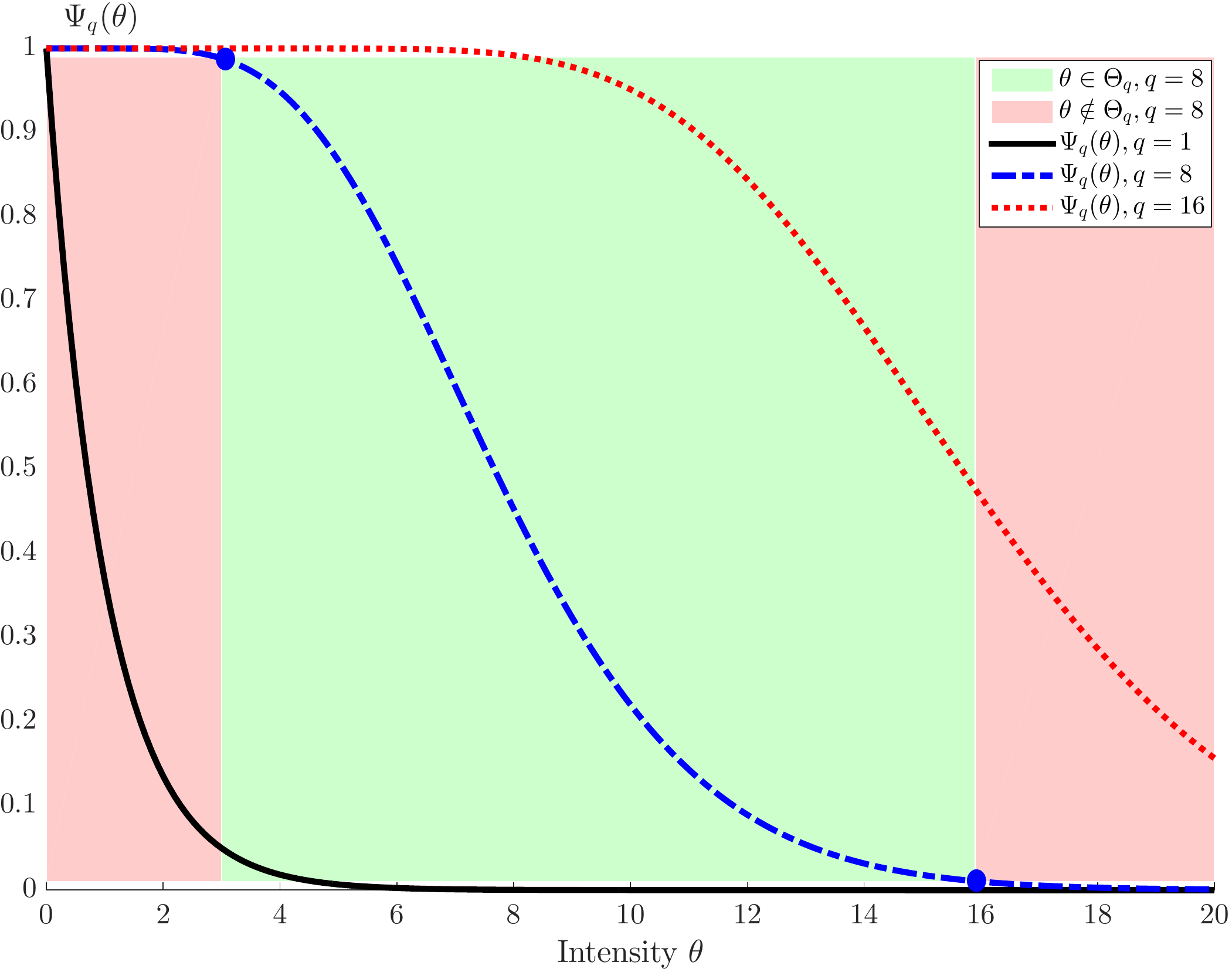}&
\includegraphics[width=0.40\textwidth]{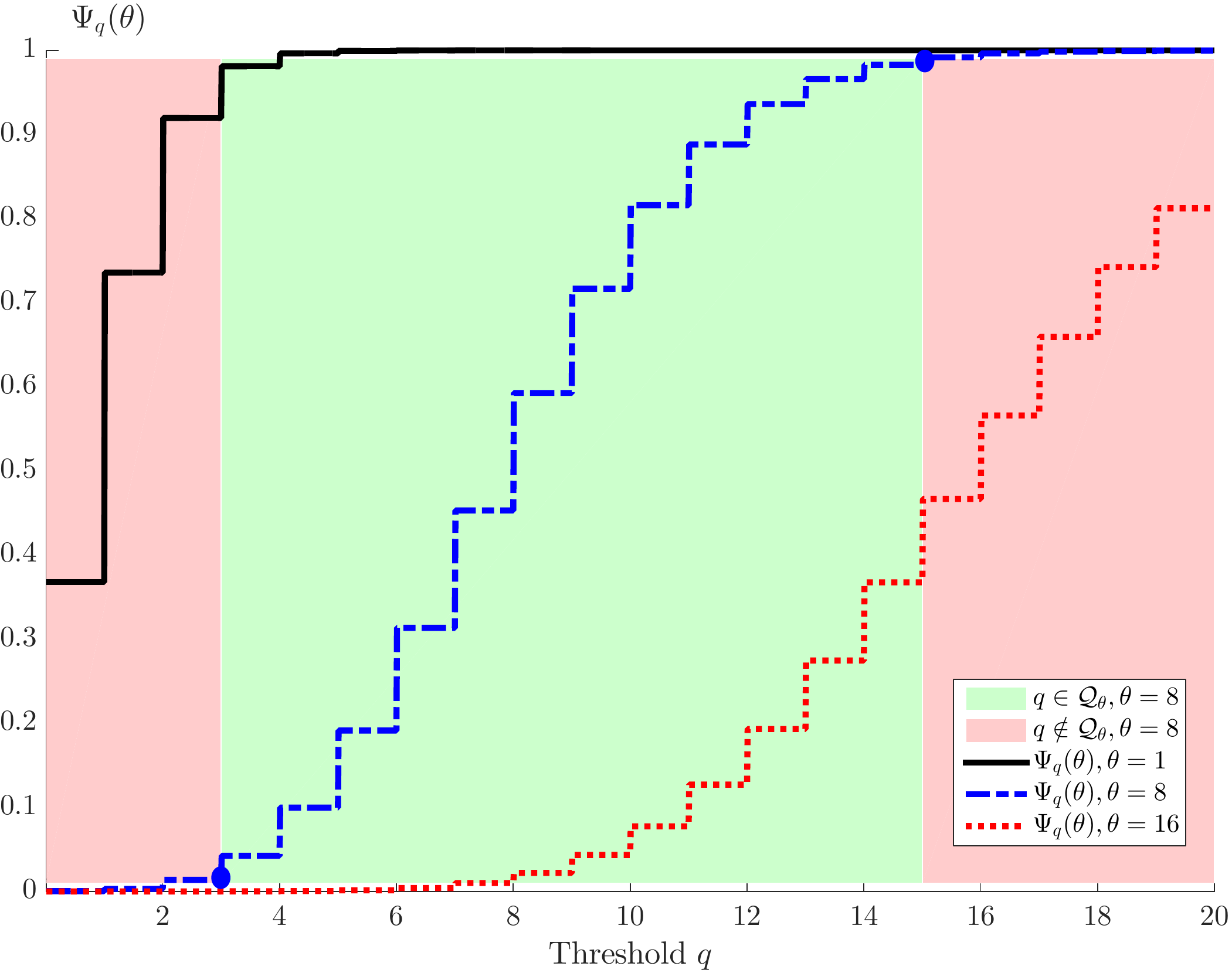}\\
\footnotesize (a) $\Psi_q(\theta)$ vs. $\theta$&
\footnotesize (b) $\Psi_q(\theta)$ vs. $q$
\end{tabular}
\caption{$\Psi_q(\theta)$ as a function of $\theta$ and $q$. In defining, $\calQ_\theta$ and $\Theta_q$, we set $\epsilon=0.01$. }
\label{fig:Psi}
\end{figure}

{\color{black}
\section{Comparison with the threshold design scheme by Yang \cite{Yang_2012}}

In this section, we compare our threshold scheme with the one in \cite{Yang_2012}.

First, we recall that the optimality of our method is based on a lower-bound $L_q(c)$ for the per-pixel SNR:
\begin{equation} \label{eq:oracle}
q^*(c) = \argmax{q \in \N} \; \textrm{SNR}_q(c) \approx \argmax{q \in \N} \; L_q(c) = \left\lfloor \frac{\alpha c}{K} \right\rfloor
\end{equation}
Therefore, the optimal threshold is a function of $c$, which changes in space and in time.

In contrast, \cite{Yang_2012} uses a checkerboard pattern by alternating two thresholds ($q^*_1$, $q^*_2$). These two thresholds are obtained by maximizing the Cram\'{e}r-Rao lower bound (CRLB) over a range of light intensity values $[c_{\min},c_{\max}]$:
\begin{equation}
(q^*_1, q^*_2) = \underset{1\leq q_1,q_2 \leq q_{\max}}{\textrm{argmin}} \int_{c_{\min}}^{c_{\max}} CRLB(q_1,q_2,c)\; dc.
\end{equation}
As a result, the threshold is optimal in the \emph{average sense}. To compare the two approaches, we followed the same steps in \cite{Yang_2012} to obtain $CRLB(q_1,q_2,c)$ for a checkerboard pattern in terms of $\Psi_q(c)$ as follows
\begin{equation}
CRLB(q_1,q_2,c)  = \sum_{i=1}^2 \; \frac{\alpha^2}{2K} \frac{e^{-2\theta}\theta^{(2(q_i-1))}}{\Gamma(q_i)^2 \Psi_{q_i}(\theta) \left[1-\Psi_{q_i}(\theta)\right]}
\end{equation}
where $\theta=\alpha c /K$. Using the parameters $\alpha=K(q_{\max-1})$, $q_{\max}=16$, $K=4$, and using trapezoidal technique for numerical integration over $c$, we obtain that $q^*_1=4$ and $q^*_2=12$. Figure~\ref{fig:result1} shows the reconstructed images using uniform threshold maps with thresholds $q\in\{1,5,8,10,15\}$, the checkerboard threshold map in \cite{Yang_2012} with $q^*_1=4$ and $q^*_2=12$, and the oracle threshold map obtained by (\ref{eq:oracle}).  In this experiment, our proposed method achieves $28.15$ dB, which is $0.83$ dB higher than the checkerboard pattern.}

\begin{figure}[!]
\centering
\begin{tabular}{cccc}
\includegraphics[width=0.22\linewidth]{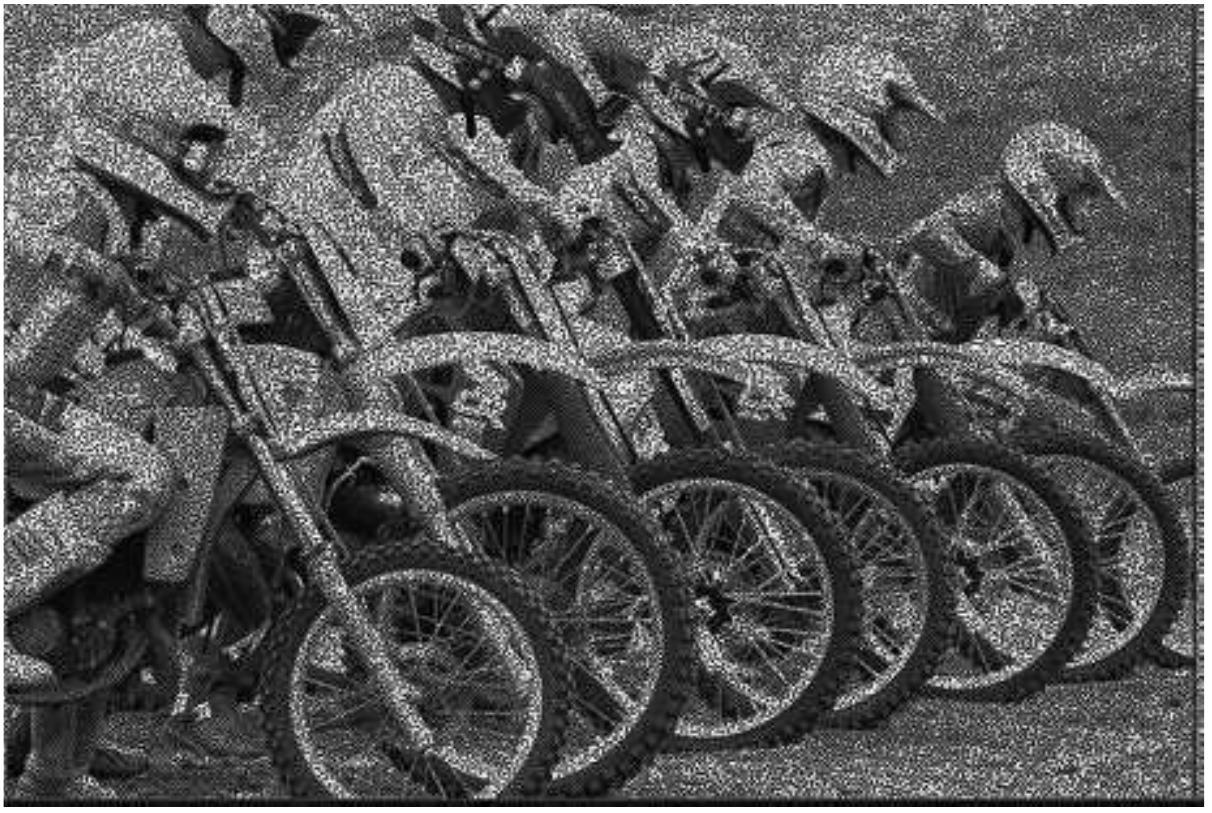}&
\includegraphics[width=0.22\linewidth]{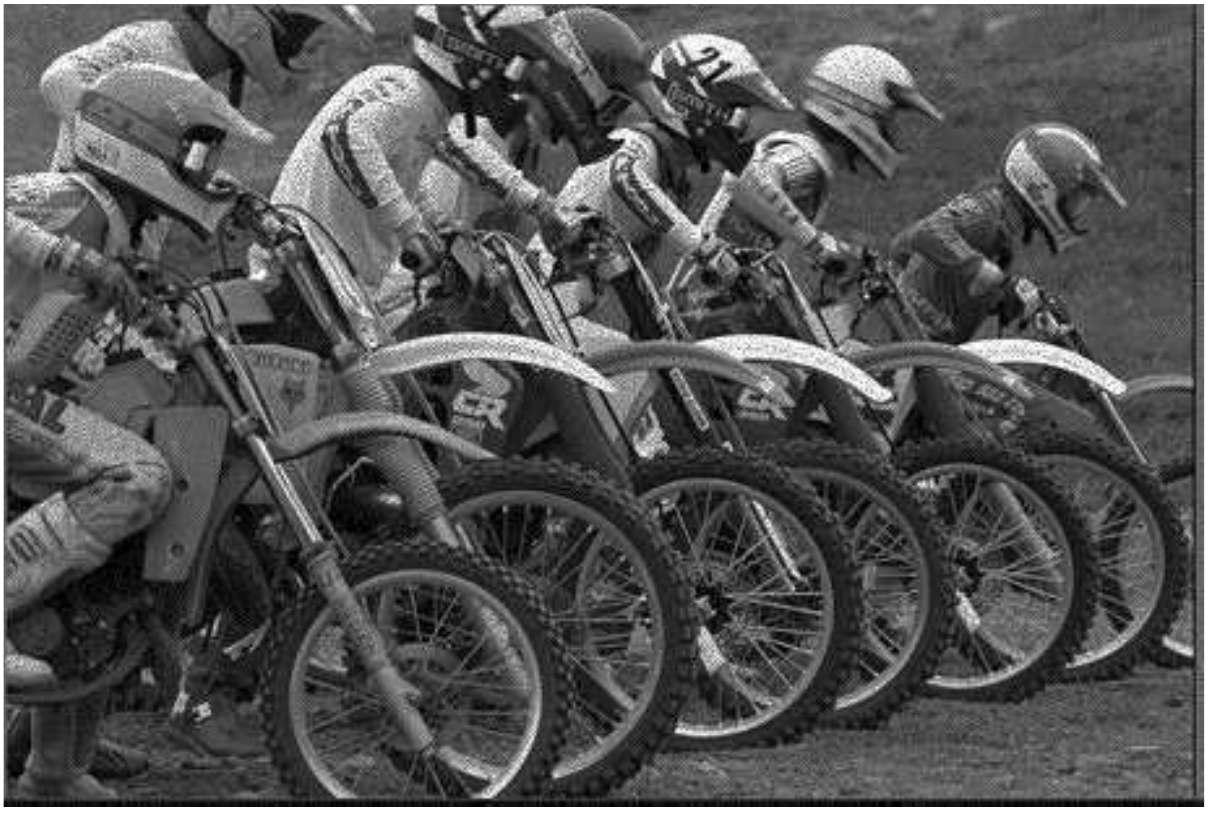}&
\includegraphics[width=0.22\linewidth]{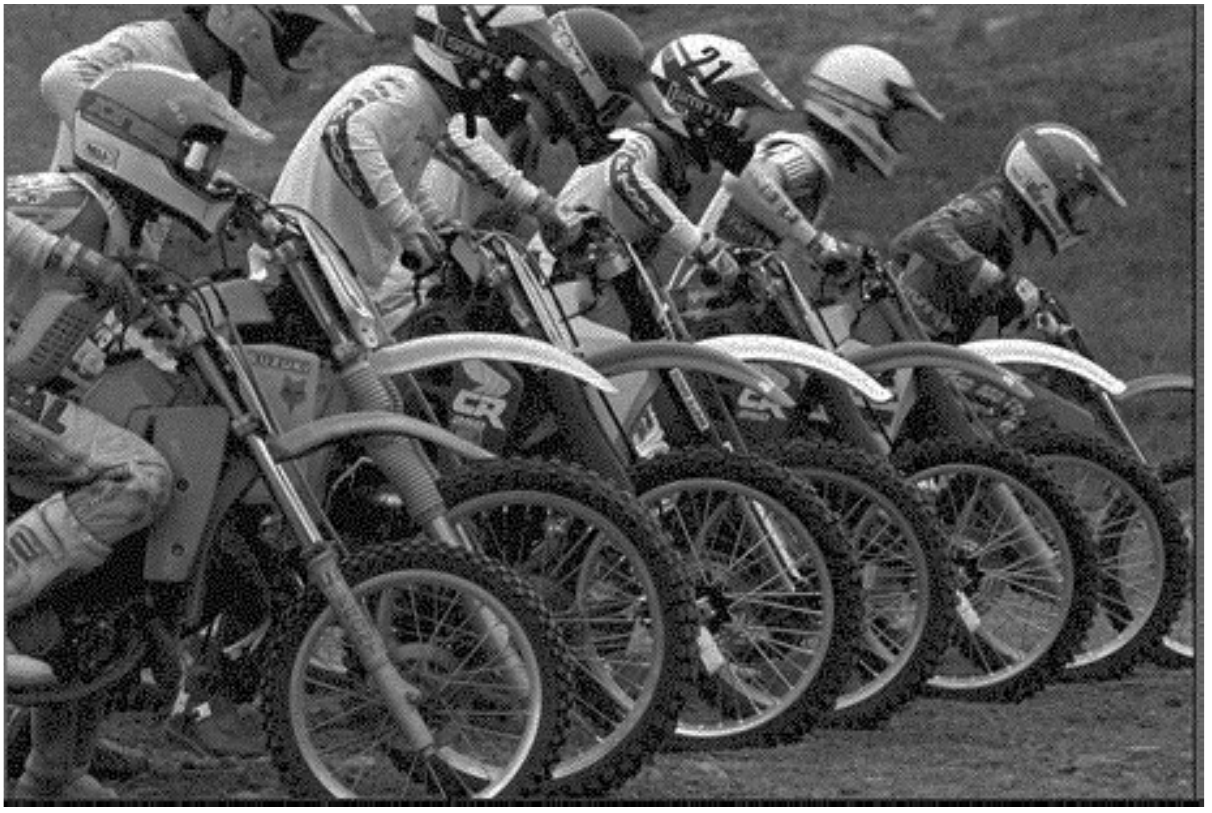}&
\includegraphics[width=0.22\linewidth]{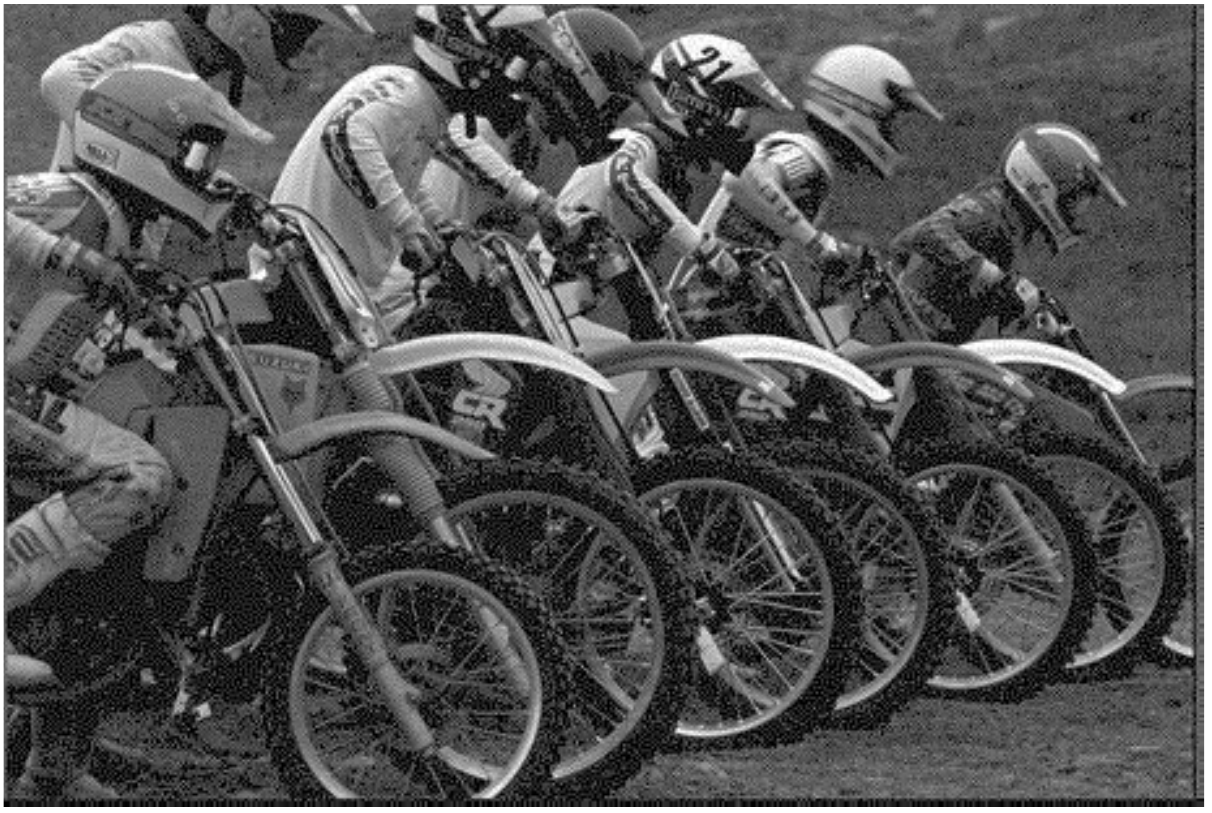}\\
(a) $q=1$,  15.97dB &
(b) $q=5$,  27.80dB &
(c) $q=8$,  25.84dB &
(d) $q=10$, 22.36dB\\
\includegraphics[width=0.22\linewidth]{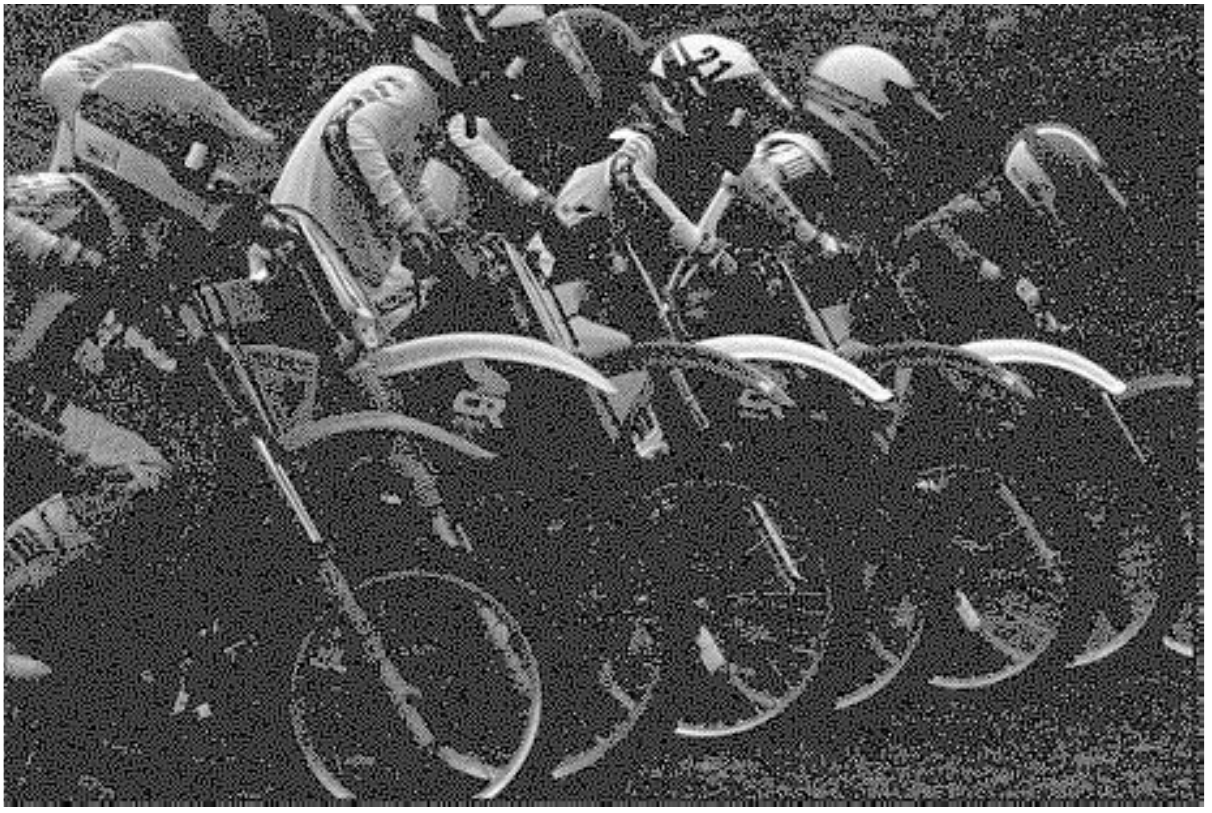}&
\includegraphics[width=0.22\linewidth]{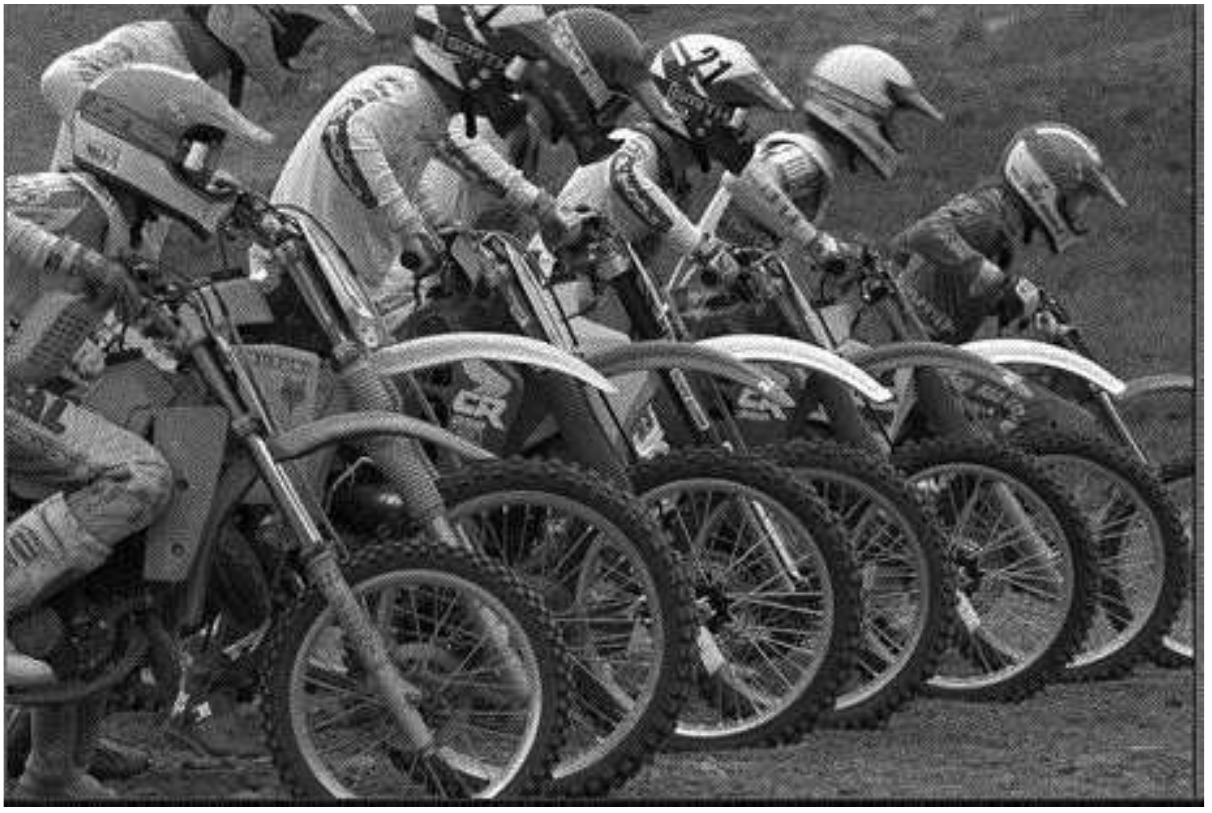}&
\includegraphics[width=0.22\linewidth]{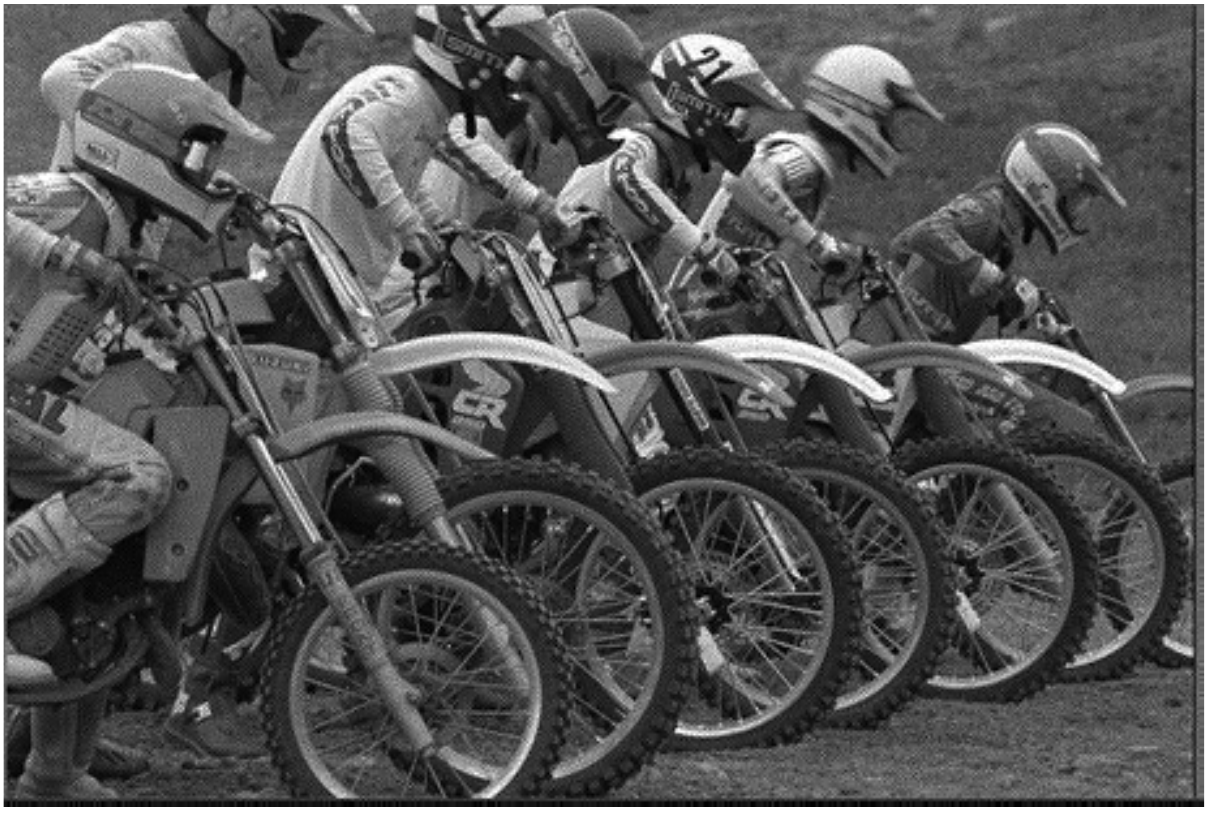}&
\includegraphics[width=0.22\linewidth]{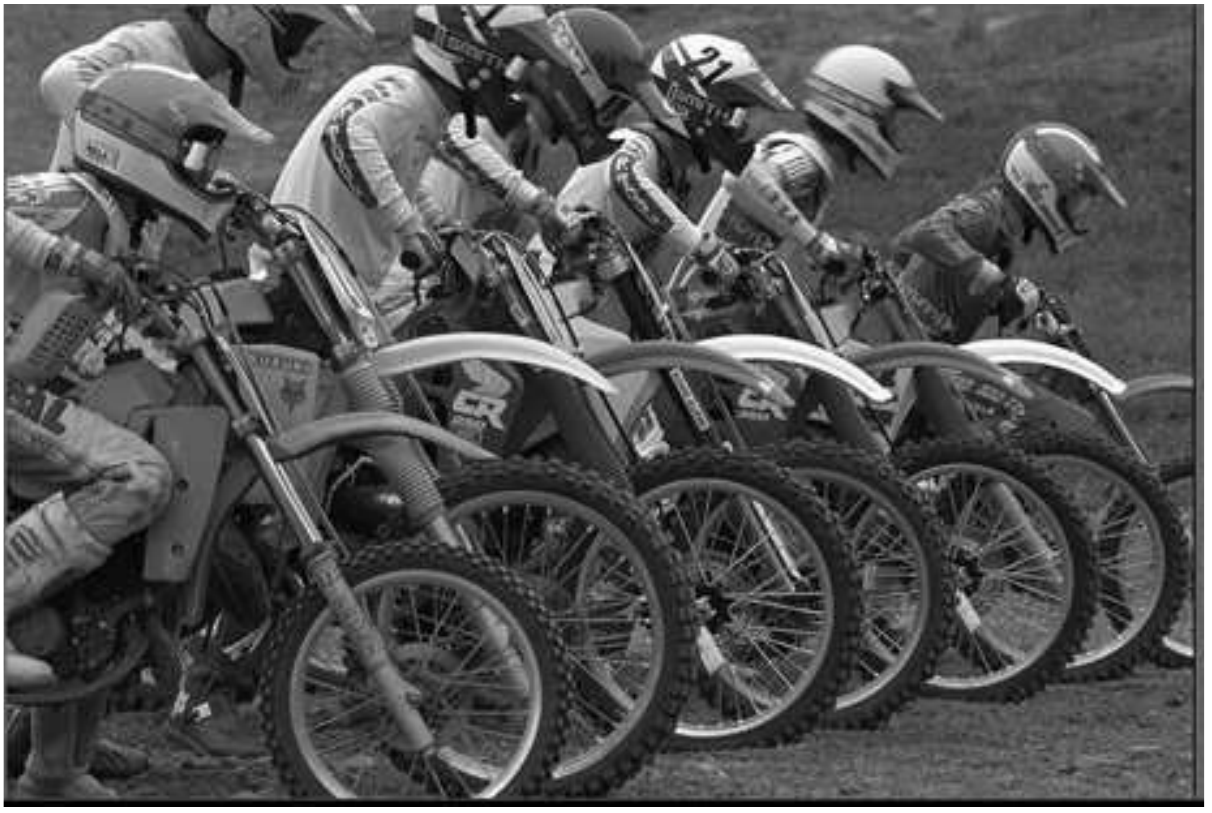}\\
(e) $q=15$, 15.35dB &
(f) \cite{Yang_2012}, 27.32dB&
(g) Proposed, 28.15dB&
(h) Ground Truth
\end{tabular}
\caption{Spatial oversampling $K=4$. Temporal oversampling $T=20$. Quadratic B-spline kernel is used in synthesis and reconstruction models. Gradient descent is used to obtain the ML estimate. For bisection threshold map, 8 frames are used for adapting the map, and 12 frames are used for reconstruction. For all other maps, the whole 20 frames are used for reconstruction. }
\label{fig:result1}
\vspace{-2ex}
\end{figure}

\section{Phase transition under different configurations}\label{sec:SRI}
In the main article, we showed the phase transition behavior of the ML estimate using $K=4$, $T=50$, and $\delta=2\times 10^{-4}$. In this section, we study the effect of changing $K$, $T$, and $\delta$ on the phase transition region width.

\textbf{As a function of $T$}. \fref{fig:PsiT1a}-\fref{fig:PsiT1b} illustrate the phase transition behavior when $T = 10, 25, 50,$ and $100$. As $T$ increases, the width of the green region increases. However, if we fix the range of the bit density $1-\E[\gamma_q(c)]$, we observe that the SNR does not vary significantly even as $T$ changes.

\textbf{As a function of $K$}. The spatial oversampling $K$ affects both the threshold ${q^{*}(c)=\lfloor\alpha c/K\rfloor}+1$ and the phase transition width. \fref{fig:QthetaKT}(a) illustrates the behavior of the threshold $q^*$ as a function of $K$. As $K$ increases, $q^*$ decreases. However, the optimal $q^*$ still stays within the set $\calQ_\theta$.

\textbf{As a function of $\delta$}. The constant $\delta$ is used to define the set $\calQ_\theta$:
\begin{equation}
\calQ_\theta \bydef \left\{q \;\Big|\; 1-\left(\frac{\delta}{2}\right)^{\frac{1}{KT}} \le \Psi_q(\theta) \le \left(\frac{\delta}{2}\right)^{\frac{1}{KT}}\right\}.
\end{equation}
The constant $\delta$ is the tolerance level. When $\delta$ increases, the size of the set $\calQ_\theta$ should also increase. This result is shown in Figure~\ref{fig:QthetaKT}(b).

Using the closed form expression of the average bit density $1-\Psi_q(\theta)$, we can calculate the average bit density at the optimal threshold $q^*=\lfloor\theta\rfloor+1$, which is shown in Figure~\ref{fig:bitdensity}. We notice that as long as $\theta\geq 1$, the average bit density is between $0.264$ and $0.630$. Within this range, we observe from \fref{fig:PsiT1a}-\fref{fig:PsiT1b} that the SNR does not vary significantly if the estimated threshold is deviated from the optimal threshold. This observation relaxes the requirement of the bisection method from obtaining the exact optimal threshold to obtaining a threshold that make the bit density equal to $0.5$. Since $0.5\in[0.264,0.630]$, we guarantee to achieve an SNR which is sufficiently close to the optimal SNR.

Controlling $\theta\geq1$ can be achieved by tuning the constant $\alpha$. Tuning $\alpha$ can be hardware-implemented by increasing the exposure period. Intuitively what $\theta\geq1$ requires is that the average number of impinging photons per jot must be at least one. If $\theta$ is less than one, then most bits will become zeros. Increasing exposure period (i.e., increasing $\alpha$) will ensure sufficient number of photons.

\begin{figure}[h]
\centering
\begin{tabular}{cc}
\includegraphics[width=0.40\textwidth]{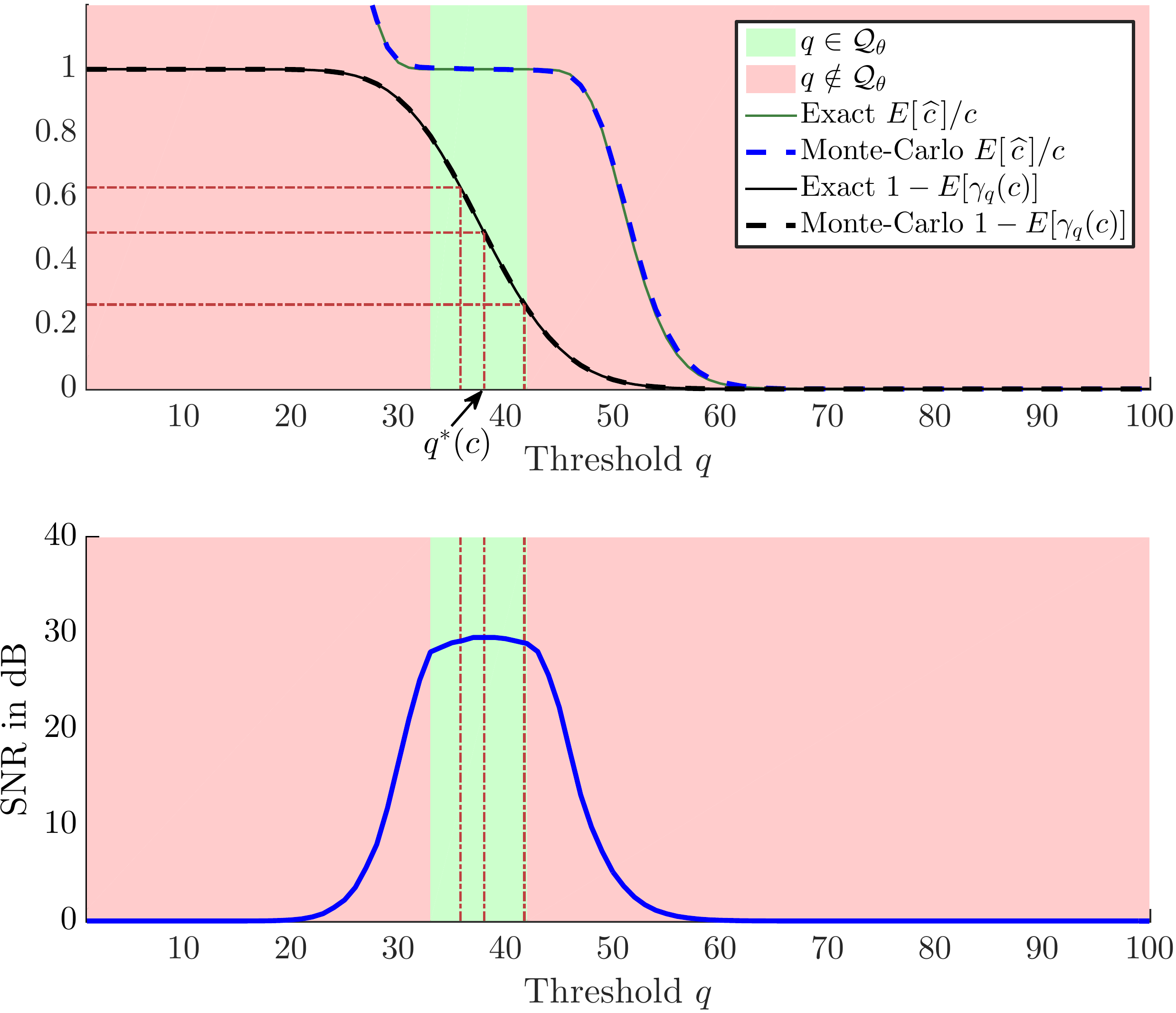}&
\includegraphics[width=0.40\textwidth]{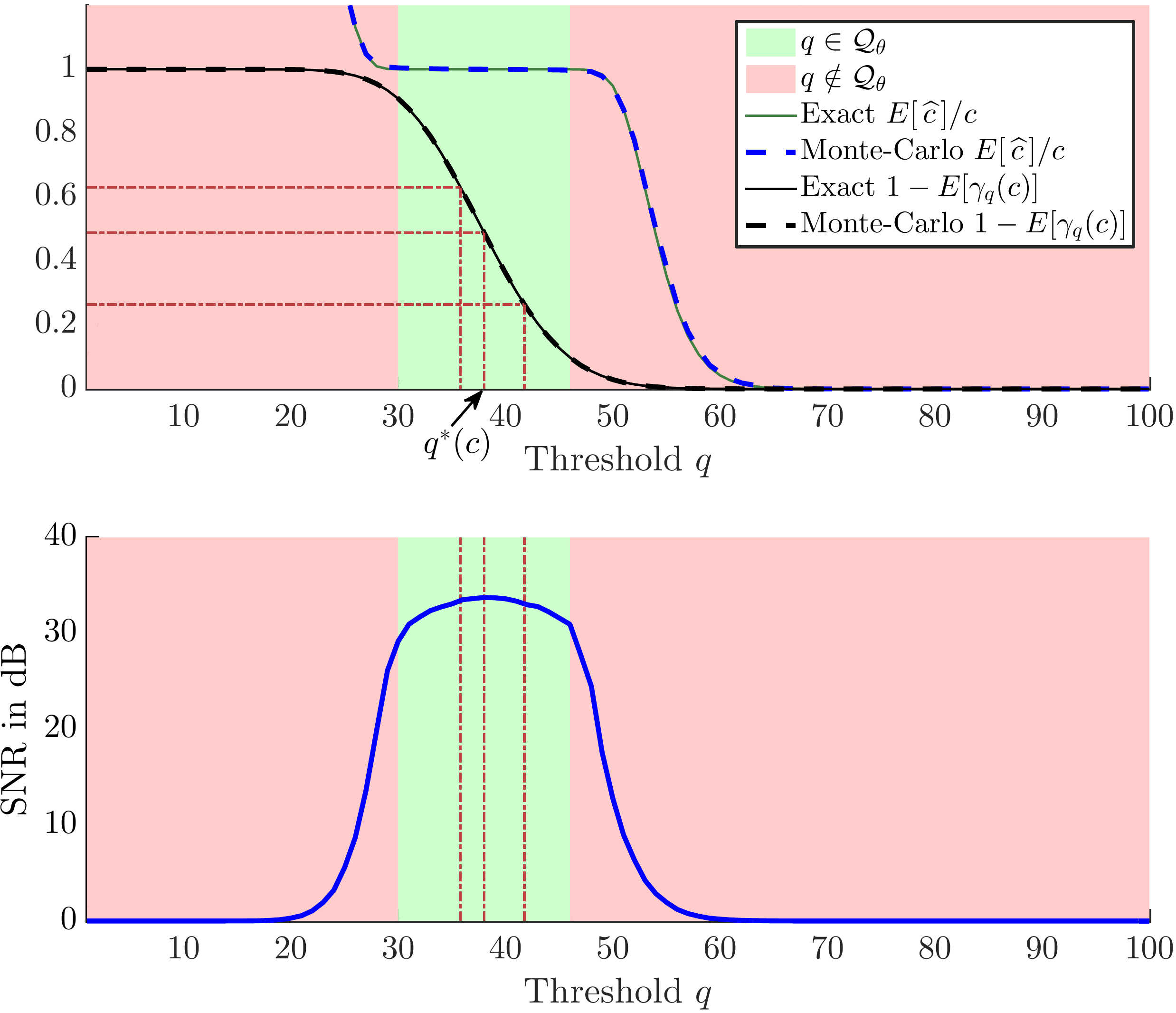}\\
\footnotesize (a) $T=10$, $\textrm{SNR}\in[28.93,29.57]$ &
\footnotesize (b) $T=25$, $\textrm{SNR}\in[32.98,33.74]$
\end{tabular}
\caption{Phase transition for $T=10$ and $T=25$. SNR range is shown for average bit density $1-\E[\gamma_q(c)]$ in the range $[0.264,0.630]$. For all cases, we set $\delta=2\times 10^{-4}$, and $K=4$.}
\label{fig:PsiT1a}
\end{figure}

\begin{figure}[h]
\centering
\begin{tabular}{cc}
\includegraphics[width=0.40\textwidth]{SRI.pdf}&
\includegraphics[width=0.40\textwidth]{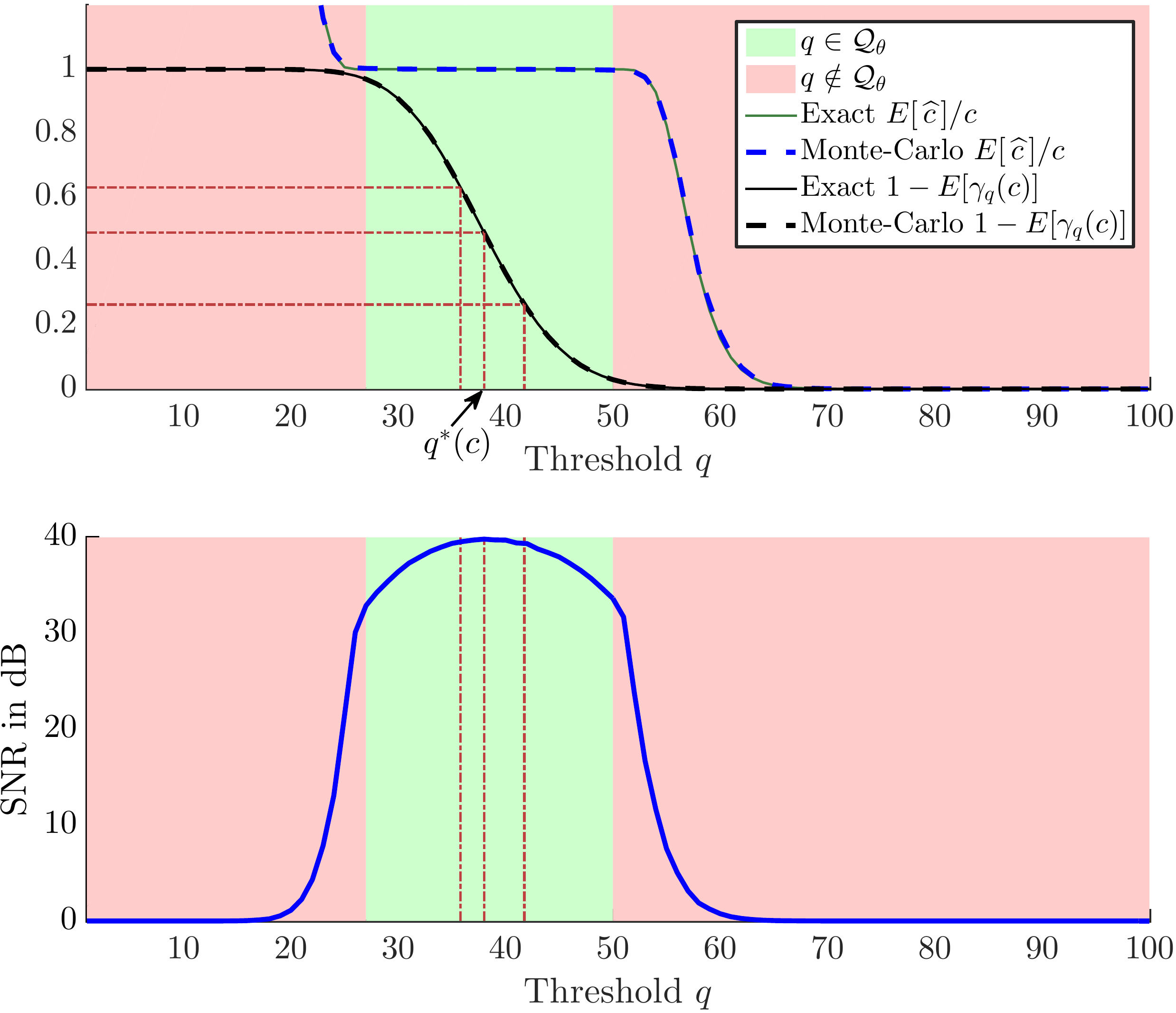} \\
\footnotesize (a) $T=50$, $\textrm{SNR}\in[36.15,36.80]$ &
\footnotesize (b) $T=100$, $\textrm{SNR}\in[39.35,39.82]$
\end{tabular}
\caption{Phase transition for $T=50$ and $T=100$. SNR range is shown for average bit density $1-\E[\gamma_q(c)]$ in the range $[0.264,0.630]$.  For all cases, we set $\delta=2\times 10^{-4}$, and $K=4$.}
\label{fig:PsiT1b}
\vspace{-1.0ex}
\end{figure}

\begin{figure}[ht]
\centering
\begin{tabular}{cc}
\includegraphics[width=0.45\textwidth]{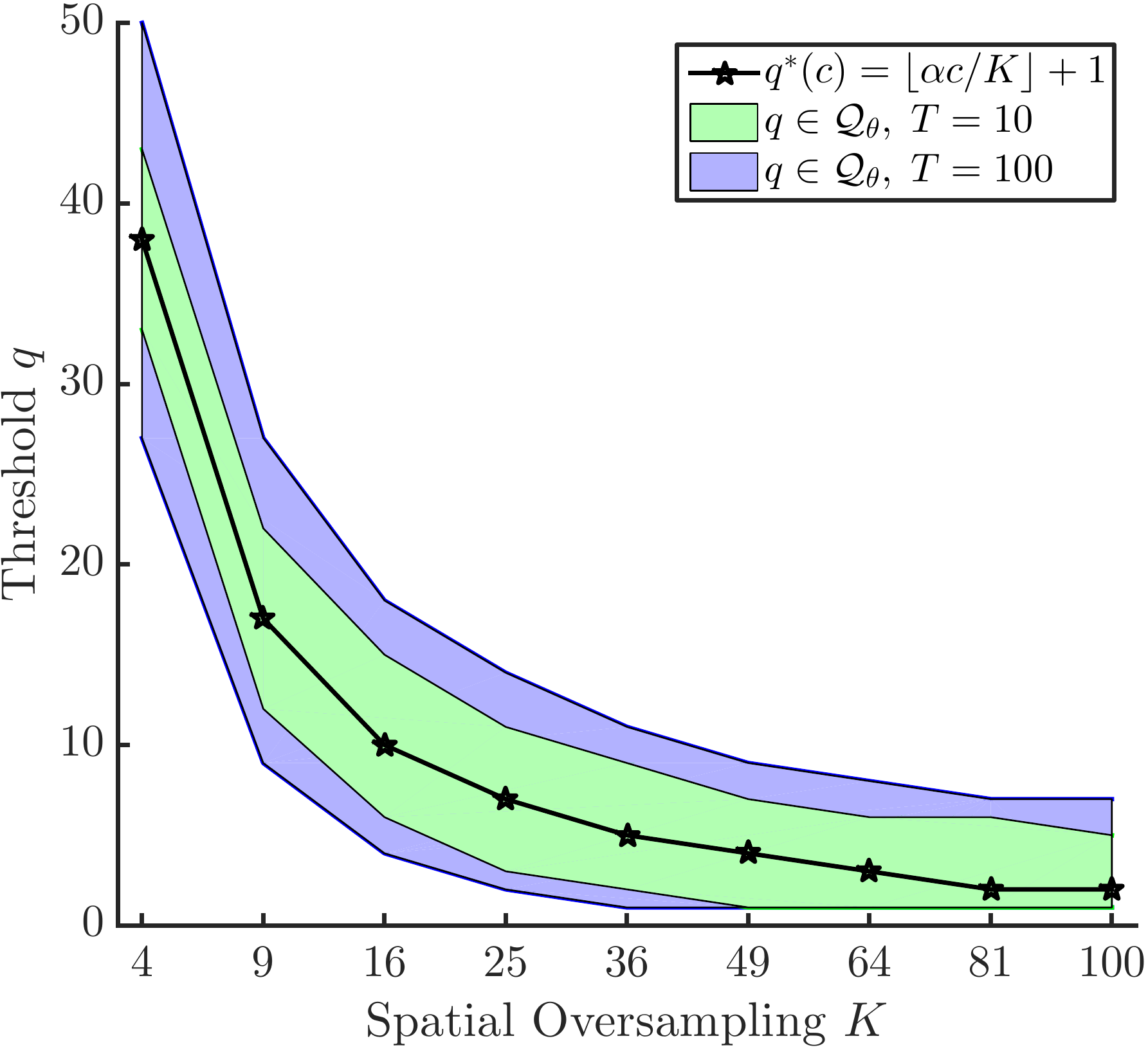}&
\includegraphics[width=0.45\textwidth]{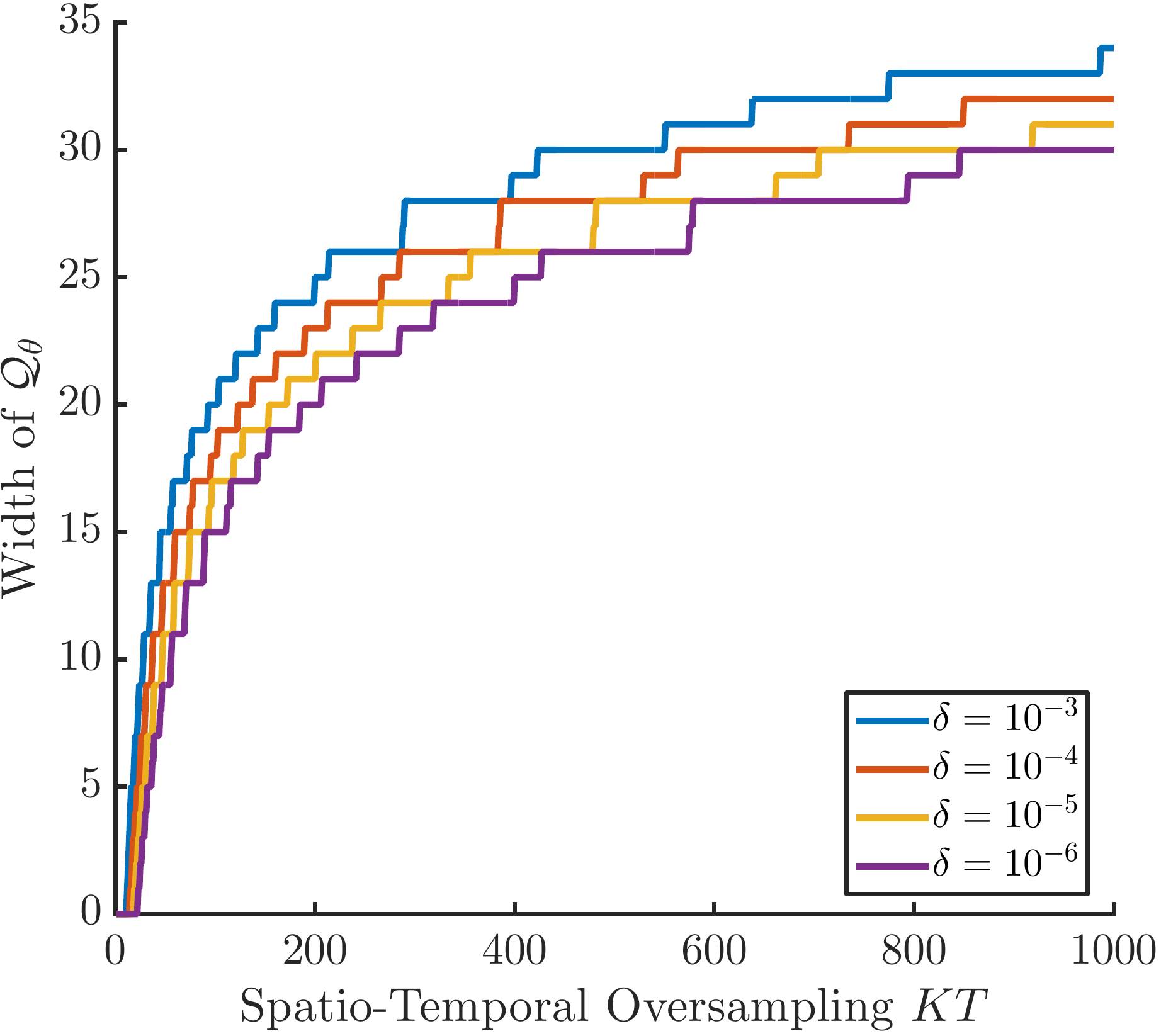}\\
\footnotesize (a) &
\footnotesize (b)
\end{tabular}
\caption{(a) The threshold $q$ and $\calQ_\theta$ as $K$ increases. (b) The width of $\calQ_\theta$ as $KT$ and $\delta$ changes.}
\label{fig:QthetaKT}
\end{figure}

\begin{figure}[ht]
\centering
\includegraphics[width=0.5\textwidth]{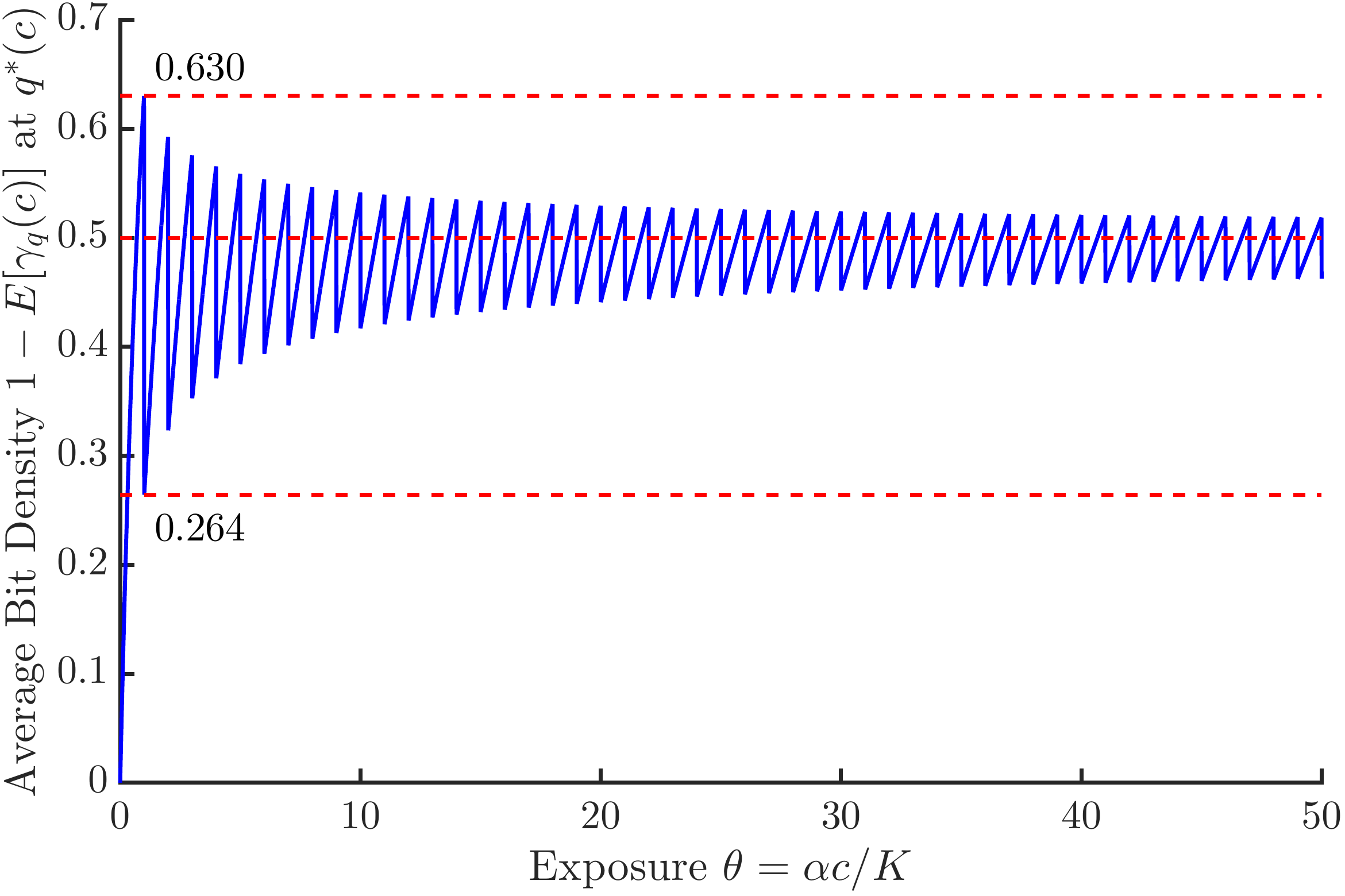}
\caption{Average bit density $1-\E[\gamma_q(c)]$ calculated at optimal threshold $q^*=\lfloor\theta\rfloor+1$.}
\label{fig:bitdensity}
\end{figure}

{\color{black}
\section{Influence of Non-Boxcar Kernel $\mathbf{G}$}

In this section, we discuss the boxcar kernel assumption in QIS model, i.e., $\mG=\frac{1}{K} \mI_{N\times N}  \otimes \mathbf{1}_{K \times 1}$. We also study the effect of assuming a general kernel $\mG$ on our results.

On QIS, we typically assume that there are micro-lenses on top of each jot or a group of jots. These micro-lenses ensure that the incident light converges onto the sensing site with no (or very minor) interference with adjacent jots or groups. As a result, we can model the incoming light using the boxcar kernel. This assumption is perhaps strong in some perspective, but it allows us to significantly simplify the theory and offer efficient implementations.

What if there is a mismatch between the physical model (e.g., using B-spline or Gaussian kernel $\mG$) and the reconstruction (e.g., using boxcar)? To see the effect of this mismatch on the reconstruction quality, we conduct two sets of experiments.
\begin{itemize}
\item \textbf{1D Signal}: We consider a 1D signal with 10 coefficients. These 10 coefficients are modulated with boxcar kernels and B-spline kernels to generate two sets of incident light. On the QIS simulator, we set the spatial and temporal oversampling factors as $K=9$ and $T=30$, respectively. Then we use the oracle threshold map for quantization. To reconstruct the images, we use boxcar kernel for both cases so that we have one matching case and one mismatching case. Figure~\ref{fig:result2} shows the reconstructed signals. As expected, when the forward model matches with the reconstruction model, the reconstructed image has the highest PSNR. However, the gap between the cases are not significant.
\item \textbf{2D Signal}: Figure~\ref{fig:result3} shows a 2D example. Similar to the 1D case, boxcar kernel leads to the best reconstruction but its gap with the other cases are not significant.
\end{itemize}

\noindent The reader might think why we do not use B-spline on the reconstruction so that it will match with the forward model? In principle this is doable, but we need an iterative algorithm to compute the ML estimate such as gradient descent as reported in \cite{Yang_2012}. In contrast, the boxcar assumption allows us to use a closed-form ML estimate, which is practically much more affordable.

\begin{figure}[ht]
 \centering
\begin{tabular}{cc}
\includegraphics[width=0.35\linewidth]{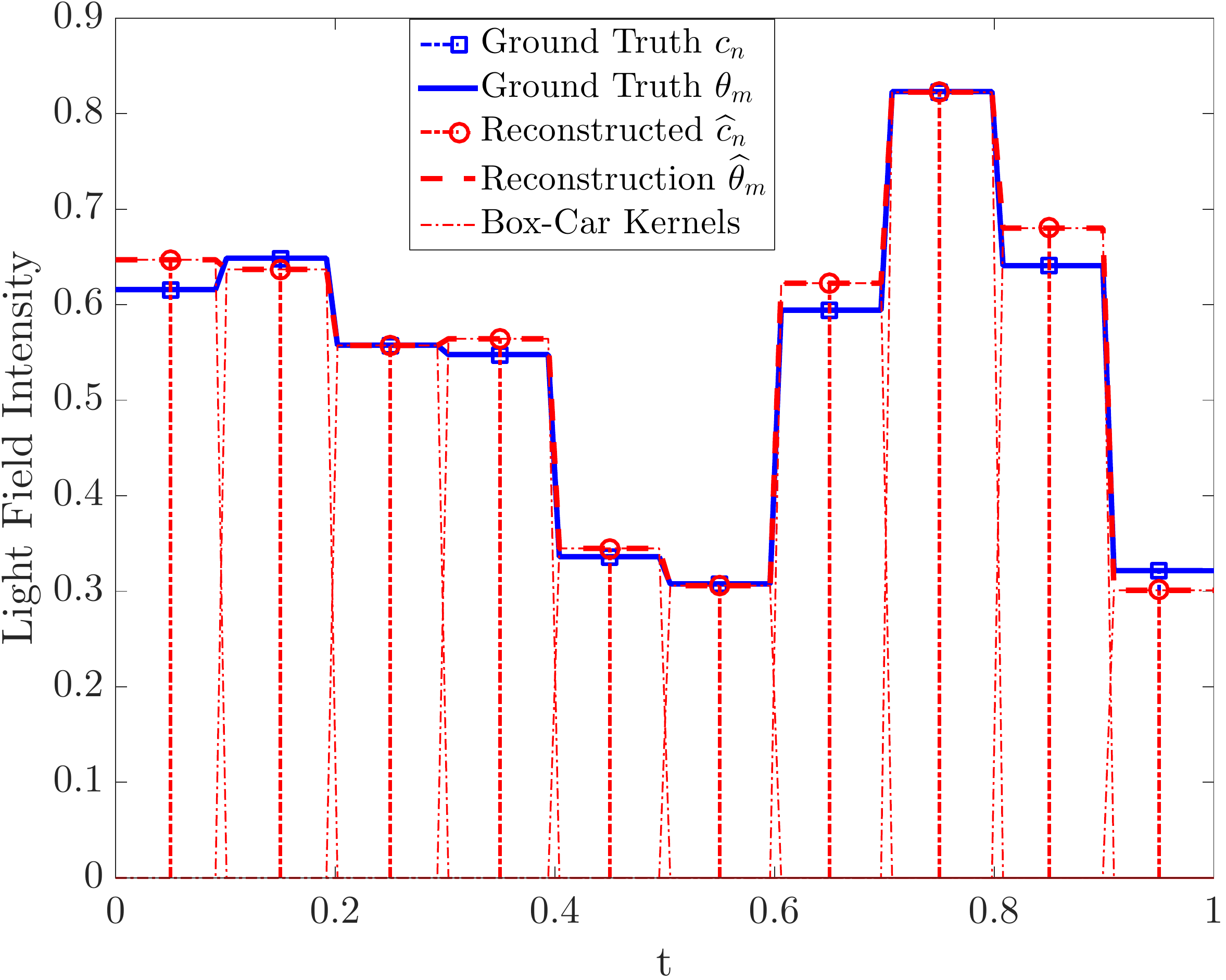} &
  \includegraphics[width=0.35\linewidth]{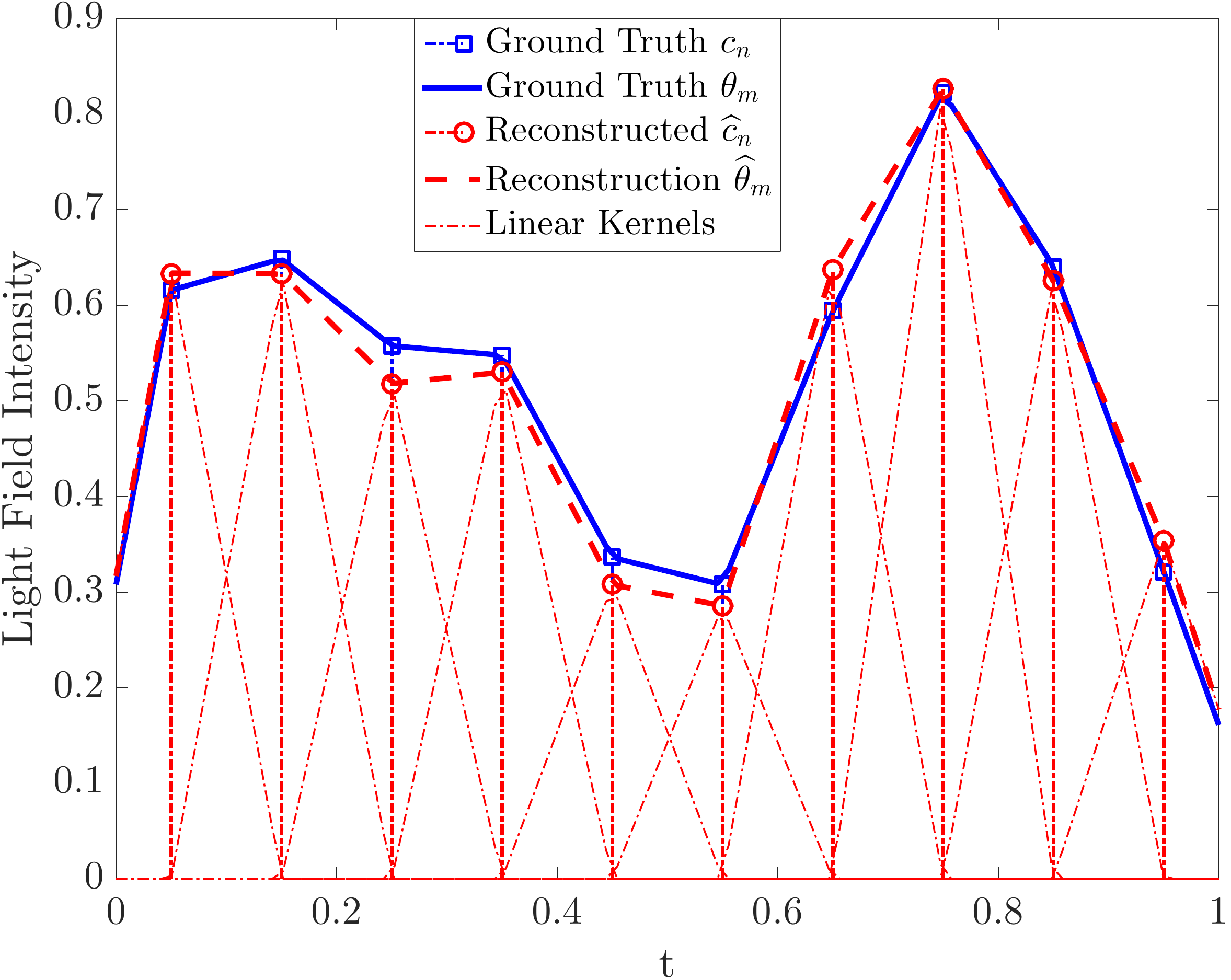}\\
\footnotesize (a) Boxcar kernel,  PSNR$ = 33.74$  dB&
  \footnotesize (b) Linear B-spline,  PSNR$ = 31.62$  dB \\
   \includegraphics[width=0.35\linewidth]{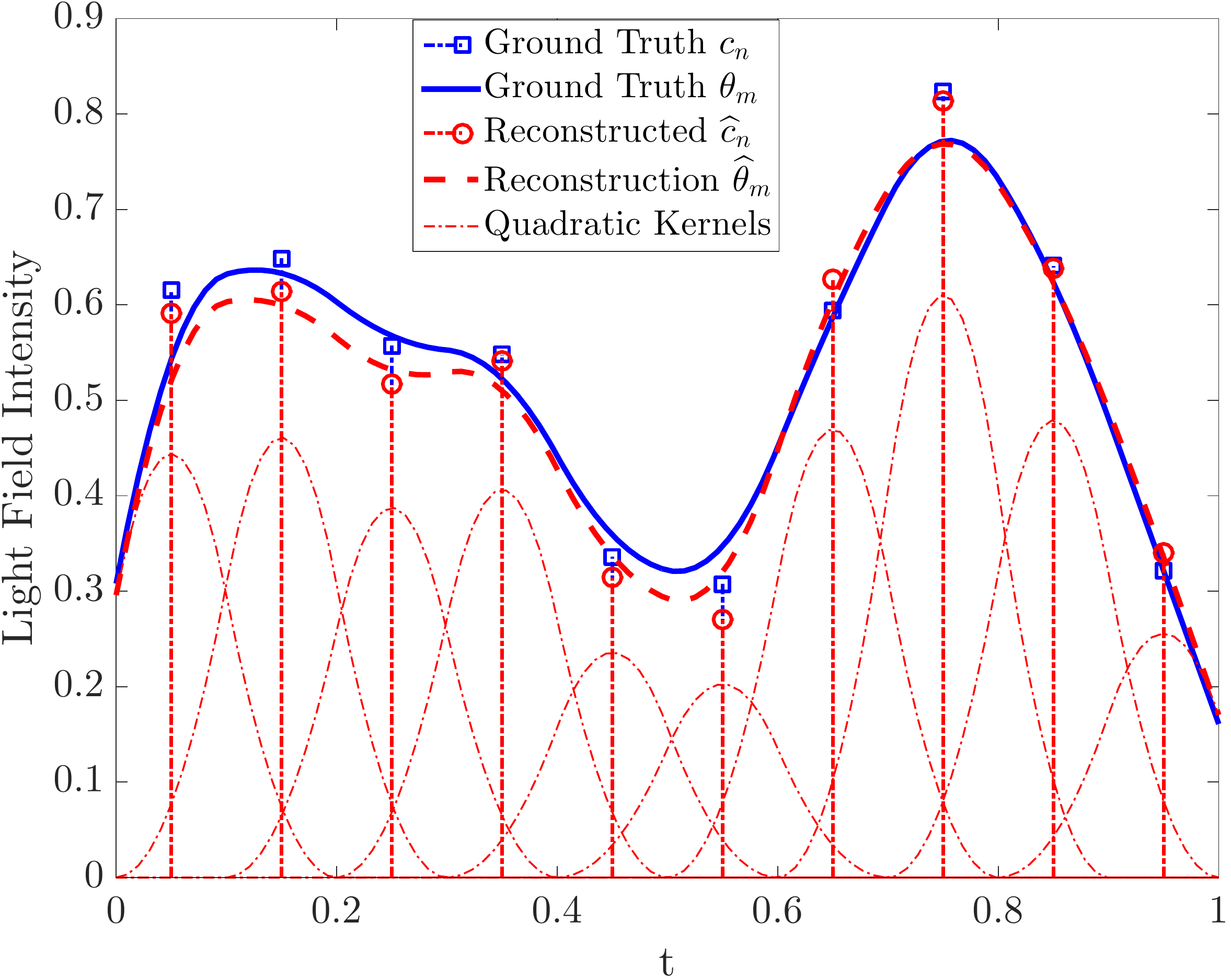}&
 \includegraphics[width=0.35\linewidth]{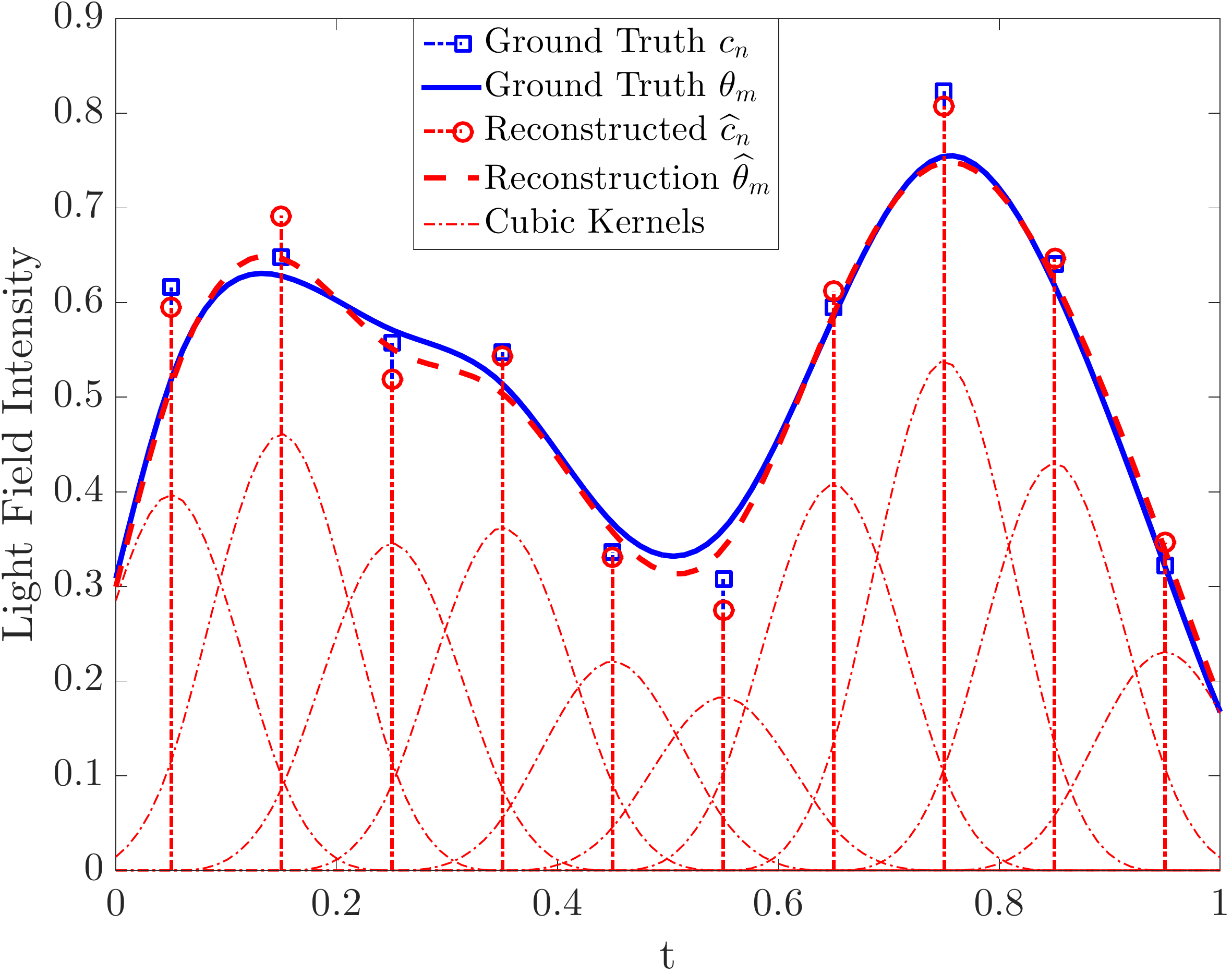}\\
 \footnotesize (c) Quadratic B-spline,  PSNR$ = 31.66$  dB &
 \footnotesize (d) Cubic B-spline,  PSNR$ = 32.15$  dB
 \end{tabular}
 \caption{Spatial oversampling $K=9$. Temporal oversampling $T=30$.  Oracle threshold map is used for quantization. Different kernels are used in synthesis and boxcar kernel is used in reconstruction. ML closed-form is used for reconstruction}
 \label{fig:result2}
 \vspace{-4ex}
 \end{figure}

\begin{figure}[!]
\centering
\begin{tabular}{cccc}
\includegraphics[width=0.23\linewidth]{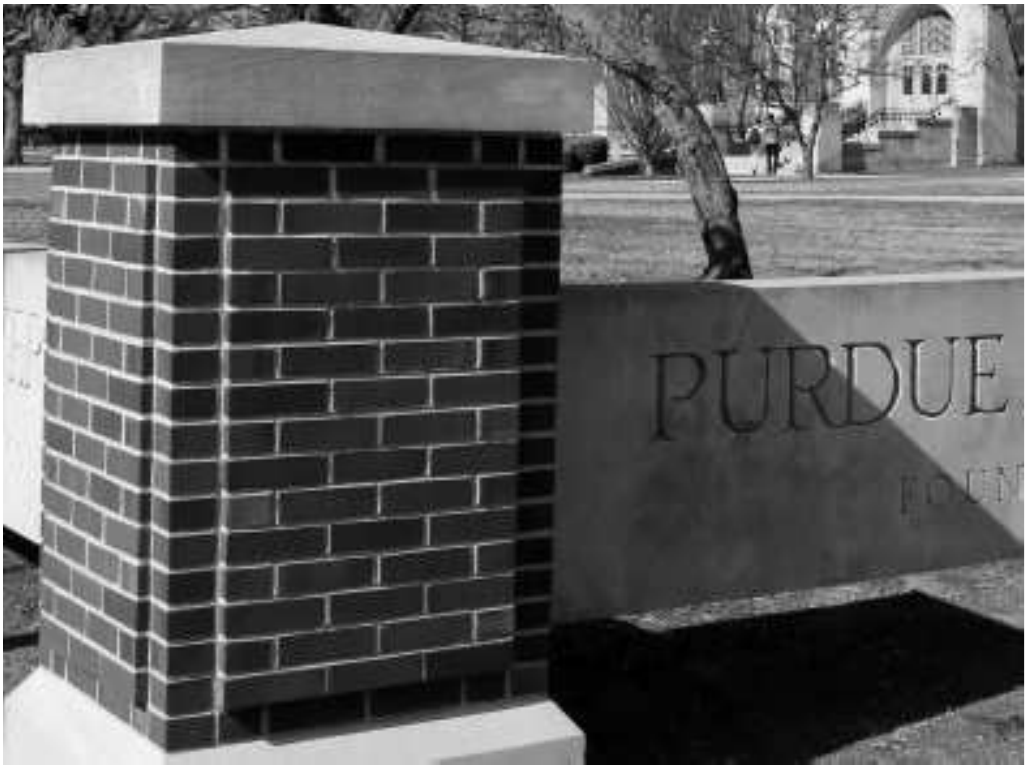}&
\includegraphics[width=0.23\linewidth]{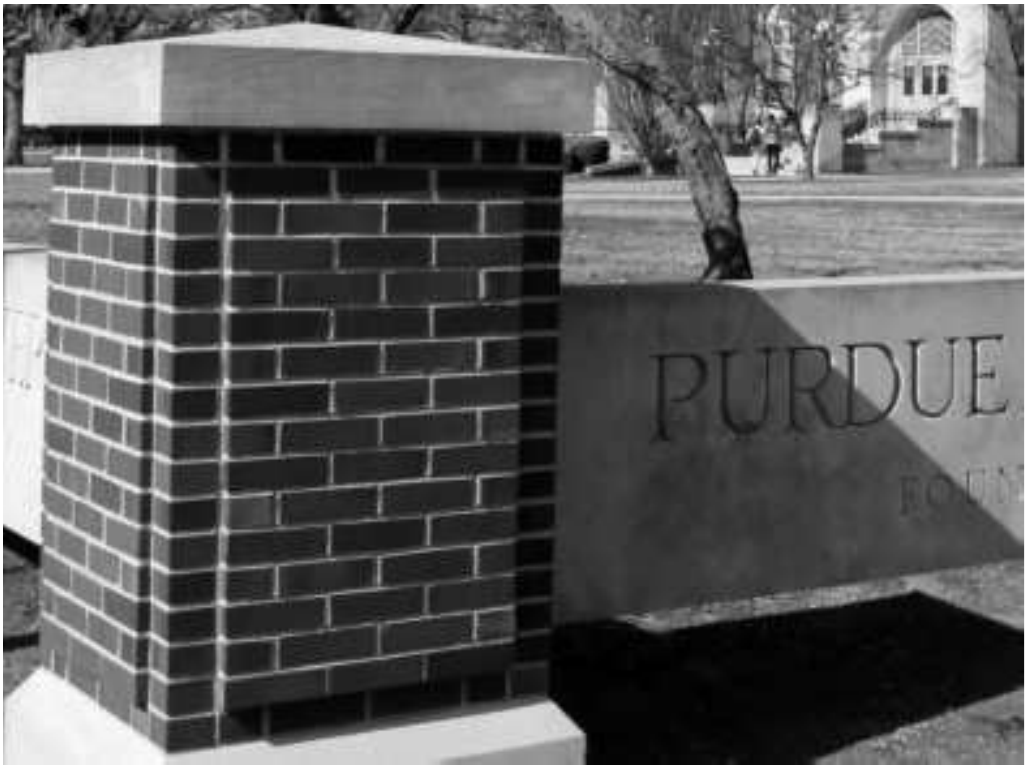}&
\includegraphics[width=0.23\linewidth]{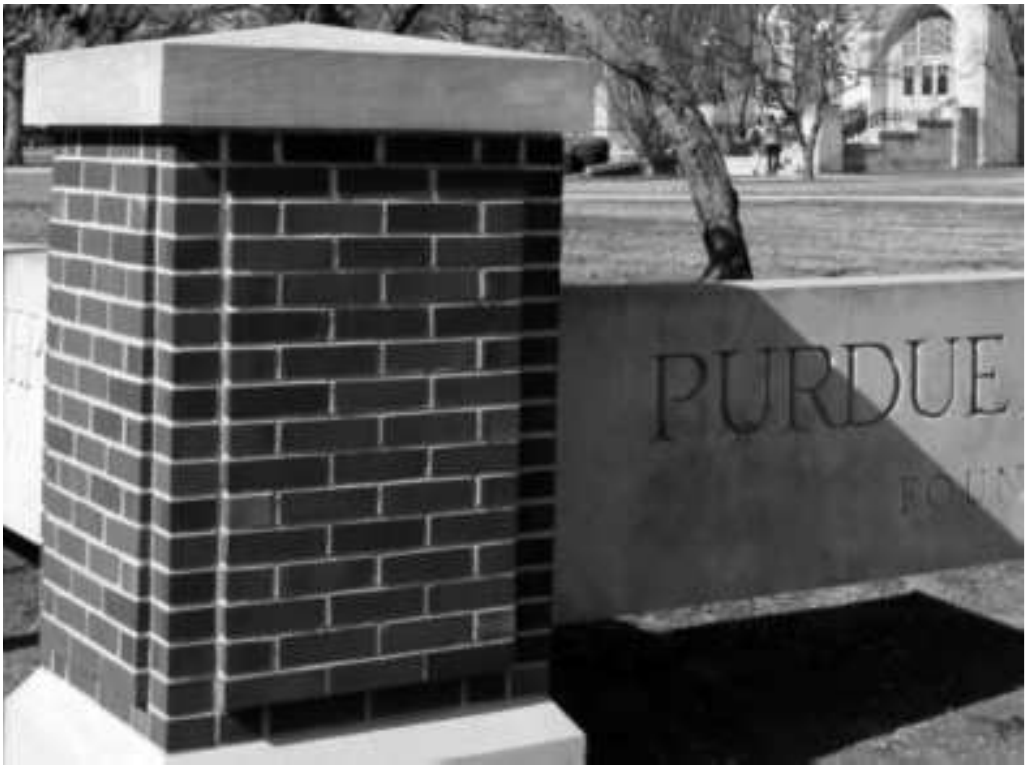}&
\includegraphics[width=0.23\linewidth]{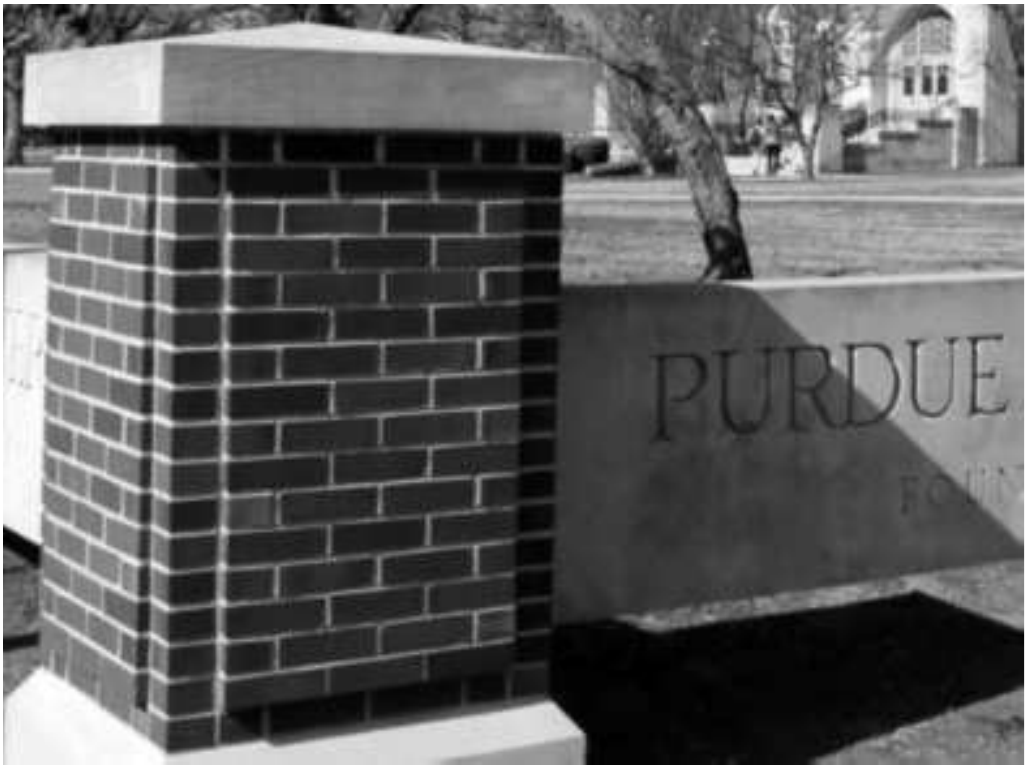}\\
\footnotesize (a) \begin{tabular}{@{}c@{}}Boxcar \\ Ground Truth\end{tabular}&
\footnotesize (b) \begin{tabular}{@{}c@{}}Linear B-spline \\ Ground Truth\end{tabular}&
\footnotesize (c) \begin{tabular}{@{}c@{}}Quadratic B-spline \\ Ground Truth\end{tabular}&
\footnotesize (d) \begin{tabular}{@{}c@{}}Cubic B-spline \\ Ground Truth\end{tabular}\\
\includegraphics[width=0.23\linewidth]{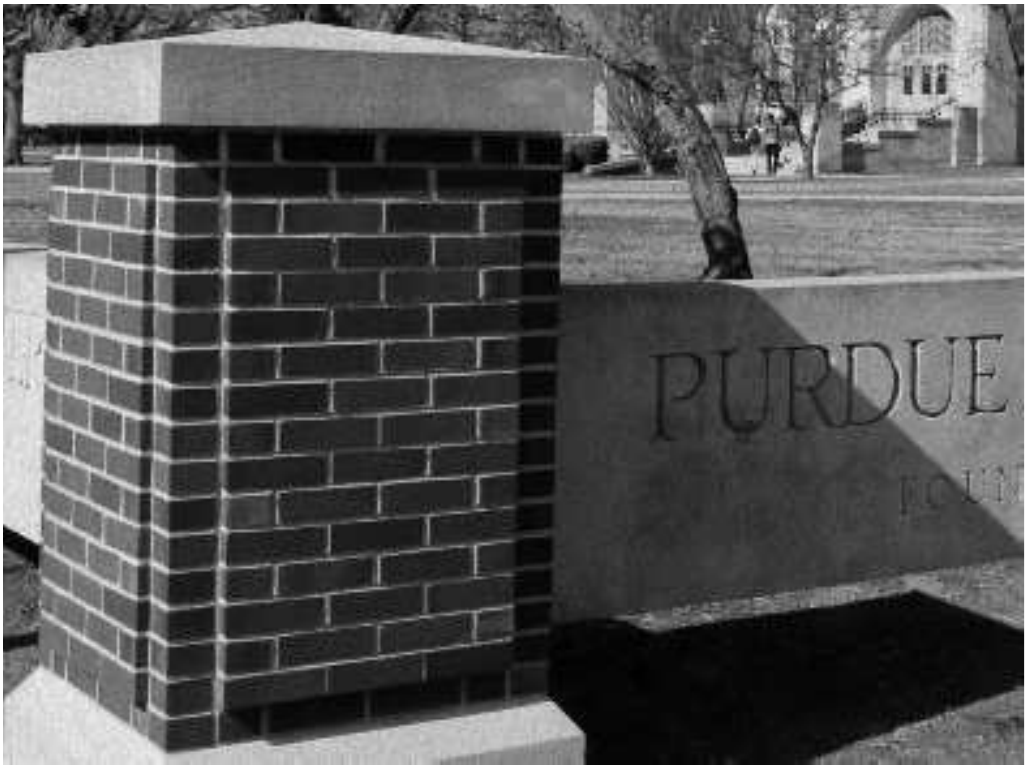}&
\includegraphics[width=0.23\linewidth]{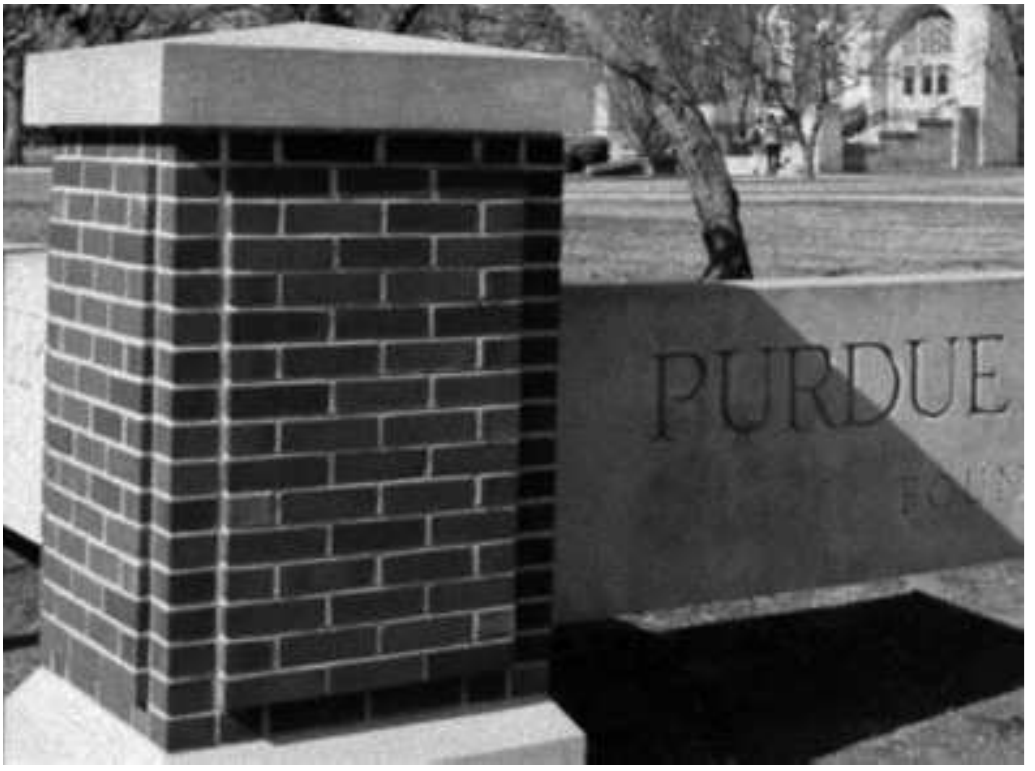}&
\includegraphics[width=0.23\linewidth]{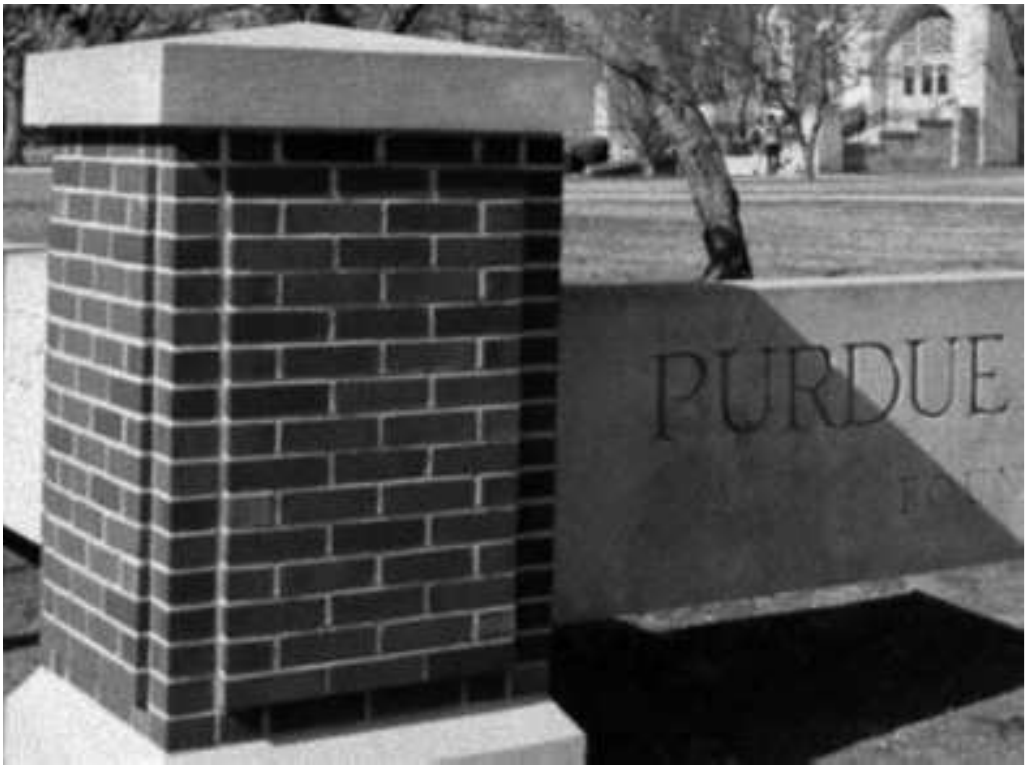}&
\includegraphics[width=0.23\linewidth]{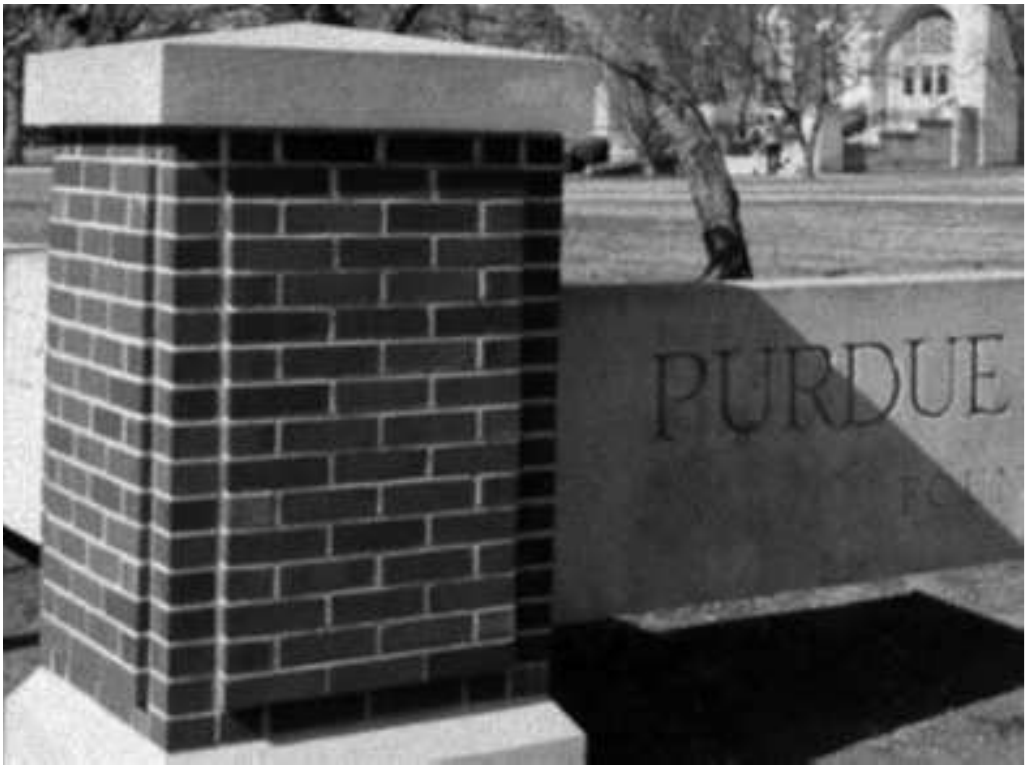}\\
\footnotesize (e) \begin{tabular}{@{}c@{}}Reconstruction \\ PSNR$ = 33.33$\end{tabular}&
\footnotesize (f) \begin{tabular}{@{}c@{}}Reconstruction \\ PSNR$ = 30.67$\end{tabular}&
\footnotesize (g) \begin{tabular}{@{}c@{}}Reconstruction \\ PSNR$ = 32.54$\end{tabular}&
\footnotesize (h) \begin{tabular}{@{}c@{}}Reconstruction \\ PSNR$ = 33.74$\end{tabular}
\end{tabular}
\caption{Ground truth and reconstructed images using simulated binary measurements synthesized by (a)(e) Boxcar, (b)(f) Linear B-spline, (c)(g) Quadratic B-spline, and (d)(h) Cubic B-spline kernels. In this experiment, we spatially oversample each pixel by $K = 4 \times 4$ binary bits and we use $T = 15$ independent temporal measurements. We use $8$ frames for learning the threshold map using bisection method, and the remaining $7$ frames are used for image reconstruction using the ML closed-form by the boxcar kernel assumption.}
\label{fig:result3}
\vspace{-2ex}
\end{figure}

\clearpage
\section{Supplementary HDR results}

In this section, we show more results for HDR image reconstruction using our method compared to the fixed threshold approach.  Figure~\ref{fig:HDR18} show reconstructed HDR images using adapted Q-map by the bisection algorithm, and fixed Q-maps with low threshold ($q=1$) and high threshold ($q_{\max}=16$). The spatial and temporal oversampling factors are $K=4$, and $T=13$, respectively. Sensor gain is $\alpha=K^2/(q_{\max}-1)$.

\begin{figure}[h]
\centering
\begin{tabular}{cccc}
\includegraphics[width=0.23\linewidth]{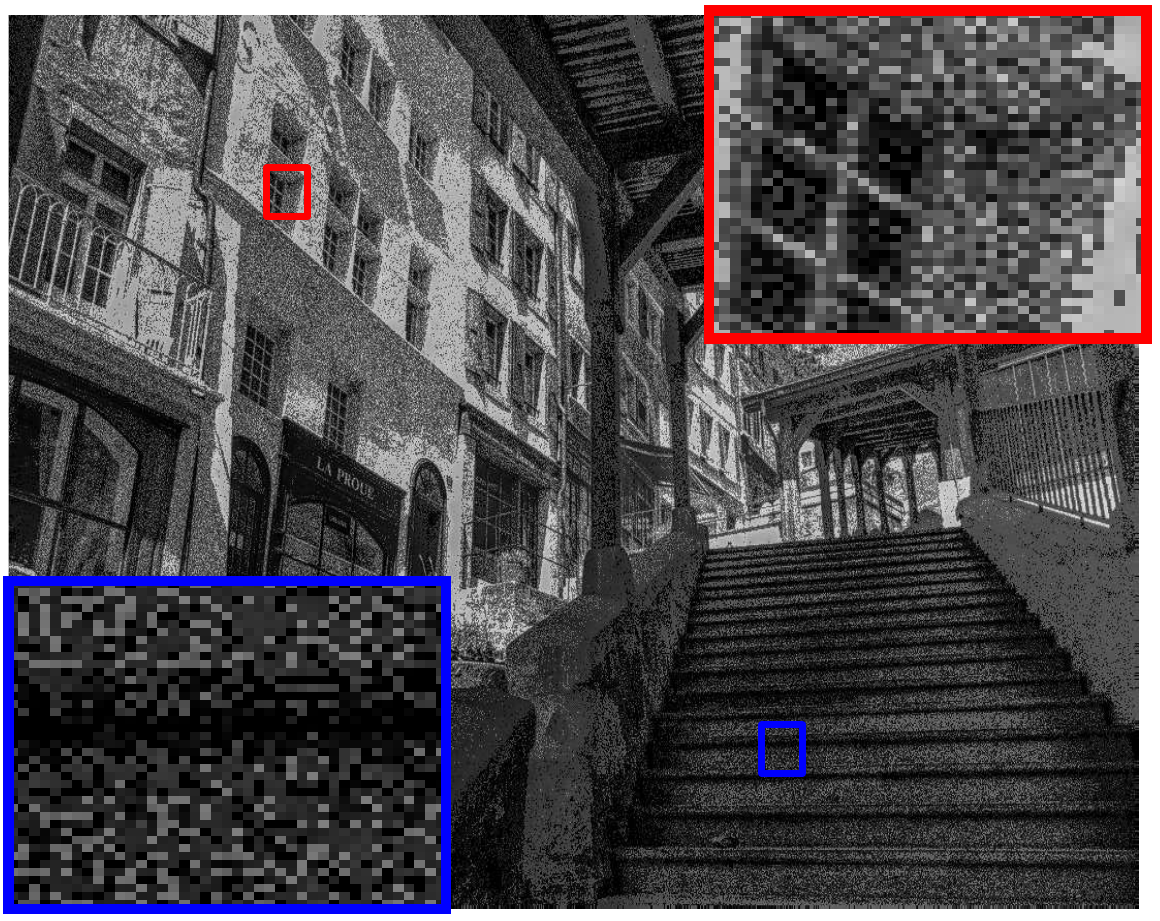}&
\includegraphics[width=0.23\linewidth]{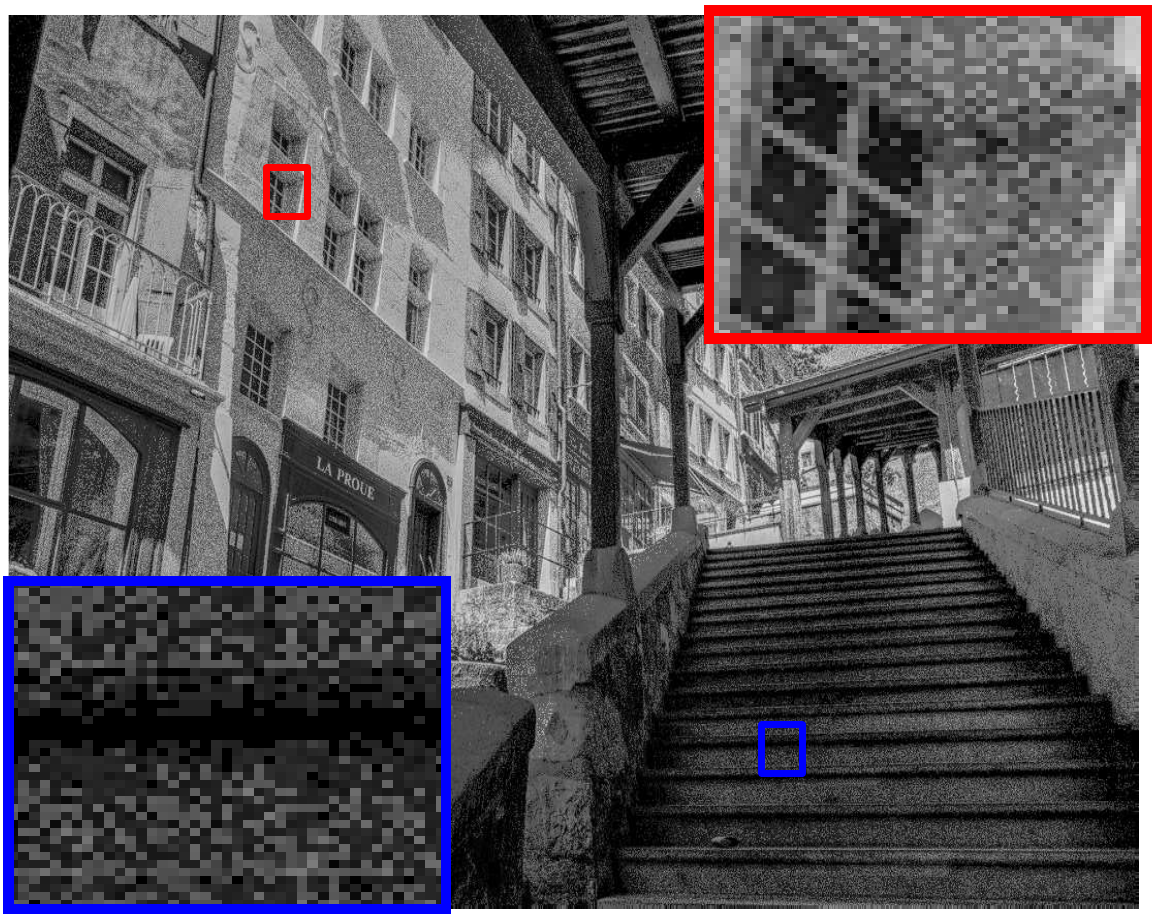}&
\includegraphics[width=0.23\linewidth]{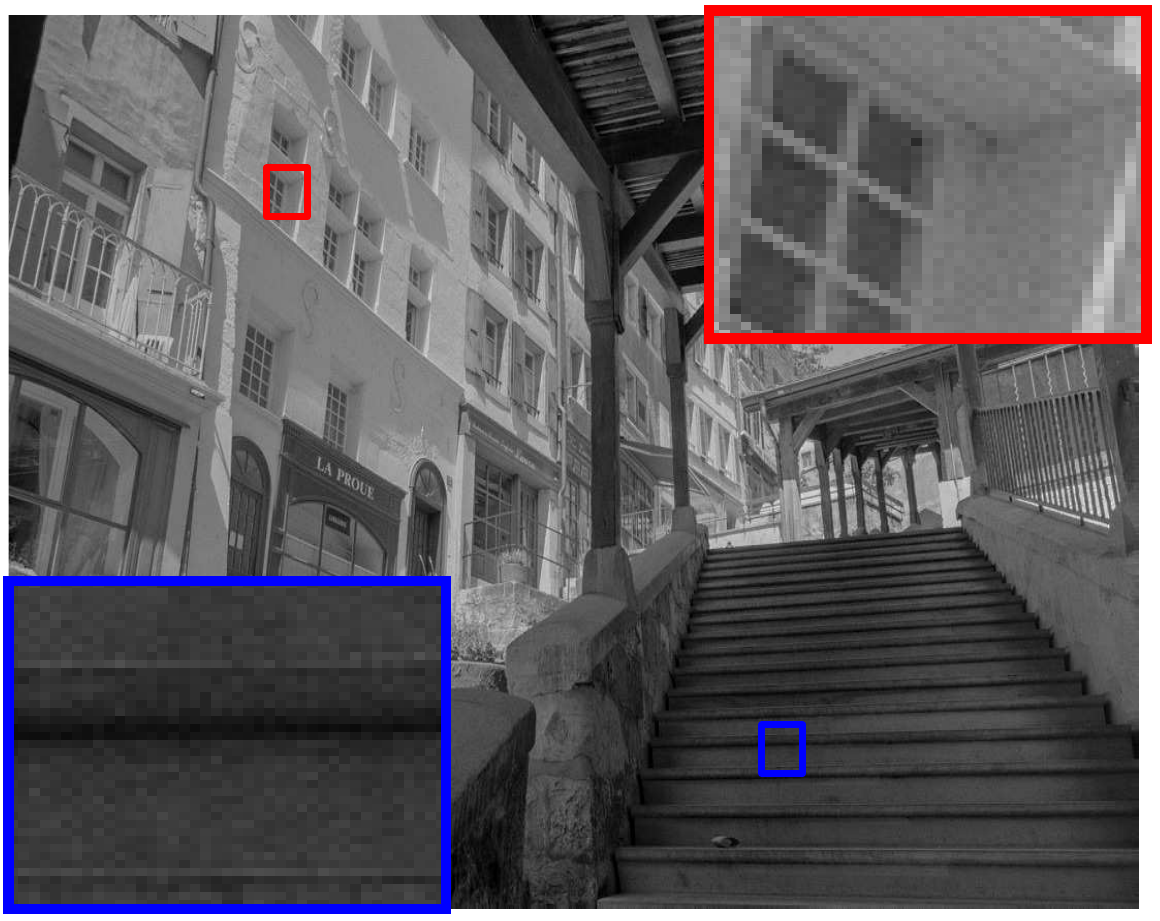}&
\includegraphics[width=0.23\linewidth]{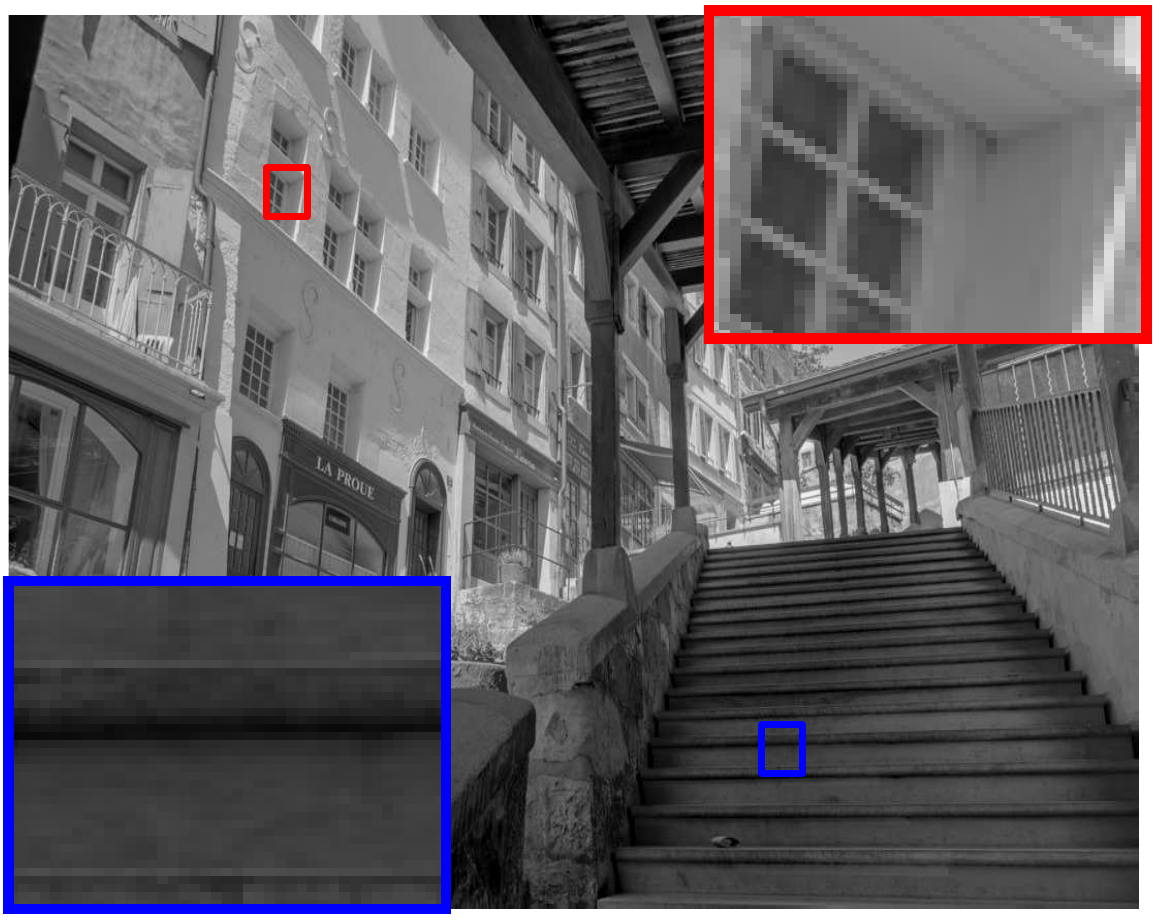}\\
\footnotesize $q = 1$, $15.94$ dB  &
\footnotesize $q = 16$, $20.77$ dB &
\footnotesize Proposed, $29.97$ dB &
\footnotesize Ground Truth\\
\includegraphics[width=0.23\linewidth]{C20_ql.pdf}&
\includegraphics[width=0.23\linewidth]{C20_qh.pdf}&
\includegraphics[width=0.23\linewidth]{C20_qs.pdf}&
\includegraphics[width=0.23\linewidth]{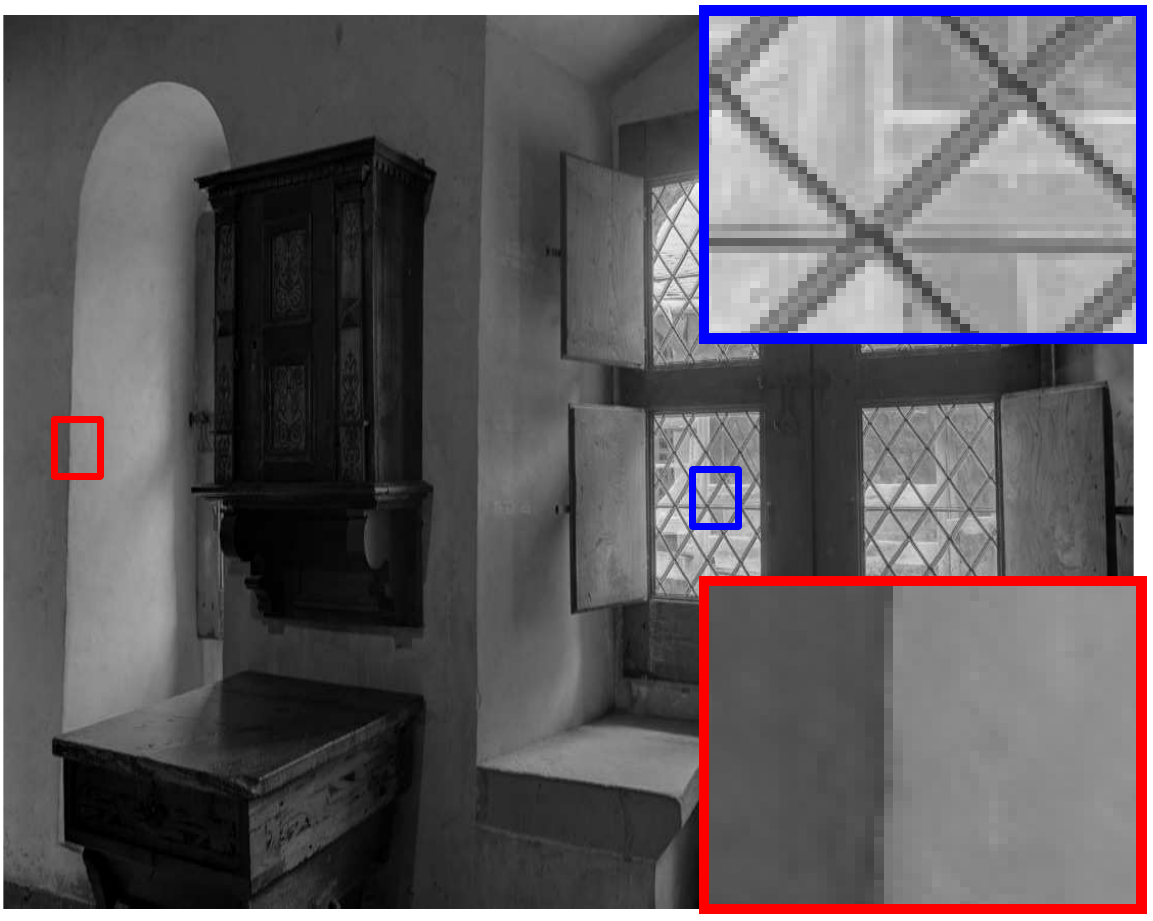}\\
\footnotesize $q = 1$, $17.94$ dB  &
\footnotesize $q = 16$, $20.77$ dB &
\footnotesize Proposed, $31.46$ dB &
\footnotesize  Ground Truth\\
\includegraphics[width=0.23\linewidth]{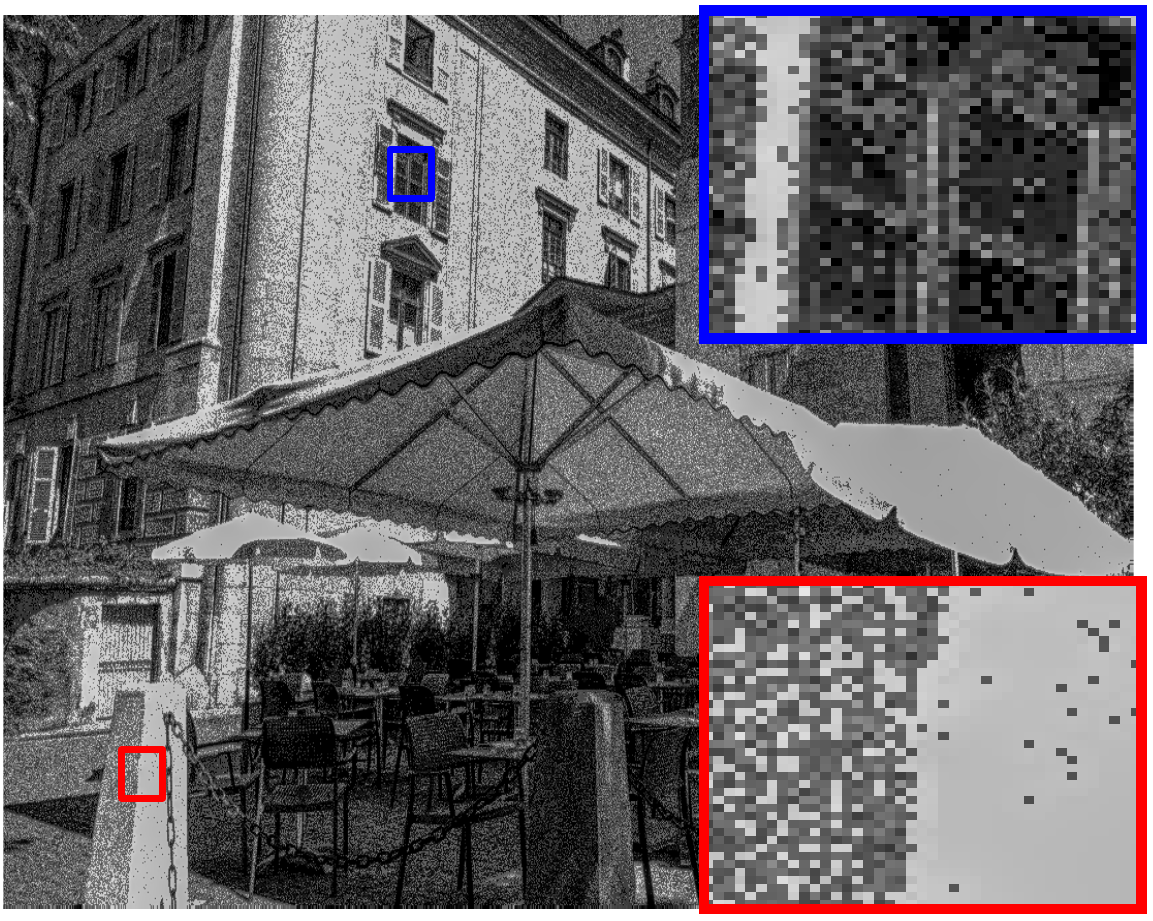}&
\includegraphics[width=0.23\linewidth]{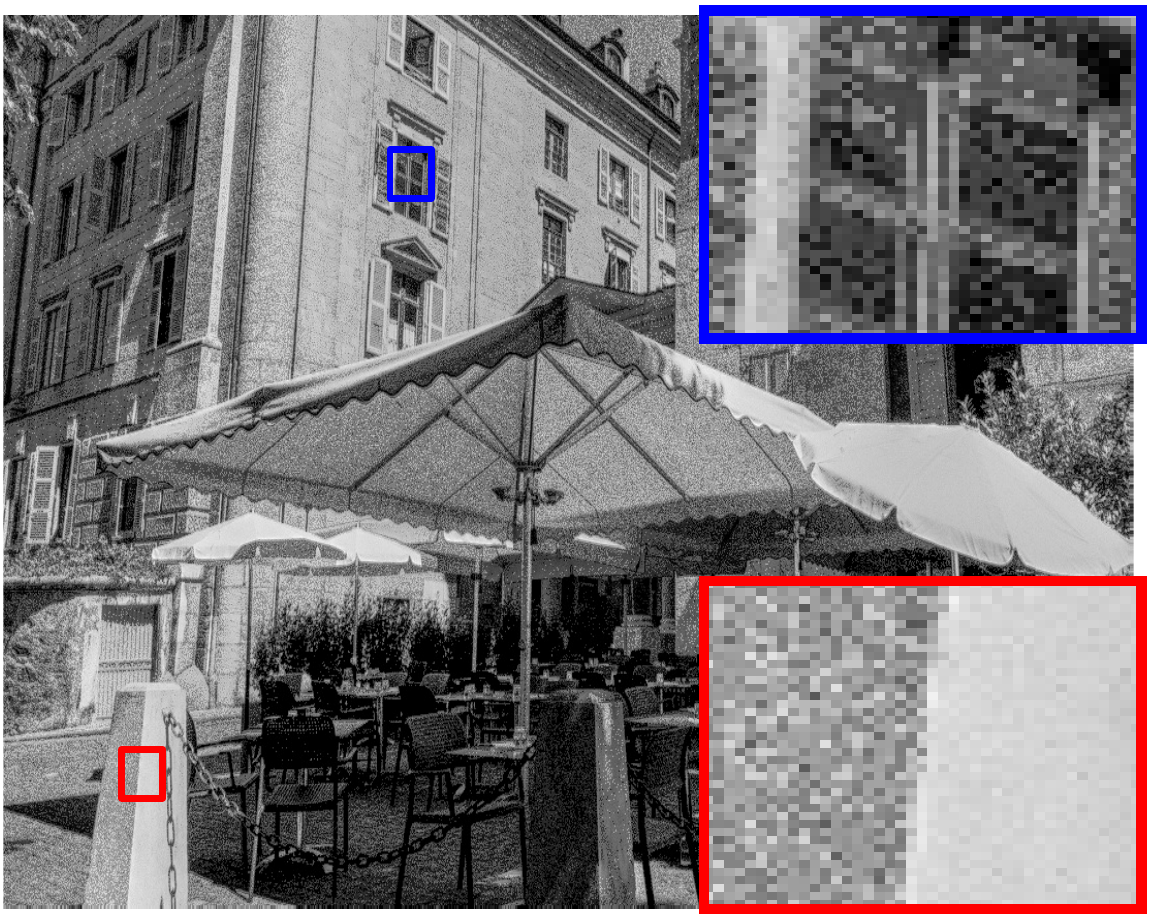}&
\includegraphics[width=0.23\linewidth]{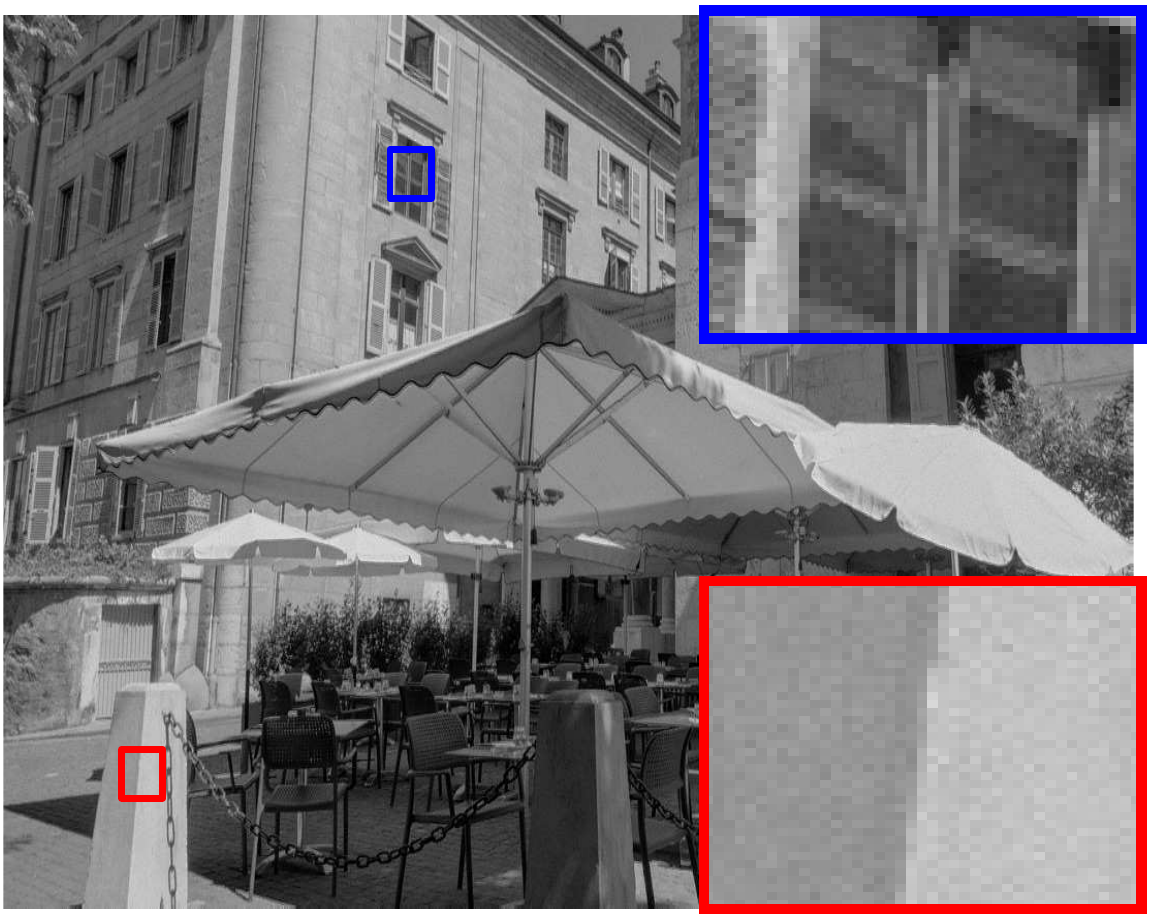}&
\includegraphics[width=0.23\linewidth]{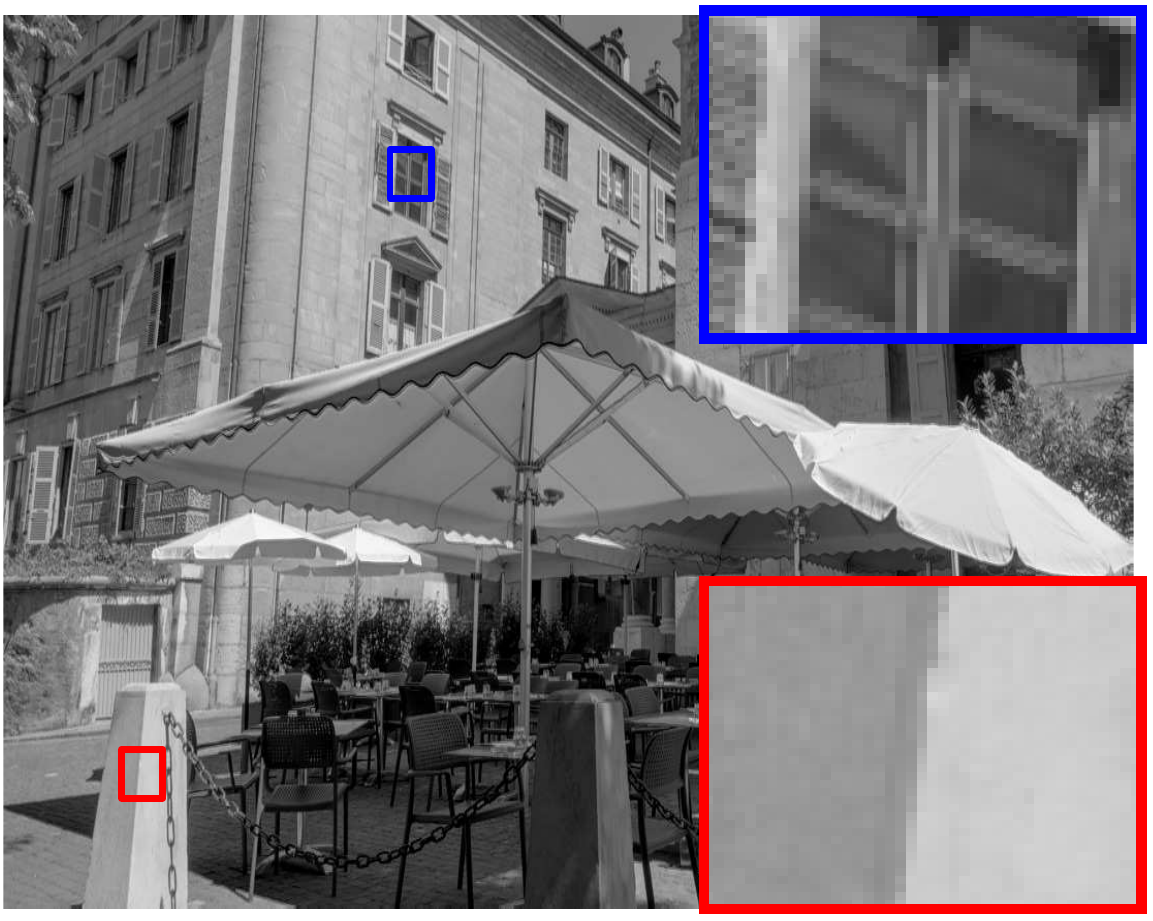}\\
\footnotesize $q = 1$, $15.74$ dB  &
\footnotesize $q = 16$, $20.01$ dB &
\footnotesize Proposed, $31.65$ dB &
\footnotesize Ground Truth
\end{tabular}
\caption{Reconstructed HDR images using different threshold maps}
\label{fig:HDR18}
\end{figure}

\bibliographystyle{IEEEbib}
\bibliography{refs}